\documentclass{article}

\usepackage{microtype}
\usepackage{graphicx}
\usepackage{subfigure}
\usepackage{booktabs} 
\usepackage{hyperref}
\usepackage{natbib}

\usepackage{fullpage}

\usepackage{hyperref}
\usepackage{url}
\hypersetup{colorlinks=true,linkcolor=blue,citecolor=blue,linktocpage}
\usepackage{tabularx,colortbl}
\usepackage{xcolor}
\usepackage{multicol}
\usepackage{algorithm,algorithmic}
\usepackage{wrapfig}
\usepackage{enumitem}

\usepackage{amsmath,amssymb,amsthm,mathtools}

\newtheorem{theorem}{Theorem}
\newtheorem{definition}{Definition}
\newtheorem{assumption}{Assumption}

\newtheorem{lemma}{Lemma}
\newtheorem{proposition}{Proposition}
\newtheorem{corollary}{Corollary}

\newcommand{\x}{{\normalfont\textrm x}}
\newcommand{\y}{{\normalfont\textrm y}}
\renewcommand{\v}{{\textrm v}}
\newcommand{\e}{{\textrm e}}
\renewcommand{\u}{{\textrm u}}
\renewcommand{\a}{{\textrm a}}

\newcommand{\1}{\textrm 1}

\newcommand{\p}{{\textrm p}}
\newcommand{\q}{{\textrm q}}
\newcommand{\w}{{\normalfont\textrm w}}
\renewcommand{\b}{{\textrm b}}

\newcommand{\I}{\textrm{I}}
\newcommand{\z}{\textrm{z}}

\newcommand{\A}{\textrm{A}}
\newcommand{\B}{\textrm{B}}
\newcommand{\M}{\textrm{M}}

\newcommand{\W}{\textrm{W}}
\newcommand{\C}{{\normalfont\textrm{C}}}

\newcommand{\U}{\textrm{U}}
\newcommand{\V}{\textrm{V}}

\newcommand{\Q}{\textrm{Q}}

\newcommand{\X}{{\normalfont\textrm{X}}}

\renewcommand{\L}{\textrm{L}}

\newcommand{\cD}{\mathcal{D}}
\newcommand{\cJ}{\mathcal{J}}
\newcommand{\cF}{\mathcal{F}}
\newcommand{\cG}{\mathcal{G}}
\newcommand{\cH}{\mathcal{H}}
\newcommand{\cL}{\mathcal{L}}
\newcommand{\cM}{\mathcal{M}}

\newcommand{\cW}{\mathcal{W}}
\newcommand{\cX}{\mathcal{X}}
\newcommand{\cY}{\mathcal{Y}}
\newcommand{\cS}{\mathcal{S}}

\newcommand{\cZ}{\mathcal{Z}}

\newcommand{\fR}{\mathfrak{R}}

\newcommand{\bb}{\mathbb}
\newcommand{\R}{\bb R}

\newcommand{\bP}{\bb P}
\newcommand{\E}{\bb E}

\DeclareMathOperator*{\var}{Var}
\DeclareMathOperator*{\diag}{diag}
\DeclareMathOperator*{\rank}{Rank}
\DeclareMathOperator*{\tr}{Tr}

\newcommand{\abs}[1]{\left| #1 \right|}
\newcommand{\relu}{\sigma}
\newcommand{\vect}[1]{\operatorname{vec}\left( #1 \right)}

\newcommand{\minim}[2]{\underset{#1}{\textrm{min}} \ #2}

\newcommand{\sout}[1]{}
\renewcommand{\hat}{\widehat}

\title{Dropout: Explicit Forms and Capacity Control}

\author{Raman Arora\thanks{Johns Hopkins University, email:
    arora@cs.jhu.edu}~ \and Peter Bartlett\thanks{University of California, Berkeley, email: bartlett@cs.berkeley.edu}~ \and Poorya Mianjy\thanks{Johns Hopkins University, email: mianjy@jhu.edu}~ \and Nathan Srebro\thanks{TTI Chicago, email: nati@ttic.edu}}

\begin{document}
\maketitle

\begin{abstract}
We investigate the capacity control provided by dropout in various machine learning problems. First, we study dropout for matrix completion, where it induces a data-dependent regularizer that, in expectation, equals the weighted trace-norm of the product of the factors. In deep learning, we show that the data-dependent regularizer due to dropout directly controls the Rademacher complexity of the underlying class of deep neural networks. These developments enable us to give concrete generalization error bounds for the dropout algorithm in both matrix completion as well as training deep neural networks. We evaluate our theoretical findings on real-world datasets, including MovieLens, MNIST, and Fashion-MNIST.
\end{abstract}

\section{Introduction}\label{sec:intro}

Dropout is a popular algorithmic regularization technique for training deep neural networks that aims at ``breaking co-adaptation'' among neurons by randomly dropping them at  training time~\citep{hinton2012improving}. Dropout has been shown effective across a wide range of machine learning tasks, from classification~\citep{srivastava2014dropout,szegedy2015going} to regression~\citep{toshev2014human}. Notably, dropout is considered an essential component in the design of AlexNet~\citep{krizhevsky2012imagenet}, which won the prominent ImageNet challenge in 2012 with a significant margin and helped transform the field of computer vision.

Dropout regularizes the empirical risk by randomly perturbing the model parameters during training. A natural first step toward understanding generalization due to dropout, therefore, is to instantiate the explicit form of the regularizer due to dropout. In linear regression, with dropout applied to the input layer (i.e., on the input features), the explicit regularizer was shown to be akin to a data-dependent ridge penalty~\cite{srivastava2014dropout,wager2013dropout,baldi2013understanding,wang2013fast}. In factored models dropout yields more exotic forms of regularization. For instance, dropout induces regularizer that behaves similar to nuclear norm regularization in matrix factorization~\cite{cavazza2018dropout}, in single hidden-layer linear networks  \cite{mianjy2018implicit}, and in deep linear networks \cite{mianjy2019dropout}. However, none of the works above discuss how the induced regularizer provides capacity control, or equivalently, help us establish generalization bounds for dropout.

In this paper, we provide an answer to this question. We give {\emph{explicit forms}} of the regularizers induced by dropout for the matrix sensing problem and two-layer neural networks with ReLU activations. Further, we establish {\emph{capacity control}} due to dropout and give precise generalization bounds. Our key contributions are as follows. 

\vspace*{-8pt}
\begin{enumerate}
\item In Section~\ref{sec:sensing}, we study dropout for matrix completion, wherein, the matrix factors are dropped randomly during training. We show that this algorithmic procedure induces a data-dependent  regularizer that behaves similar to the weighted trace-norm which has been shown to yield strong generalization guarantees for matrix completion~\citep{foygel2011learning}.
\vspace*{-2pt}
\item In Section~\ref{sec:dnn}, we study dropout in two-layer ReLU networks. We show that the regularizer induced by dropout is a data-dependent measure that behaves as $\ell_2$-path norm~\cite{neyshabur2015path}, and establish data-dependent generalization bounds.  

\vspace*{-2pt}
\item In Section~\ref{sec:exp}, we present empirical evaluations that confirm our theoretical findings for matrix completion and deep regression on real world datasets including the MovieLens data, as well as the MNIST and Fashion MNIST datasets.  
\end{enumerate}

\subsection{Related Work}\label{sec:related}

Dropout was first introduced by~\citet{hinton2012improving} as an effective heuristic for algorithmic regularization, yielding lower test errors on the MNIST and TIMIT datasets. In a subsequent work,~\citet{srivastava2014dropout} reported similar improvements over several tasks in computer vision (on CIFAR-10/100 and ImageNet datasets), speech recognition, text classification and genetics. 

Thenceforth, dropout has been widely used in training state-of-the-art systems for several tasks including large-scale visual recognition~\cite{szegedy2015going}, large vocabulary continuous speech recognition~\cite{dahl2013improving}, image question answering~\cite{yang2016stacked}, handwriting recognition~\cite{pham2014dropout}, sentiment prediction and question classification~\cite{kalchbrenner2014convolutional}, dependency parsing~\cite{chen2014fast},  and brain tumor segmentation~\cite{havaei2017brain}.

Following the empirical success of dropout, there have been several studies in recent years aimed at establishing theoretical underpinnings of why and how dropout helps with generalization. Early work of~\citet{baldi2013understanding} showed that for a single linear unit (and a single sigmoid unit, approximately), dropout amounts to weight decay regularization on the weights. A similar result was shown by \cite{mcallester2013pac} in a PAC-Bayes setting. For generalized linear models, \citet{wager2013dropout} established that dropout performs an adaptive regularization which is % first-order 
equivalent to a data-dependent scaling of the weight decay penalty. In their follow-up work, \citet{wager2014altitude} show that for linear classification, under a generative assumption on the data, dropout improves the convergence rate of the generalization error. In this paper, we focus on predictors represented in a factored form and give generalization bounds for matrix learning problems and single hidden layer ReLU networks.

In a related line of work, \citet{helmbold2015inductive} study the structural properties of the dropout regularizer 
in the context of linear classification. They characterize the landscape of the dropout criterion in terms of unique minimizers and establish non-monotonic and non-convex nature of the regularizer. In a follow up work, \citet{ helmbold2017surprising} extend their analysis to dropout in deep ReLU networks and surprisingly find that the nature of regularizer is different from that in linear classification. In particular, they show that unlike weight decay, dropout regularizer in deep networks can grow exponentially with depth and remains invariant to rescaling of inputs, outputs, and network weights. We confirm some of these findings in our theoretical analysis. However, counter to the claims of \citet{ helmbold2017surprising}, we argue that dropout does indeed prevent co-adaptation.

In a closely related approach as ours, the works of \citet{zhai2018adaptive}, \citet{gao2016dropout}, and \citet{wan2013regularization} bound the Rademacher complexity of deep neural networks trained using dropout. In particular, \citet{gao2016dropout} show that the Rademacher complexity of the target class decreases polynomially or exponentially, for shallow and deep networks, respectively, albeit they assume additional norm bounds on the weight vectors. Similarly, the works of \cite{wan2013regularization} and \cite{zhai2018adaptive} assume that certain norms of the weights are bounded, and show that the Rademacher complexity of the target class decreases with dropout rates. We argue in this paper that dropout alone does not directly control the norms of the weight vectors; therefore, each of the works above fail to capture the practice. We emphasize that none of the previous works provide a generalization guarantee, i.e., a bound on the gap between the population risk and the empirical risk, merely in terms of the value of the explicit regularizer due to dropout. We give a first such result for dropout in the context of matrix completion and for a single hidden layer ReLU network.

There are a bunch of other works that do not fall into any of the categories above, and, in fact, are somewhat unrelated to the focus in this paper. Nonetheless, we discuss them here for completeness. 
For instance, \citet{gal2016dropout} study dropout as Bayesian approximation , \citet{bank2018relationship} draw insights from frame theory to connect the notion of equiangular tight frames with dropout training in auto-encoders. Also, some recent works have considered variants of dropout. For instance, \citet{mou2018dropout} consider a variant of dropout, which they call
``truthful'' dropout, that ensures that the output of the randomly perturbed network is unbiased. However, rather than bound generalization error, \citet{mou2018dropout} bound the gap between the population risk and the dropout objective, i.e., the empirical risk plus the explicit regularizer. \citet{li2016improved} study a yet another variant based on multinomoal sampling (different nodes are dropped with different rates), and establish sub-optimality bounds for stochastic optimization of linear models (for convex Lipschitz loss functions).

\vspace*{-10pt}
\paragraph{Matrix Factorization with Dropout.} Our study of dropout is motivated in part by recent works of \citet{cavazza2018dropout}, \citet{mianjy2018implicit}, and \citet{mianjy2019dropout}. This line of work was initiated by \citet{cavazza2018dropout}, who studied dropout for low-rank matrix factorization without constraining the rank of the factors or adding an explicit regularizer to the objective. They show that dropout in the context of matrix factorization yields an explicit regularizer whose convex envelope is given by nuclear norm. This result is further strengthened by \citet{mianjy2018implicit} who show that induced regularizer is indeed nuclear norm. 

While matrix factorization is not a learning problem per se (for instance, what is training versus test data), in follow-up works by \citet{mianjy2018implicit} and \citet{mianjy2019dropout}, the authors show that training deep linear networks with $\ell_2$-loss using dropout reduces to the matrix factorization problem if the marginal distribution of the input feature vectors is assumed to be isotropic, i.e., $\E[\x\x^\top]=\I$. We note that this is a strong assumption. If we do not assume isotropy, we show that dropout induces a data-dependent regularizer which amounts to a simple scaling of the parameters and, therefore, does not control capacity in any meaningful way. We revisit this discussion in Section~\ref{sec:regression}. 

To summarize, while we are motivated by \citet{cavazza2018dropout}, the  problem setup, the nature of statements in this paper, and the tools we use are different from that in \cite{cavazza2018dropout}. Our proofs are simple and quickly verified. We do build closely on the prior work of \citet{mianjy2018implicit}. 

However, different from \citet{mianjy2018implicit}, we rigorously argue for dropout in matrix completion by 1) showing that the induced regularizer is equal to weighted trace-norm, which as far as we know, is a novel result, 2) giving strong generalization bounds, and 3) providing extensive experimental evidence that dropout provides state of the art performance on one of the largest datasets in recommendation systems research. Beyond that we rigorously extend our results to two layer ReLU networks, describe the explicit regularizer, bound the Rademacher complexity of the hypothesis class controlled by dropout, show precise generalization bounds, and support them with empirical results.

\subsection{Notation and Preliminaries}\label{sec:notation}
We denote matrices, vectors, scalar variables and sets by Roman capital letters, Roman small letters, small letters, and script letters, respectively (e.g. $\X$, $\x$, $x$, and $\cX$). For any integer $d$, we represent the set $\{ 1,\ldots,d \}$ by $[d]$. For any vector $\x \in \R^d$, $\diag(\x)\in \R^{d\times d}$ represents the diagonal matrix with
the $i^{th}$ diagonal entry equal to $\x_i$, and $\sqrt{\x}$ is the elementwise squared root of $\x$. Let $\| \x \|$ represent the $\ell_2$-norm of vector $\x$, and $\| \X \|$, $\| \X \|_F$, and $\| \X\|_*$ represent the spectral norm, the Frobenius norm, and the nuclear norm of matrix $\X$, respectively. Let $\X^\dagger$ denote the Moore-Penrose pseudo-inverse of $\X$. Given a positive definite matrix $\C$, we denote the Mahalonobis norm as $\|\x\|_\C^2 = \x^\top\C\x$. For a random variable $\x$ that takes values in $\cX$, given $n$ i.i.d. samples $\{\x_1,\cdots,\x_n\}$, the empirical average of a function $f:\cX\to \R$ is denoted by $\hat\E_i[f(\x_i)]:=\frac1n\sum_{i\in[n]}f(\x_i)$. Furthermore, we denote the second moment of $\x$ as $\C:=\E[\x\x^\top]$. The standard inner product is represented by $\langle \cdot, \cdot \rangle$, for vectors or matrices, where $\langle \X,\X' \rangle = \tr(\X^\top \X')$.

We are primarily interested in understanding how dropout controls the capacity of the hypothesis class when using dropout for training. To that end, we consider Rademacher complexity, a sample dependent measure of complexity of a hypothesis class that can directly bound the generalization gap~\citep{bartlett2002rademacher}. Formally, let $\cS = \{(\x_1, y_1),\ldots , (\x_n, y_n)\}$ be a sample of size $n$. Then, the empirical Rademacher complexity of a function class $\cF$ with respect to $\cS$, and the expected Rademacher complexity are defined, respectively, as
\vspace*{-10pt}
\begin{equation*}
\fR_\cS(\cF) = \E_{\sigma} \sup_{f\in \cF} \frac1n \sum_{i=1}^n \sigma_i f(\x_i), \ \ \fR_n(\cF) = \E_\x[\fR_\cS(\cF)],
\end{equation*}
where $\sigma_i$ are i.i.d. Rademacher random variables. 
\section{Matrix Sensing}\label{sec:sensing}
We begin with understanding dropout for matrix sensing, a problem which arguably is an important instance of a matrix learning problem with lots of applications, and is well understood from a  theoretical perspective. The problem setup is the following. Let $\M_* \in \R^{d_2 \times d_0}$ be a matrix with rank $r_*:=\rank(\M_*)$. Let $\A^{(1)},\ldots,\A^{(n)}$ be a set of measurement matrices of the same size as $\M_*$. The goal of matrix sensing is to recover the matrix $\M_*$ from $n$ observations of the form $y_i = \langle \M_*, \A^{(i)} \rangle$ such that $n\ll d_2 d_0$. 

A natural approach is to represent the matrix in terms of factors and solve the following \emph{empirical risk} minimization problem:
\begin{equation}\label{prob:erm_ms}
\minim{\U,\V}{\hat{L}(\U,\V):=\hat\E_i(y_i - \langle \U\V^\top, \A^{(i)} \rangle)^2} 
\end{equation}
where $\U = [\u_1,\ldots,\u_{d_1}] \in \R^{d_2\times d_1},\V = [\v_1,\ldots,\v_{d_1}] \in \R^{d_0\times d_1}$. When the number of factors is unconstrained, i.e., when $d_1 \gg r_*$, there exist many ``bad'' empirical minimizers, i.e., those with a large \emph{true risk} $L(\U,\V):=\E(y-\langle\U\V^\top,\A\rangle )^2$. 
Interestingly, \citet{li2018algorithmic} showed recently that under a restricted isometry property (RIP), despite the existence of such poor ERM solutions, gradient descent with proper initialization is \emph{implicitly} biased towards finding solutions with minimum nuclear norm -- this is an important result which was first conjectured and empirically verified by~\citet{gunasekar2017implicit}. We do not make an RIP assumption here. Further, we argue that for the most part, modern machine learning systems employ \emph{explicit} regularization techniques. In fact, as we show in the experimental section, the \emph{implicit} bias due to (stochastic) gradient descent does not prevent it from blatant overfitting in the matrix completion problem.

We propose solving the ERM problem~\eqref{prob:erm_ms} using dropout, where at training time, corresponding columns of $\U$ and $\V$ are dropped uniformly at random. As opposed to an \emph{implicit} effect of gradient descent, dropout  \emph{explicitly} regularizes the empirical objective. It is then natural to ask, in the case of matrix sensing, if dropout also biases the ERM towards certain low norm solutions. To answer this question, we begin with the observation that dropout can be viewed as an instance of SGD on the following objective~\cite{cavazza2018dropout,mianjy2018implicit}: 
\begin{equation}\label{eq:dropout_erm_sensing}
\hat{L}_{\text{drop}}(\U,\V)=\hat\E_j\E_\B(y_j - \langle \U\B\V^\top, \A^{(j)} \rangle)^2, 
\end{equation}
where $\B\in \R^{d_1\times d_1}$ is a diagonal matrix whose diagonal elements are Bernoulli random variables distributed as $\B_{ii} \sim \frac1{1-p}\textrm{Ber}(1-p)$. % In this case, 
It is easy to show that for  $p\in[0,1)$:
\begin{align}
\label{eq:dropout_obj_sensing}
    \hat{L}_\text{drop}(\U,\V) &=\hat{L}(\U,\V)+\frac{p}{1-p}\hat{R}(\U,\V),
\end{align}
where $\hat{R}(\U,\V):=\sum_{i=1}^{d_1}\hat\E_j(\u_i^\top \A^{(j)}\v_i)^2$ is a data-dependent term that captures the \emph{explicit} regularizer due to dropout. A similar result was shown by \citet{cavazza2018dropout} and \citet{mianjy2018implicit}, but we provide a proof for completeness (see Proposition~\ref{prop:dropout_reg_sensing} in the Appendix). 

We show that the \emph{explicit} regularizer concentrates around its expected value w.r.t. the data distribution (see Lemma~\ref{lem:concentration} in the Appendix). Furthermore, given that we seek a minimum of $\hat{L}_{\textrm{drop}}$, it suffices to consider the factors with the minimal value of the regularizer among all that yield the same empirical loss. This motivates studying the the following distribution-dependent \emph{induced} regularizer:
\begin{equation*}
\Theta(\M)\!:=\!\!\minim{\U\V^\top=\M}{\!R(\U,\!\V)}, \quad\textrm{where} \ R(\U,\!\V)\!:=\! \E_\A[\hat{R}(\U,\!\V)]. 
\end{equation*}
For a wide range of random measurements, $\Theta(\cdot)$ turns out to be a ``suitable'' regularizer. Here, we instantiate two important examples (see Proposition~\ref{prop:induced} in the Appendix).

\begin{paragraph}{Gaussian Measurements.}
For all $j\in [n]$, let $\A^{(j)}$ be standard Gaussian matrices. In this case, it is easy to see that $\L(\U,\V)=\|\M_*-\U\V^\top\|_F^2$ and we recover the matrix factorization problem. Furthermore, we know from~\cite{cavazza2018dropout, mianjy2019dropout} that dropout regularizer acts as trace-norm regularization, i.e., $\Theta(\M) = \frac1{d_1}\| \M \|_*^2$.
\end{paragraph}

\begin{paragraph}{Matrix Completion.}
For all $j\in [n]$, let $\A^{(j)}$ be an indicator matrix whose $(i,k)$-th element is selected randomly with probability $p(i)q(k)$, where $p(i)$ and $q(k)$ denote the probability of choosing the $i$-th row and the $k$-th column, respectively. Then
$$\Theta(\M)= \frac1{d_1} \| \diag(\sqrt p)\U\V^\top \diag(\sqrt q) \|_*^2$$ is the \emph{weighted trace-norm} studied by~\citet{srebro2010collaborative} and~\citet{foygel2011learning}.
\end{paragraph}

These observations are specifically important because they connect dropout, an algorithmic heuristic in deep learning, to strong complexity measures that are empirically effective as well as theoretically well understood. To illustrate, here we give a generalization bound for matrix completion using dropout in terms of the value of the \emph{explicit} regularizer at the minimizer. 
\begin{theorem}\label{thm:generalization_sensing_expected} \normalfont 
Assume that $d_2\geq d_0$ and $\|\M_*\| \leq 1$. Furthermore, assume that $\min_{i,k}p(i)q(k) \geq \frac{\log(d_2)}{n\sqrt{d_2 d_0}}$. Let $(\U,\V)$ be a minimizer of the dropout ERM objective in equation~\eqref{eq:dropout_obj_sensing}.
Let $\alpha$ be such that $R(\U,\V)\leq \alpha/{d_1}$. Then, for any $\delta\in(0,1)$, the following generalization bounds holds with probability at least $1-\delta$ over a sample of size $n$:
\begin{equation*}
L(g(\U\V^\top)) \leq \hat{L}(\U,\V) + 8\sqrt\frac{2\alpha d_2 \log(d_2) + \frac14 \log(2/\delta)}{n}
\end{equation*}
where $g(\M)$ thresholds $\M$ at $\pm 1$, i.e. $g(\M)(i,j) = \max\{ -1, \min\{ 1, \M(i,j) \} \},$ and $L(g(\U\V^\top)):=\E(y-\langle g(\U\V^\top),\A\rangle )^2$ is the \emph{true risk} of $g(\U\V^\top)$.
\end{theorem}

The proof of Theorem~\ref{thm:generalization_sensing_expected} follows from standard generalization bounds for $\ell_2$ loss~\citep{mohri2018foundations} based on the Rademacher complexity~\citep{bartlett2002rademacher} of the class of functions with weighted trace-norm bounded by $\sqrt{\alpha}$, i.e. $\cM_\alpha:=\{ \M: \ \| \diag(\sqrt\p) \M \diag(\sqrt\q) \|_*^2 \leq \alpha \}$. The non-degeneracy condition $\min_{i,j}p(i)q(j) \geq \frac{\log(d_2)}{n\sqrt{d_2 d_0}}$ is required to obtain a bound on the Rademacher complexity of $\cM_\alpha$, as established by~\citet{foygel2011learning}.

We note that for large enough sample size, $\hat{R}(\U, \V) \approx R(\U, \V) \approx \Theta(\U\V^\top) = \frac1{d_1}\| \diag(\sqrt p)\U\V^\top \diag(\sqrt\q) \|_*^2$, where the second approximation is due the fact that the pair $(\U,\V)$ is a minimizer. That is, compared to the weighted trace-norm, the value of the explicit regularizer at the minimizer roughly scales as $1/d_1$. Hence the assumption $\hat{R}(\U,\V)\leq \alpha/{d_1}$ in the statement of the corollary.

In practice, for models that are trained with dropout, the training error $\hat{L}(\U,\V)$ is negligible (see Figure~\ref{fig:sensing} for experiments on the MovieLens dataset). Moreover, given that the sample size is large enough, the third term can be made arbitrarily small. Having said that, the second term, which is $\tilde{O}(\sqrt{\alpha d_2/n})$, dominates the right hand side of generalization error bound in Theorem~\ref{thm:rademacher_sensing}. In Appendix, we also give optimistic generalization bounds that decay as $\tilde{O}(ad_2/n)$.

Finally, the required sample size heavily depends on the value of the explicit regularizer (i.e., $\alpha/d_1$)
at a minimizer, and hence, on the dropout rate $p$. In particular, increasing the dropout rate increases the regularization parameter $\lambda:=\frac{p}{1-p}$, thereby intensifying the penalty due to the explicit regularizer. Intuitively, a larger dropout rate $p$ results in a smaller $\alpha$, thereby a tighter generalization gap can be guaranteed. We show through experiments that that is indeed the case in practice.
\section{Non-linear Networks}\label{sec:dnn}
Next, we focus on neural networks with a single hidden layer. Let $\cX\subseteq \R^{d_0}$ and $\cY \subseteq [-1, 1]^{d_2}$ denote the input and output spaces, respectively. Let $\cD$ denote the joint probability distribution on $\cX \times \cY$. 
Given $n$ examples $\{ (\x_i,\y_i)\}_{i=1}^n \sim \cD^n$ drawn i.i.d. from the joint distribution and a loss function $\ell: \cY \times \cY \to \R$, the goal of learning is to find a hypothesis $f_\w:\cX\to\cY$, parameterized by $\w$, that has a small \emph{population risk} $L(f_\w):=\E_{\cD}[\ell(f_{\w}(\x),\y)]$. 

We focus on the squared $\ell_2$ loss, i.e., $\ell(\y,\y')= \|\y - \y'\|^2$, and study the generalization properties of the dropout algorithm for minimizing the \emph{empirical risk} ${\widehat{L}(f_\w):=\hat\E_i[\| \y_i - f_\w(\x_i) \|^2]}$. We consider the hypothesis class associated with feed-forward neural networks with $2$ layers, i.e., functions of the form $f_\w(\x)=\U \sigma(\V^\top \x )$, where $\U=[\u_1,\ldots,\u_{d_1}] \in \R^{d_2\times d_1}, \V=[\v_1,\ldots,\v_{d_1}] \in \R^{d_0\times d_1}$ are the weight matrices. The parameter $\w$ is the collection of weight matrices $\{\U,\V\}$ and $\sigma: \R \to \R$ is the ReLU activation function applied entrywise to an input vector.

As in Section~\ref{sec:sensing}, we view dropout as an instance of stochastic gradient descent on the following \emph{dropout objective}: 
\begin{equation}\label{eq:drop_obj}
{\!\hat{L}_\text{drop}(\w)} := \hat\E_i \E_{\B}\| \y_i\! -\! \U  \B\sigma(\V^\top\x_i) \|^2,
\end{equation}
where $\B$ is a diagonal random matrix  with diagonal elements distributed identically and independently as $\B_{ii}\sim \frac1{1-p} \text{Bern}(1-p), \ i \in [d_1]$, for some \emph{dropout rate} $p$. We seek to understand the \emph{explicit  regularizer} due to dropout:
\begin{equation}\label{eq:exp_reg}
\hat{R}(\w):=\hat{L}_\text{drop}(\w)-\hat{L}(\w). % \tag{explicit regularizer}
\end{equation}

We denote the output of the $i$-th hidden node on an input vector $\x$ by $a_{i}(\x)\in \R$; for example, $a_{2}(\x)=\sigma(\v_2^\top \x)$. Similarly, the vector $\a(\x)\in \R^{d_1}$ denotes the activation of the hidden layer on input $\x$. Using this notation, we can rewrite the objective in~(\ref{eq:drop_obj}) as $\hat{L}_\text{drop}(\w):=\E_i \E_{\B}\|\y_i - \U \B\a(\x_i)\|^2$. It is then easy to show that the regularizer due to dropout in~(\ref{eq:exp_reg}) is given as (see Proposition~\ref{prop:dropout_reg_dnn} in Appendix):
\begin{equation*}
\hat{R}(\w)= \frac{p}{1-p}{\sum_{j=1}^{d_1}\|\u_j\|^2 \hat{a}_j^2}, \ \textrm{where} \ \hat{a}_j=\sqrt{\hat\E_i a_{j}(\x_i)^2}.
\end{equation*}

The explicit regularizer $\hat{R}(\w)$ is a summation over hidden nodes, of the product of the squared norm of the outgoing weights with the empirical second moment of the output of the corresponding neuron. We should view it as a data-dependent variant of the $\ell_2$ path-norm of the network, studied recently by \citet{neyshabur2015norm} and shown to yield capacity control in deep learning. Indeed, 
if we consider ReLU activations and input distributions that are symmetric and isotropic~\cite{mianjy2018implicit}, the expected regularizer is equal to the sum over all paths from input to output of the product of the squares of weights along the paths, i.e.,

\begin{equation}
R(\w):=\E[\hat{R}(\w)]=\frac12\sum_{i_0,i_1,i_2=1}^{d_0,d_1,d_2}\U(i_2,i_1)^2\V(i_0,i_1)^2, \nonumber
\end{equation}
which is precisely the squared $\ell_2$ path-norm of the network. We refer the reader to Proposition~\ref{prop:reg_two_layer} in the Appendix for a formal statement and proof.

\subsection{Generalization Bounds}
To understand the generalization properties of dropout, we focus on the following distribution-dependent~class
$$\cF_{\alpha} := \{f_\w : \x \mapsto \u^\top\sigma(\V^\top\x), \  \sum_{i=1}^{d_1}| u_i | a_i \leq \alpha  \}.$$
where $\u\in\R^{d_1}$ is the top layer weight vector, $u_i$ denotes the $i$-th entry of $\u$, and $a_i^2:=\E_\x [\hat{a}_i^2] = \E_\x[a_{i}(\x)^2]$ is the expected squared neural activation for the $i$-th hidden node. For simplicity, we focus on networks with one output neuron $d_2=1$; extension to multiple output neurons is rather straightforward.

We argue that networks that are trained with dropout belong to the class $\cF_{\alpha}$, for a small value of $\alpha$. In particular, by Cauchy-Schwartz inequality, it is easy to to see that $\sum_{i=1}^{d_1}|u_i|a_i \leq \sqrt{d_1 R(\w)}$. Thus, for a fixed width, dropout implicitly controls the function class $\cF_{\alpha}$. More importantly, this inequality is loose if a small subset of hidden nodes $\cJ \subset [d_1]$ \emph{``co-adapt''} in a way that for all $j \in [d_1]\setminus \cJ$, the other hidden nodes are almost inactive, i.e. $ u_j a_j \approx 0$. In other words, by minimizing the expected regularizer, dropout is biased towards networks where the gap between $R(\w)$ and $(\sum_{i=1}^{d_1}|u_i|a_i)^2/d_1$ is small, which in turn happens if $|u_i|a_i \approx |u_j|a_j, \forall i,j \in [d_1]$. In this sense, dropout breaks \emph{``co-adaptation''} between neurons by promoting solutions with nearly equal contribution from hidden neurons.

As we mentioned in the introduction, a bound on the dropout regularizer is not sufficient to guarantee a bound on a norm-based complexity measures that are common in the deep learning literature (see, e.g. \cite{golowich2018size} and the references therein), whereas a norm bound on the weight vector would imply a bound on the explicit regularizer due to dropout. Formally, we show the following. 
\begin{proposition}\label{prop:weak_measure}
For any $C>0$, there exists a distribution on the unit Euclidean sphere, and a network $f_{\w}:\x\mapsto \sigma(\w^\top\x)$, such that $R(\w)=\sqrt{\E\sigma(\w^\top\x)^2}\leq 1$, while $\|\w\|>C$.
\end{proposition}

In other words, even though we connect the dropout regularizer to path-norm, the data-dependent nature of the regularizer prevents us from leveraging that connection in data-independent manner (i.e., for all distributions). At the same time, making strong distributional assumptions (as in Proposition~\ref{prop:reg_two_layer}) would be impractical. Instead, we argue for the following milder condition on the input distribution which we show as sufficient to ensure generalization.
\begin{assumption}[$\beta$-retentive]~\label{ass:symmetry}\normalfont
The marginal input distribution is $\beta$-retentive for some $\beta\in(0,1/2]$, if for any non-zero vector $\v \in \R^d$, it holds that $\E\sigma(\v^\top\x)^2 \geq \beta \E(\v^\top\x)^2$. 
\end{assumption}

Intuitively, what the assumption implies is that the variance (aka, the information or signal in the data) in the pre-activation at any node in the network is not quashed considerably due to the non-linearity. In fact, no reasonable training algorithm should learn weights where $\beta$ is small. However, we steer clear from algorithmic aspects of dropout training, and make the assumption above for every weight vector as we will need it when carrying out a union bound.

We now present the first main result of this section, which bounds the Rademacher complexity of $\cF_{\alpha}$ in terms of $\alpha$, the retentiveness coefficient $\beta$, and the Mahalanobis norm of the data with respect to the pseudo-inverse of the second moment, i.e. $\|\X\|_{\C^\dagger}^2=\sum_{i=1}^{n}\x_i^\top\C^{\dagger}\x_i$.

\begin{theorem}\label{thm:rademacher_dnn_new}
For  any sample $S=\{(\x_i,\y_i)\}_{i=1}^n$ of size $n$, 
$$\fR_{\cS}(\cF_{\alpha}) \leq \frac{2\alpha \| \X\|_{\C^\dagger}}{n\sqrt{\beta}}.$$ 
Furthermore, it holds for the expected Rademacher complexity that $\fR_n(\cF_{\alpha}) \leq 2\alpha \sqrt\frac{\rank(\C)}{\beta n}.$
\end{theorem}

First, note that the bound depends on the quantity $\|\X\|_{\C^\dagger}$ which can be in the same order as $\|\X\|_F$ with both scaling as $\asymp \sqrt{n d_0}$; the latter is more common in the literature~\cite{neyshabur2018towards,bartlett2017spectrally,neyshabur2017pac,golowich2018size,neyshabur2015norm}. This is unfortunately unavoidable, unless one makes stronger distributional assumptions. 

Second, as we discussed earlier, the dropout regularizer directly controls the value of $\alpha$, thereby controlling the Rademacher complexity in Theorem~\ref{thm:rademacher_dnn_new}. This bound also gives us a bound on the Rademacher complexity of the networks trained using dropout. To see that, consider the following class of networks with bounded explicit regularizer, i.e., $\cH_r:=\{ h_\w:\x\mapsto \u^\top\sigma(\V^\top\x), \ R(\u,\V)\leq r \}$. Then, Theorem~\ref{thm:rademacher_dnn_new} yields $\fR_{\cS}(\cH_r) \leq \frac{2\sqrt{d_1 r} \| \X\|_{\C^\dagger}}{n\sqrt{\beta}}.$ In fact, we can show that this bound is tight up to $1/\sqrt{\beta}$ by a reduction to the linear case. Formally, we show the following. 

\begin{theorem}[Lowerbound on $\fR_{\cS}$]\label{thm:lb_rademacher} 
There is a constant $c$ such that for any scalar $r > 0$,
$\fR_{\cS}(\cH_r) \geq \frac{c\sqrt{d_1 r} \| \X\|_{\C^\dagger}}{n}.$
\end{theorem}
Moreover, it is easy to give a generalization bound based on Theorem~\ref{thm:rademacher_dnn_new} that depends only on the distribution dependent quantities $\alpha$ and $\beta$. Let $g_\w(\cdot) := \max\{ -1, \min\{1, f_\w(\cdot)\} \}$ project the network output $f_\w$ onto the range $[-1, 1]$. We have the following generalization gurantees for $g_\w$.

\begin{corollary}\label{cor:gen_bound_beta_regression}
For any $\w\in \cF_{\alpha}$, for any $\delta\in(0,1)$, the following generalization bound holds with probability at least $1-\delta$ over a sample $\cS$ of size $n$
\begin{equation*}
L_\cD(g_\w) \leq \hat{L}_\cS(g_\w) + \frac{16\alpha\|\X\|_{\C^\dagger}}{\sqrt{\beta}n}+12\sqrt\frac{\log(2/\delta)}{2n}
\end{equation*}
\end{corollary}

\paragraph{$\beta$-independent Bounds.} 
Geometrically, $\beta$-retentiveness requires that for any hyperplane passing through the origin, both halfspaces contribute significantly to the second moment of the data in the direction of the normal vector. 
It is not clear, however, if $\beta$ can be estimated efficiently, given a dataset. Nonetheless, when $\cX\subseteq\R_+^{d_0}$, which is the case for image datasets, a simple \emph{symmetrization} technique, described below, allows us to give bounds that are $\beta$-independent; note that the bound still depend on the sample as $\|\X\|_{\C^\dagger}$. Here is the randomized symmetrization we propose. \emph{Given a training sample $\cS=\{(\x_i,y_i),\ i \in [n]\}$, consider the following randomized perturbation, $\cS'=\{(\zeta_i\x_i,y_i),\ i \in [n]\}$, where $\zeta_i$'s are i.i.d. Rademacher random  variables.} We give a generalization bound (w.r.t. the original data distribution) for the hypothesis class with bounded regularizer w.r.t. the perturbed data distribution. 
\begin{corollary}\label{cor:gen_bound_nobeta_regression}
Given an i.i.d. sample $\cS=\{(\x_i, \y_i)\}_{i=1}^n$, let
$$\cF'_{\alpha} := \{f_\w : \x \mapsto \u^\top\sigma(\V^\top\x), \  \sum_{i=1}^{d_1}| u_i | a_i' \leq \alpha  \},$$
where ${a_i'}^2:=\E_{\x,\zeta}[a_{i}(\zeta\x)^2]$. For any $\w\in \cF'_{\alpha}$, for any $\delta\in(0,1)$, the following generalization bound holds with probability at least $1-\delta$ over a sample of size $n$ and the randomization in symmetrization
\begin{align*}
L_\cD(g_\w) \leq 2\hat{L}_{\cS'}(g_\w) + \frac{46\alpha \|\X\|_{\C^\dagger}}{n}+24\sqrt\frac{\log(2/\delta)}{2n}.
\end{align*}
\end{corollary}
Note that the population risk of the
clipped predictor $g_\w(\cdot) := \max\{ -1, \min\{1, f_\w(\cdot)\} \}$ is bounded in terms of the empirical risk on $\cS'$. Finally, we verify in Section~\ref{sec:exp} that symmetrization of the training set, on MNIST and FashionMNIST datasets, does not have an effect on the performance of the trained models.

\begin{table*}[t]
\centering
{\begin{tabular}{c || c | c || c | c | c | c}
    & \multicolumn{2}{c}{plain SGD} & \multicolumn{4}{c}{dropout} \\
    \textbf{width} & last iterate & best iterate & $p=0.1$ & $p=0.2$ & $p=0.3$ & $p=0.4$\\ \hline
    $d_1 = 30$  & 0.8041 & 0.7938 & \cellcolor{gray}{0.7805} & \cellcolor{gray!30}{0.785} &  0.7991 & 0.8186 \\
    $d_1 = 70$  & 0.8315 & 0.7897 & 0.7899 & \cellcolor{gray!30}{0.7771} & \cellcolor{gray}{0.7763} & 0.7833 \\
    $d_1 = 110$  & 0.8431 & 0.7873 & 0.7988 & 0.7813 & \cellcolor{gray}{0.7742} & \cellcolor{gray!30}{0.7743} \\
    $d_1 = 150$  & 0.8472 & 0.7858 & 0.8042 & 0.7852 & \cellcolor{gray!30}{0.7756} & \cellcolor{gray}{0.7722} \\
    $d_1 = 190$  & 0.8473 & 0.7844 &  0.8069 & 0.7879 & \cellcolor{gray!30}{0.7772} & \cellcolor{gray}{0.772} \\
    \end{tabular}}
    \caption{\small MovieLens dataset: Test RMSE of plain SGD as well as the dropout algorithm with various dropout rates for various factorization sizes. The grey cells shows the best performance(s) in each row.
    \label{tab:movielens}}
\end{table*}

\section{Role of Parametrization}
\label{sec:regression}
In this section, we argue that parametrization plays an important role in determining the nature of the inductive bias. 

We begin by considering matrix sensing in non-factorized form, which entails minimizing $\hat{L}(\M):=\hat\E_i(y_i - \langle \vect{\M}, \textrm{vec}\big(\A^{(i)}\big) \rangle)^2$,
where $\vect{\M}$ denotes the column vectorization of $\M$. Then, the expected explicit regularizer due to dropout equals $R(\M) = \frac{p}{1-p} \| \vect{\M} \|_{\diag(\C)}^2$, where $\C = \E[\vect{\A}\vect{\A}^\top]$ is the second moment of the measurement matrices. For instance, with Gaussian measurements, the second moment equals the identity matrix, in which case, the regularizer reduces to the Frobenius norm of the parameters ${R}(\M) = \frac{p}{1-p}\| \M\|_F^2$. While such a ridge penalty yields a useful inductive bias in linear regression, it is not \emph{``rich''} enough to capture the kind of inductive bias that provides rank control in matrix sensing. 

However, simply representing the hypotheses in a factored form alone is not sufficient in terms of imparting a rich inductive bias to the learning problem. Recall that in linear regression, dropout, when applied on the input features, yields ridge regularization. However, if we were to represent the linear predictor in terms of a deep linear network, then we argue that the effect of dropout is markedly different. Consider a deep linear network, $f_\w:\x \mapsto \W_k\cdots\W_1\x$ with a single output neuron. In this case,~\citet{mianjy2019dropout} show that $\nu\| f \|_{\hat\C}^2 = \minim{ f_\w = f}{\hat{R}(\w)},$ where $\nu$ is a regularization parameter independent of the parameters $\w$. Consequently, in deep linear networks with a single output neuron, dropout reduces to solving 
\begin{equation*}
\minim{\u\in \R^{d_0}}{\hat\E_i(y_i-\u^\top\x_i)^2 + \nu \|\u\|_{\hat\C}^2}. 
\end{equation*}
All the minimizers of the above problem are solutions to the system of linear equations $(1+\frac{\nu}n)\X\X^\top\u = \X\y$, where $\X = [\x_1,\ldots,\x_n] \in \R^{d_0 \times n}, \y = [y_1; \ldots; y_n] \in \R^{n}$ are the design matrix and the response vector, respectively. In other words, the dropout regularizer manifests itself merely as a scaling of the parameters which does not seem to be useful in any meaningful way.

What we argue above may at first seem to contradict the results of Section~\ref{sec:sensing} on matrix sensing, which is arguably an instance of regression with a two-layer linear network. Note though that casting matrix sensing in a factored form as a linear regression problem requires us to use a convolutional structure. This is easy to check since 
\begin{align*}
\langle \U\V^\top, \A \rangle &= \langle \vect{\U^\top}, \vect{\V^\top\A^\top}\rangle \\
&= \langle \vect{\U^\top}, (\I_{d_2} \otimes \V^\top) \vect{\A^\top}\rangle, 
\end{align*}
where $\otimes$ is the Kronecker product, and we used the fact that $\vect{\A\B}=(\I \otimes \A) \vect{\B}$ for any pair of matrices $\A,\B$. The expression $(\I \otimes \V^\top)$ represents a fully connected convolutional layer with $d_1$ filters specified by columns of $\V$. The convolutional structure in addition to dropout is what imparts the problem of matrix sensing the nuclear norm regularization. For nonlinear networks, however, a simple feed-forward structure suffices as we saw in Section~\ref{sec:dnn}. 

\section{Experimental Results}\label{sec:exp}
In this section, we report our empirical findings on real world datasets. All results are averaged over 50 independent runs with random initialization.

\begin{figure*}[t]
\centering
\begin{tabular}{ccc}
\hspace*{-20pt} 
\includegraphics[width=0.365\textwidth]{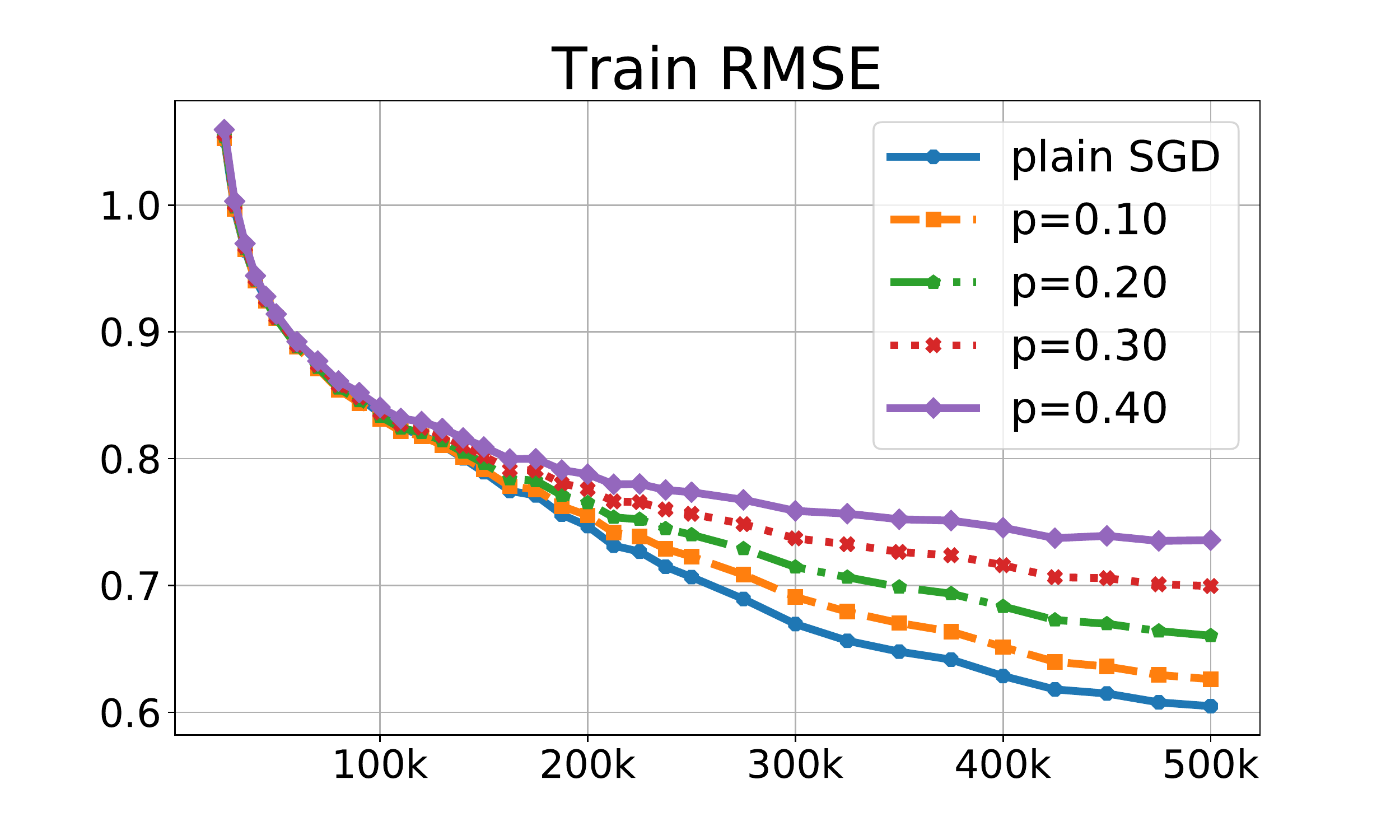}
&
\hspace*{-22pt} 
\includegraphics[width=0.365\textwidth]{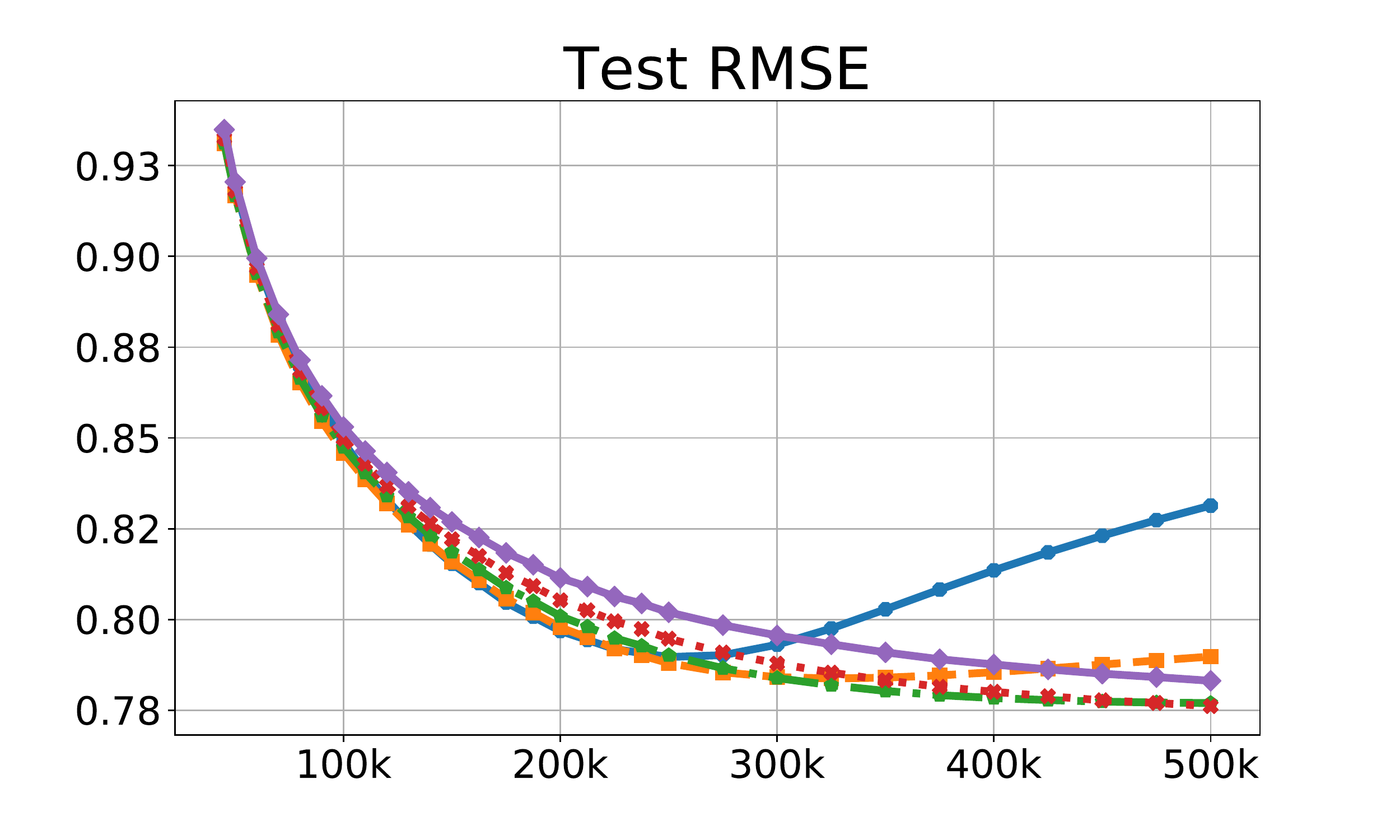}
&
\hspace*{-22pt} 
\includegraphics[width=0.365\textwidth]{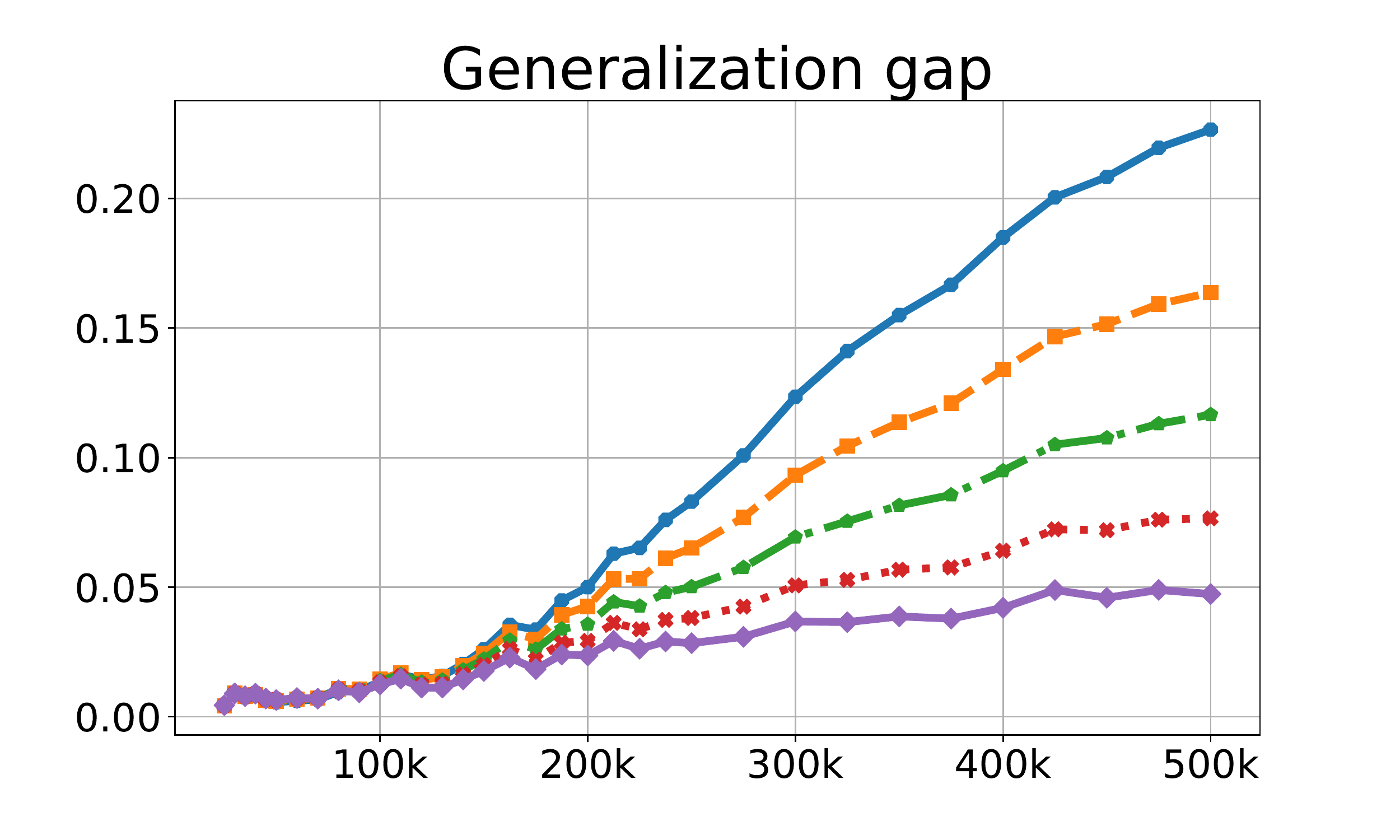}
\vspace*{-15pt} 
\end{tabular}
\caption{\small{MovieLens dataset: the training error ({\bf left}), the test error ({\bf middle}), and the generalization gap ({\bf right}) for plain SGD and dropout with $p \in \{ 0.1, 0.2, 0.3, 0.4 \}$ as a function of the number of iterations. The factorization size is $d_1 = 70$.}}
\label{fig:sensing}
\end{figure*}

\vspace*{-5pt}
\subsection{Matrix Completion}
We evaluate dropout on the MovieLens dataset~\cite{harper2016movielens}, a publicly available collaborative filtering dataset that contains 10M ratings for 11K movies by 72K users of the online movie recommender~service~MovieLens.

We initialize the factors using the standard He initialization scheme. We train the model for 100 epochs over the training data, where we use a fixed learning rate of $\texttt{lr} = 1$, and a batch size of $2000$. We report the results for plain SGD ($p=0.0$) as well as the dropout algorithm with $p \in \{ 0.1, 0.2, 0.3, 0.4 \}$.

Figure~\ref{fig:sensing} shows the progress in terms of the training and test error as well as the gap between them as a function of the number of iterations for $d_1 = 70$. It can be seen that plain SGD is the fastest in minimizing the empirical risk. The dropout rate clearly determines the trade-off between the approximation error and the estimation error: as the dropout rate $p$ increases, the algorithm favors less complex solutions that suffer larger empirical error (left figure) but enjoy smaller generalization gap (right figure). The best trade-off here seems to be achieved by a moderate dropout rate of $p=0.3$. We observe similar behaviour for different factorization sizes; please see the Appendix for additional plots with factorization sizes $d_1 \in \{ 30, 110, 150, 190 \}$.

It is remarkable, how even in the ``simple'' problem of matrix completion, plain SGD lacks a proper inductive bias. As seen in the middle plot, without \emph{explicit} regularization, in particular, without early stopping or dropout, SGD starts overfitting. We further illustrate this in Table~\ref{tab:movielens}, where we compare the test root-mean-squared-error (RMSE) of plain SGD with the dropout algorithm, for various factorization sizes. To show the superiority of dropout over SGD with early stopping, we give SGD the advantage of having access to the \emph{test set} (and not a separate validation set), and report the best iterate in the third column. Even with this impractical privilege, dropout performs better ($> 0.01$ difference in test RMSE).
\vspace{-5pt}

\subsection{Neural Networks}
\begin{figure*}
\centering
\begin{tabular}{ccc}
\hspace*{-20pt} 
\includegraphics[width=0.365\textwidth]{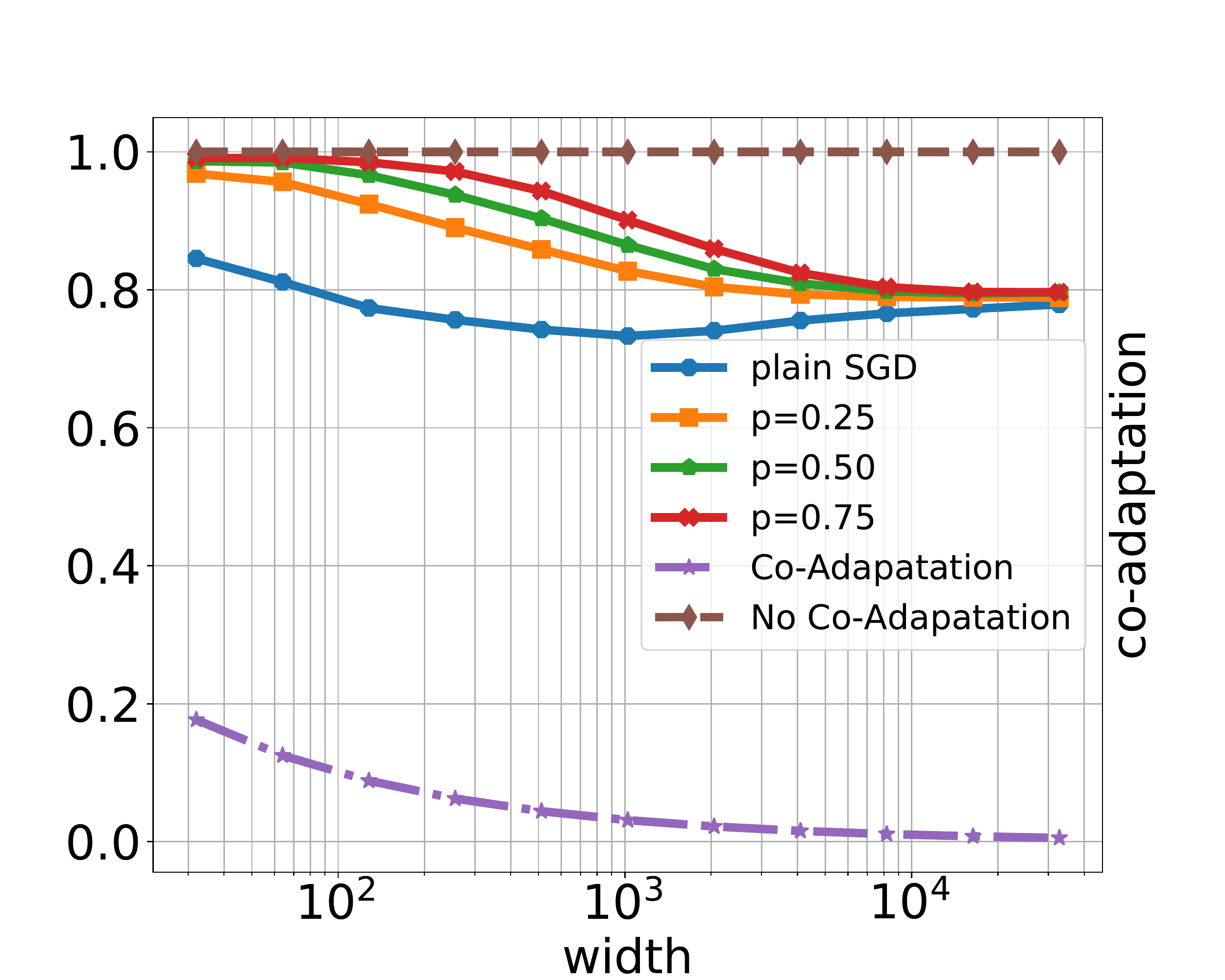}
&
\hspace*{-22pt} 
\includegraphics[width=0.365\textwidth]{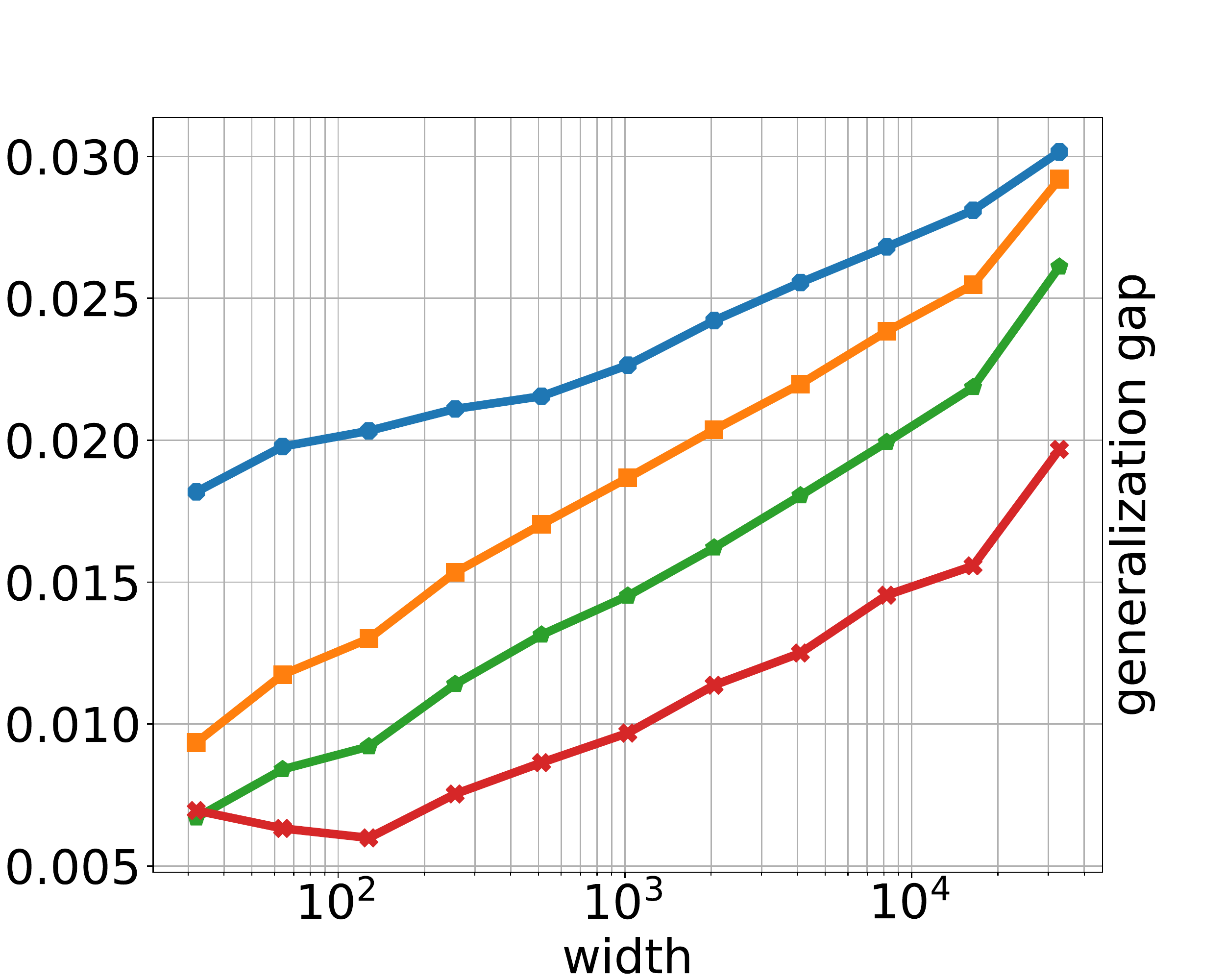}
&
\hspace*{-22pt} 
\includegraphics[width=0.365\textwidth]{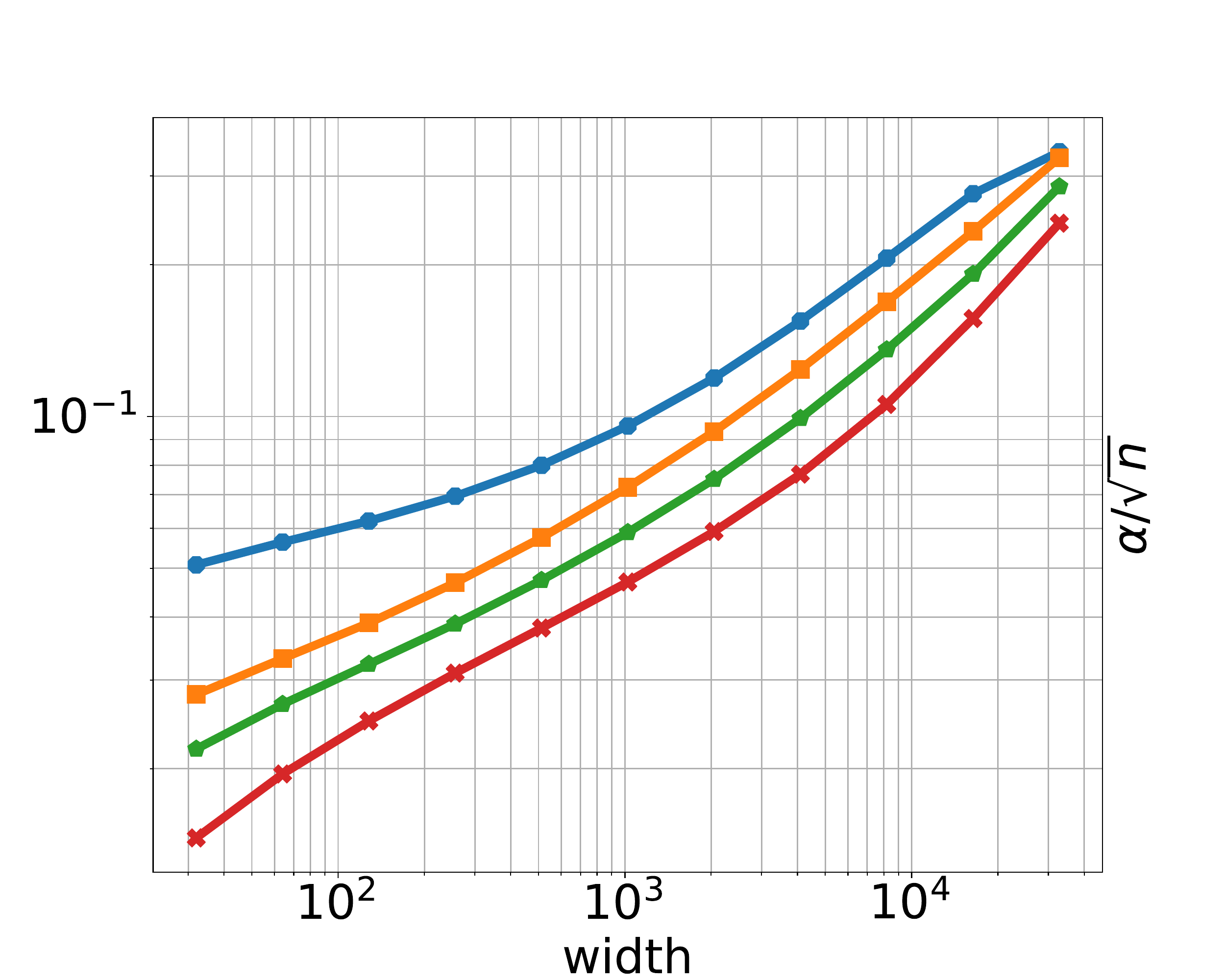}
\vspace*{-15pt} 
\end{tabular}
\caption{\label{fig:2nn}({\bf left}) \emph{``co-adaptation''}; ({\bf middle}) generalization gap; and ({\bf right}) $\alpha/\sqrt{n}$ as a function of the width of networks trained with dropout on MNIST. In left figure, the dashed brown and dotted purple lines represent minimal and maximal co-adaptations, respectively.}
\end{figure*}
We train 2-layer neural networks with and without dropout, on MNIST dataset of handwritten digits and Fashion MNIST dataset of Zalando's article images, each of which contains 60K training examples and 10K test examples, where each example is a $28\times 28$ grayscale image associated with a label from $10$ classes. We extract two classes $\{4, 7\}$ and label them as $\{ -1, +1 \}$~\footnote{We observe similar results across other choices of target classes.}. The learning rate in all experiments is set to $\texttt{lr}=1e-3$. We train the models for 30 epochs over the training set. We run the experiments \emph{both with and without symmetrization}. Here we only report the results with symmetrization, and on the MNIST dataset. For experiments without symmetrization, and experiments on FashionMNIST, please see the Appendix. We remark that \emph{under the above experimental setting, the trained networks achieve $100\%$ training accuracy.}

For any node $i\in[d_1]$, we define its \emph{flow} as $\psi_i:=|u_i|a_i$ (respectively $\psi_i:=|u_i|a_i'$ for symmetrized data), which measures the overall contribution of a node to the output of the network. Co-adaptation occurs when a small subset of nodes dominate the overall function of the network. We argue that $\phi(\w)=\frac{\|\psi\|_1}{\sqrt{d_1}\|\psi\|_2}$ is a suitable measure of co-adaptation (or lack thereof) in a network parameterized by $\w$. In case of high co-adaptation, only a few nodes have a high flow, which implies $\phi(\w) \approx \frac{1}{\sqrt{d_1}}$. At the other end of the spectrum, all nodes are equally active, in which case $\phi(\w) \approx 1$. Figure~\ref{fig:2nn} (left) illustrates this measure as a function of the network width for several dropout rates $p\in\{ 0, 0.25, 0.5, 0.75 \}$. In particular, we observe that a higher dropout rate corresponds to less co-adapation. More interestingly, even plain SGD is \emph{implicitly} biased towards networks with less co-adapation. Moreover, for a fixed dropout rate, the regularization effect due to dropout decreases as we increase the width. Thus, it is natural to expect more co-adaptation as the network becomes wider, which is what we observe in the plots.

The generalization gap is plotted in Figure~\ref{fig:2nn} (middle). As expected, increasing dropout rate decreases the generalization gap, uniformly for all widths. In our experiments, the generalization gap increases with the width of the network. The figure on the right shows the quantity $\alpha/\sqrt{n}$ that shows up in the Rademacher complexity bounds in Section~\ref{sec:dnn}. We note that, the bound on the Rademacher complexity is predictive of the generalization gap, in the sense that a smaller bound corresponds to a curve with smaller generalization gap. 
\vspace{-10pt}
\section{Conclusion}\label{sec:disc}
Motivated by the success of dropout in deep learning, we study a dropout algorithm for matrix sensing and show that it enjoys strong generalization guarantees as well as competitive test performance on the MovieLens dataset. We then focus on deep regression under the squared loss and show that the regularizer due to dropout serves as a strong complexity measure for the underlying class of neural networks, using which we give a generalization error bound in terms of the value of the regularizer.

\section*{Acknowledgements}
This research was supported, in part, by NSF BIGDATA award IIS-1546482 and NSF CAREER award IIS-1943251. The seeds of this work were sown during the summer 2019 workshop on the Foundations of Deep Learning at the Simons Institute for the Theory of Computing. Raman Arora acknowledges the support provided by the Institute for Advanced Study, Princeton, New Jersey as part of the special year on Optimization, Statistics, and Theoretical Machine Learning.
\bibliographystyle{icml2020}
\bibliography{references.bib}

\begin{thebibliography}{44}
\providecommand{\natexlab}[1]{#1}
\providecommand{\url}[1]{\texttt{#1}}
\expandafter\ifx\csname urlstyle\endcsname\relax
  \providecommand{\doi}[1]{doi: #1}\else
  \providecommand{\doi}{doi: \begingroup \urlstyle{rm}\Url}\fi

\bibitem[Baldi \& Sadowski(2013)Baldi and Sadowski]{baldi2013understanding}
Baldi, P. and Sadowski, P.~J.
\newblock Understanding dropout.
\newblock In \emph{Advances in Neural Information Processing Systems (NIPS)},
  pp.\  2814--2822, 2013.

\bibitem[Bank \& Giryes(2018)Bank and Giryes]{bank2018relationship}
Bank, D. and Giryes, R.
\newblock On the relationship between dropout and equiangular tight frames.
\newblock \emph{arXiv preprint arXiv:1810.06049}, 2018.

\bibitem[Bartlett \& Mendelson(2002)Bartlett and
  Mendelson]{bartlett2002rademacher}
Bartlett, P.~L. and Mendelson, S.
\newblock Rademacher and gaussian complexities: Risk bounds and structural
  results.
\newblock \emph{Journal of Machine Learning Research}, 3\penalty0
  (Nov):\penalty0 463--482, 2002.

\bibitem[Bartlett et~al.(2017)Bartlett, Foster, and
  Telgarsky]{bartlett2017spectrally}
Bartlett, P.~L., Foster, D.~J., and Telgarsky, M.~J.
\newblock Spectrally-normalized margin bounds for neural networks.
\newblock In \emph{Advances in Neural Information Processing Systems}, pp.\
  6240--6249, 2017.

\bibitem[Cavazza et~al.(2018)Cavazza, Haeffele, Lane, Morerio, Murino, and
  Vidal]{cavazza2018dropout}
Cavazza, J., Haeffele, B.~D., Lane, C., Morerio, P., Murino, V., and Vidal, R.
\newblock Dropout as a low-rank regularizer for matrix factorization.
\newblock \emph{Int. Conf. on Artificial Intelligence and Statistics
  (AISTATS)}, 2018.

\bibitem[Chen \& Manning(2014)Chen and Manning]{chen2014fast}
Chen, D. and Manning, C.
\newblock A fast and accurate dependency parser using neural networks.
\newblock In \emph{Proceedings of the 2014 conference on empirical methods in
  natural language processing (EMNLP)}, pp.\  740--750, 2014.

\bibitem[Dahl et~al.(2013)Dahl, Sainath, and Hinton]{dahl2013improving}
Dahl, G.~E., Sainath, T.~N., and Hinton, G.~E.
\newblock Improving deep neural networks for lvcsr using rectified linear units
  and dropout.
\newblock In \emph{2013 IEEE international conference on acoustics, speech and
  signal processing}, pp.\  8609--8613. IEEE, 2013.

\bibitem[Foygel et~al.(2011)Foygel, Shamir, Srebro, and
  Salakhutdinov]{foygel2011learning}
Foygel, R., Shamir, O., Srebro, N., and Salakhutdinov, R.~R.
\newblock Learning with the weighted trace-norm under arbitrary sampling
  distributions.
\newblock In \emph{Advances in Neural Information Processing Systems}, pp.\
  2133--2141, 2011.

\bibitem[Gal \& Ghahramani(2016)Gal and Ghahramani]{gal2016dropout}
Gal, Y. and Ghahramani, Z.
\newblock Dropout as a bayesian approximation: Representing model uncertainty
  in deep learning.
\newblock In \emph{Int. Conf. Machine Learning (ICML)}, 2016.

\bibitem[Gao \& Zhou(2016)Gao and Zhou]{gao2016dropout}
Gao, W. and Zhou, Z.-H.
\newblock Dropout rademacher complexity of deep neural networks.
\newblock \emph{Science China Information Sciences}, 59\penalty0 (7):\penalty0
  072104, 2016.

\bibitem[Golowich et~al.(2018)Golowich, Rakhlin, and Shamir]{golowich2018size}
Golowich, N., Rakhlin, A., and Shamir, O.
\newblock Size-independent sample complexity of neural networks.
\newblock In \emph{Conference On Learning Theory}, pp.\  297--299, 2018.

\bibitem[Gunasekar et~al.(2017)Gunasekar, Woodworth, Bhojanapalli, Neyshabur,
  and Srebro]{gunasekar2017implicit}
Gunasekar, S., Woodworth, B.~E., Bhojanapalli, S., Neyshabur, B., and Srebro,
  N.
\newblock Implicit regularization in matrix factorization.
\newblock In \emph{Advances in Neural Information Processing Systems}, pp.\
  6151--6159, 2017.

\bibitem[Harper \& Konstan(2016)Harper and Konstan]{harper2016movielens}
Harper, F.~M. and Konstan, J.~A.
\newblock The movielens datasets: History and context.
\newblock \emph{Acm transactions on interactive intelligent systems (tiis)},
  5\penalty0 (4):\penalty0 19, 2016.

\bibitem[Havaei et~al.(2017)Havaei, Davy, Warde-Farley, Biard, Courville,
  Bengio, Pal, Jodoin, and Larochelle]{havaei2017brain}
Havaei, M., Davy, A., Warde-Farley, D., Biard, A., Courville, A., Bengio, Y.,
  Pal, C., Jodoin, P.-M., and Larochelle, H.
\newblock Brain tumor segmentation with deep neural networks.
\newblock \emph{Medical image analysis}, 35:\penalty0 18--31, 2017.

\bibitem[Helmbold \& Long(2015)Helmbold and Long]{helmbold2015inductive}
Helmbold, D.~P. and Long, P.~M.
\newblock On the inductive bias of dropout.
\newblock \emph{Journal of Machine Learning Research (JMLR)}, 16:\penalty0
  3403--3454, 2015.

\bibitem[Helmbold \& Long(2017)Helmbold and Long]{helmbold2017surprising}
Helmbold, D.~P. and Long, P.~M.
\newblock Surprising properties of dropout in deep networks.
\newblock \emph{The Journal of Machine Learning Research}, 18\penalty0
  (1):\penalty0 7284--7311, 2017.

\bibitem[Hinton et~al.(2012)Hinton, Srivastava, Krizhevsky, Sutskever, and
  Salakhutdinov]{hinton2012improving}
Hinton, G.~E., Srivastava, N., Krizhevsky, A., Sutskever, I., and
  Salakhutdinov, R.~R.
\newblock Improving neural networks by preventing co-adaptation of feature
  detectors.
\newblock \emph{arXiv preprint arXiv:1207.0580}, 2012.

\bibitem[Kalchbrenner et~al.(2014)Kalchbrenner, Grefenstette, and
  Blunsom]{kalchbrenner2014convolutional}
Kalchbrenner, N., Grefenstette, E., and Blunsom, P.
\newblock A convolutional neural network for modelling sentences.
\newblock \emph{arXiv preprint arXiv:1404.2188}, 2014.

\bibitem[Krizhevsky et~al.(2009)Krizhevsky, Hinton,
  et~al.]{krizhevsky2009learning}
Krizhevsky, A., Hinton, G., et~al.
\newblock Learning multiple layers of features from tiny images.
\newblock Technical report, Citeseer, 2009.

\bibitem[Krizhevsky et~al.(2012)Krizhevsky, Sutskever, and
  Hinton]{krizhevsky2012imagenet}
Krizhevsky, A., Sutskever, I., and Hinton, G.~E.
\newblock Imagenet classification with deep convolutional neural networks.
\newblock In \emph{Advances in neural information processing systems}, pp.\
  1097--1105, 2012.

\bibitem[Li et~al.(2018)Li, Ma, and Zhang]{li2018algorithmic}
Li, Y., Ma, T., and Zhang, H.
\newblock Algorithmic regularization in over-parameterized matrix sensing and
  neural networks with quadratic activations.
\newblock In \emph{Conference On Learning Theory}, pp.\  2--47, 2018.

\bibitem[Li et~al.(2016)Li, Gong, and Yang]{li2016improved}
Li, Z., Gong, B., and Yang, T.
\newblock Improved dropout for shallow and deep learning.
\newblock In \emph{Advances in neural information processing systems}, pp.\
  2523--2531, 2016.

\bibitem[McAllester(2013)]{mcallester2013pac}
McAllester, D.
\newblock A pac-bayesian tutorial with a dropout bound.
\newblock \emph{arXiv preprint arXiv:1307.2118}, 2013.

\bibitem[Mianjy \& Arora(2019)Mianjy and Arora]{mianjy2019dropout}
Mianjy, P. and Arora, R.
\newblock On dropout and nuclear norm regularization.
\newblock In \emph{International Conference on Machine Learning}, 2019.

\bibitem[Mianjy et~al.(2018)Mianjy, Arora, and Vidal]{mianjy2018implicit}
Mianjy, P., Arora, R., and Vidal, R.
\newblock On the implicit bias of dropout.
\newblock In \emph{International Conference on Machine Learning}, pp.\
  3537--3545, 2018.

\bibitem[Mohri et~al.(2018)Mohri, Rostamizadeh, and
  Talwalkar]{mohri2018foundations}
Mohri, M., Rostamizadeh, A., and Talwalkar, A.
\newblock \emph{Foundations of machine learning}.
\newblock MIT press, 2018.

\bibitem[Mou et~al.(2018)Mou, Zhou, Gao, and Wang]{mou2018dropout}
Mou, W., Zhou, Y., Gao, J., and Wang, L.
\newblock Dropout training, data-dependent regularization, and generalization
  bounds.
\newblock In \emph{International Conference on Machine Learning}, pp.\
  3642--3650, 2018.

\bibitem[Neyshabur et~al.(2015{\natexlab{a}})Neyshabur, Salakhutdinov, and
  Srebro]{neyshabur2015path}
Neyshabur, B., Salakhutdinov, R.~R., and Srebro, N.
\newblock Path-sgd: Path-normalized optimization in deep neural networks.
\newblock In \emph{Advances in Neural Information Processing Systems}, pp.\
  2422--2430, 2015{\natexlab{a}}.

\bibitem[Neyshabur et~al.(2015{\natexlab{b}})Neyshabur, Tomioka, and
  Srebro]{neyshabur2015norm}
Neyshabur, B., Tomioka, R., and Srebro, N.
\newblock Norm-based capacity control in neural networks.
\newblock In \emph{Conference on Learning Theory}, pp.\  1376--1401,
  2015{\natexlab{b}}.

\bibitem[Neyshabur et~al.(2017)Neyshabur, Bhojanapalli, and
  Srebro]{neyshabur2017pac}
Neyshabur, B., Bhojanapalli, S., and Srebro, N.
\newblock A pac-bayesian approach to spectrally-normalized margin bounds for
  neural networks.
\newblock \emph{arXiv preprint arXiv:1707.09564}, 2017.

\bibitem[Neyshabur et~al.(2018)Neyshabur, Li, Bhojanapalli, LeCun, and
  Srebro]{neyshabur2018towards}
Neyshabur, B., Li, Z., Bhojanapalli, S., LeCun, Y., and Srebro, N.
\newblock Towards understanding the role of over-parametrization in
  generalization of neural networks.
\newblock \emph{arXiv preprint arXiv:1805.12076}, 2018.

\bibitem[Pham et~al.(2014)Pham, Bluche, Kermorvant, and
  Louradour]{pham2014dropout}
Pham, V., Bluche, T., Kermorvant, C., and Louradour, J.
\newblock Dropout improves recurrent neural networks for handwriting
  recognition.
\newblock In \emph{2014 14th international conference on frontiers in
  handwriting recognition}, pp.\  285--290. IEEE, 2014.

\bibitem[Srebro \& Salakhutdinov(2010)Srebro and
  Salakhutdinov]{srebro2010collaborative}
Srebro, N. and Salakhutdinov, R.~R.
\newblock Collaborative filtering in a non-uniform world: Learning with the
  weighted trace norm.
\newblock In \emph{Advances in Neural Information Processing Systems}, pp.\
  2056--2064, 2010.

\bibitem[Srebro et~al.(2010)Srebro, Sridharan, and
  Tewari]{srebro2010optimistic}
Srebro, N., Sridharan, K., and Tewari, A.
\newblock Optimistic rates for learning with a smooth loss.
\newblock \emph{arXiv preprint arXiv:1009.3896}, 2010.

\bibitem[Srivastava et~al.(2014)Srivastava, Hinton, Krizhevsky, Sutskever, and
  Salakhutdinov]{srivastava2014dropout}
Srivastava, N., Hinton, G.~E., Krizhevsky, A., Sutskever, I., and
  Salakhutdinov, R.
\newblock Dropout: a simple way to prevent neural networks from overfitting.
\newblock \emph{Journal of Machine Learning Research (JMLR)}, 15\penalty0 (1),
  2014.

\bibitem[Szegedy et~al.(2015)Szegedy, Liu, Jia, Sermanet, Reed, Anguelov,
  Erhan, Vanhoucke, and Rabinovich]{szegedy2015going}
Szegedy, C., Liu, W., Jia, Y., Sermanet, P., Reed, S., Anguelov, D., Erhan, D.,
  Vanhoucke, V., and Rabinovich, A.
\newblock Going deeper with convolutions.
\newblock In \emph{Proceedings of the IEEE conference on computer vision and
  pattern recognition}, pp.\  1--9, 2015.

\bibitem[Toshev \& Szegedy(2014)Toshev and Szegedy]{toshev2014human}
Toshev, A. and Szegedy, C.
\newblock Deeppose: Human pose estimation via deep neural networks.
\newblock In \emph{The IEEE Conference on Computer Vision and Pattern
  Recognition (CVPR)}, June 2014.

\bibitem[Vershynin(2018)]{vershynin2018high}
Vershynin, R.
\newblock \emph{High-dimensional probability: An introduction with applications
  in data science}, volume~47.
\newblock Cambridge University Press, 2018.

\bibitem[Wager et~al.(2013)Wager, Wang, and Liang]{wager2013dropout}
Wager, S., Wang, S., and Liang, P.~S.
\newblock Dropout training as adaptive regularization.
\newblock In \emph{Advances in Neural Information Processing Systems (NIPS)},
  2013.

\bibitem[Wager et~al.(2014)Wager, Fithian, Wang, and Liang]{wager2014altitude}
Wager, S., Fithian, W., Wang, S., and Liang, P.~S.
\newblock Altitude training: Strong bounds for single-layer dropout.
\newblock In \emph{Adv. Neural Information Processing Systems}, 2014.

\bibitem[Wan et~al.(2013)Wan, Zeiler, Zhang, Le~Cun, and
  Fergus]{wan2013regularization}
Wan, L., Zeiler, M., Zhang, S., Le~Cun, Y., and Fergus, R.
\newblock Regularization of neural networks using dropconnect.
\newblock In \emph{International conference on machine learning}, pp.\
  1058--1066, 2013.

\bibitem[Wang \& Manning(2013)Wang and Manning]{wang2013fast}
Wang, S. and Manning, C.
\newblock Fast dropout training.
\newblock In \emph{international conference on machine learning}, pp.\
  118--126, 2013.

\bibitem[Yang et~al.(2016)Yang, He, Gao, Deng, and Smola]{yang2016stacked}
Yang, Z., He, X., Gao, J., Deng, L., and Smola, A.
\newblock Stacked attention networks for image question answering.
\newblock In \emph{Proceedings of the IEEE conference on computer vision and
  pattern recognition}, pp.\  21--29, 2016.

\bibitem[Zhai \& Wang(2018)Zhai and Wang]{zhai2018adaptive}
Zhai, K. and Wang, H.
\newblock Adaptive dropout with rademacher complexity regularization.
\newblock In \emph{International Conference on Learning Representations}, 2018.
\newblock URL \url{https://openreview.net/forum?id=S1uxsye0Z}.

\end{thebibliography}


\clearpage
\appendix

\onecolumn
\setcounter{figure}{0}
\setcounter{table}{0}
\begin{center}
{\bf\LARGE Supplementary Materials for \\ ``Dropout: Explicit Forms and Capacity Control''}
\end{center}
\vspace{30pt}
\section{Auxiliary Results}

\begin{lemma}[Khintchine-Kahane inequality]\label{lem:Khintchine-Kahane}
Let $\{ \epsilon_i\}_{i=1}^n$ be i.i.d. Rademacher random variables, and $\{ \x\}_{i=1}^n \subset \R^d$. Then there exist a universal constants $c>0$ such that $$ \E\| \sum_{i=1}^n \epsilon_i \x_i \| \geq c \sqrt{\sum_{i=1}^n \|\x_i\|^2}.$$
\end{lemma}

\begin{theorem}[Hoeffding's inequality: Theorem 2.6.2~\cite{vershynin2018high}]\label{thm:hoeffding}
Let $X_1, \ldots , X_N$ be independent, mean zero, sub-Gaussian random variables. Then, for every $t \geq  0$, we have
\begin{align*}
\bP\left( \abs{ \hat\E_i{X_i}} \geq t \right) \leq 2e^{ -\frac{ct^2N^2}{\sum_{i=1}^{N}\| X_i \|_{\psi_2}^2}}
\end{align*}
\end{theorem}

\begin{theorem}[Theorem~3.1 of~\cite{mohri2018foundations}]\label{thm:generalization_classification}
Let $\cG$ be a family of functions mapping from $\cZ$ to $[0, 1]$. Then, for any $\delta > 0$, with probability at least $1 - \delta$ over a sample $\cS=\{\z_1,\ldots,\z_n\}$, the following holds for all $g\in \cG$
\begin{align*}
E[g(\z)] &\leq \frac1n\sum_{i=1}^{n}g(\z_i)+2\fR_n(\cG)+\sqrt\frac{\log(1/\delta)}{2n}\\
E[g(\z)] &\leq \frac1n\sum_{i=1}^{n}g(\z_i)+2\fR_\cS(\cG)+3\sqrt\frac{\log(1/\delta)}{2n}
\end{align*}
\end{theorem}

\begin{theorem}[Theorem~10.3 of~\cite{mohri2018foundations}]\label{thm:generalization_regression}
Assume that $\|h-f\|_\infty \leq M$ for all $h\in \cH$. Then, for any $\delta>0$, with probability at least $1-\delta$ over a sample $\cS=\{(\x_i,\y_i), \ i \in [n]\}$ of size $n$, the following inequalities holds uniformly for all $h\in \cH$.
\begin{align*}
\E[|h(\x)-f(\x) |^2] &\leq \hat\E_i|h(\x_i) - f(\x_i)|^2 + 4M\fR_n(\cH)+M^2\sqrt\frac{\log(2/\delta)}{2n}\\
\E[|h(\x)-f(\x) |^2] &\leq \hat\E_i|h(\x_i) - f(\x_i)|^2 + 4M\fR_\cS(\cH)+3M^2\sqrt\frac{\log(2/\delta)}{2n}
\end{align*}
\end{theorem}

\begin{theorem}[Based on Theorem~1 in \cite{srebro2010optimistic}]\label{thm:generalization_optimistic_loss}
Let $\cX$ and $\cY=[-1,1]$ denote the input space and the label space, respectively. Let $\cH \subseteq \{f:\cX \to \cY\}$ be the target function class. For any $f\in\cH$, and any $(\x,y)\in \cX \times \cY$, let $\ell(f,\x,y):=(f(\x)-y)^2$ be the squared loss. Let $L(f)=\E_\cD[\ell(f,\x,y)]$ be the population risk with respect to the joint distribution $\cD$ on $\cX\times \cY$. For any $\delta > 0$, with probability at least $1 - \delta$ over a sample of size $n$, we have for any $f \in \cH$:
\begin{align*}
L(f) &\leq L_* + K \left(  \sqrt{L_*}\left( \sqrt{2}{\log(n)}^{1.5} \fR_n(\cH)+ \sqrt\frac{4\log\frac1{\delta}}{n}\right) +2{\log(n)}^3 \fR_n^2(\cH)+ \frac{4\log\frac1{\delta}}{n}\right)
\end{align*}
where $L_*:=\min_{f\in \cH}{L(f)}$, and $K$ is a numeric constant derived from~\cite{srebro2010optimistic}.
\end{theorem}

\begin{theorem}[Theorem~3.3 in \cite{mianjy2018implicit}]\label{thm:equalization}
For any pair of matrices $\U\in \R^{d_2\times d_1}, \V\in \R^{d_0 \times d_1}$, there exist a rotation matrix $\Q\in \text{SO}(d_1)$ such that rotated matrices $\tilde\U:= \U\Q, \tilde\V:=\V\Q$ satisfy $\| \tilde\u_i \|  \| \tilde\v_i \| = \frac1{d_1}\|\U\V^\top \|_*$, for all $i\in [d_1]$.
\end{theorem}

\begin{theorem}[Theorem~1 in~\cite{foygel2011learning}]\label{thm:rademacher_sensing}
Assume that $p(i)q(j) \geq \frac{\log(d_2)}{n\sqrt{d_2 d_0}}$ for all $i\in[d_2], j\in[d_0]$. For any $\alpha>0$, let $\cM_\alpha := \{ \M \in \R^{d_2\times d_1}: \ \| \diag(\sqrt{p})\M \diag(\sqrt{q}) \|_*^2 \leq \alpha \}$ be the class of linear transformations with weighted trace-norm bounded with $\sqrt\alpha$. Then the expected Rademacher complexity of $\cM_\alpha$ is bounded as follows:
$$\fR_n(\cM_\alpha) \leq O\left(\sqrt\frac{\alpha d_2 \log(d_2)}{n}\right)$$
\end{theorem}

\section{Matrix Sensing}
\begin{proposition}[Dropout regularizer in matrix sensing]\label{prop:dropout_reg_sensing} The following holds for any $p\in[0,1)$:
\begin{equation}
\label{eqn:reg_sensing}
    \hat{L}_\text{drop}(\U,\V)=\hat{L}(\U,\V)+\lambda\hat{R}(\U,\V), %\ \text{ where } \hat{R}(\U,\V) = \sum_{i=1}^{d_1}\frac1n\sum_{j=1}^{n}(\u_i^\top \A^{(j)}\v_i)^2. 
\end{equation}
where $\hat{R}(\U,\V) = \sum_{i=1}^{d_1}\hat\E_j(\u_i^\top \A^{(j)}\v_i)^2$ and $\lambda=\frac{p}{1-p}$ is the regularization parameter.
\end{proposition}

\begin{proof}[Proof of Proposition~\ref{prop:dropout_reg_sensing}]
Similar statements and proofs can be found in several previous works~\cite{srivastava2014dropout,wang2013fast,cavazza2018dropout,mianjy2018implicit}. For completeness, we include a proof here. The following equality follows from the definition of variance:
\begin{align}\label{eq:sum_exp}
\E_\b[(y_i - \langle \U\B\V^\top, \A^{(i)} \rangle)^2] = \left(\E_\b[y_i - \langle \U\B\V^\top, \A^{(i)} \rangle] \right)^2 + \var(y_i - \langle \U\B\V^\top, \A^{(i)} \rangle)
\end{align}
Recall that for a Bernoulli random variable $\B_{ii}$, we have $\E[\B_{ii}]=1$ and $\var(\B_{ii})= \frac{p}{1-p}$. Thus, the first term on right hand side is equal to $(y_i - \langle \U\V^\top, \A^{(i)}  \rangle)^2$. For the second term we have
\begin{align*}
\var(y_i - \langle \U\B\V^\top, \A^{(i)} \rangle)&=\var(\sum_{j=1}^{d_1}\B_{jj} \u_j^\top \A^{(i)}\v_j)=\sum_{j=1}^{d_1} (\u_j^\top \A^{(i)}\v_j)^2 \var(\B_{jj})=\frac{p}{1-p}\sum_{j=1}^{d_1} (\u_j^\top \A^{(i)}\v_j)^2 
\end{align*}
Plugging the above into Equation~\eqref{eq:sum_exp} and averaging over samples we get
\begin{align*}
\hat{L}_{\text{drop}}(\U,\V) &= \hat\E_i\E_\b[(y_i - \langle \U\B\V^\top, \A^{(i)} \rangle)^2] \\
&= \hat\E_i(y_i - \langle \U\V^\top, \A^{(i)} \rangle)^2 +\hat\E_i \frac{p}{1-p} \sum_{j=1}^{d_1} (\u_j^\top \A^{(i)}\v_j)^2 \\
&= \hat{L}(\U,\V) + \frac{p}{1-p}\hat{R}(\U,\V).
\end{align*}
which completes the proof.
\end{proof}

\begin{lemma}[Concentration in matrix completion]\label{lem:concentration}
For $\ell \in [n]$, let $\A^{(\ell)}$ be an indicator matrix whose $(i,j)$-th element is selected according to some distribution. Assume $\U,\V$ is such that %$\max_i \|\U(i,:)\|^2 \leq \gamma, \ \max_i \|\V(i,:)\|^2 \leq \gamma$. %$\max_{i,j}|\U_{ij}\V_{ij}| \leq 1$.
$\| \U^\top \|_{2,\infty} \| \V \|_{\infty,\infty}  \leq \gamma$. Then, with probability at least $1-\delta$ over a sample of size $n$, we have that $$|R(\U,\V) - \hat{R}(\U,\V)| \leq \frac{C \gamma^2 \sqrt{\log(2/\delta)}}{\sqrt{n}}.$$
\end{lemma}

\begin{proof}[Proof of Lemma~\ref{lem:concentration}]
Define $X_{\ell} := \sum_{w=1}^{d_1}(\u_w^\top \A^{(\ell)} \v_w)^2$ and observe that
\begin{align*}
X_{\ell} &=\sum_{w=1}^{d_1}\left(\sum_{i,j}\U_{iw}\V_{jw} \A^{(\ell)}_{ij}\right)^2 = \sum_{w=1}^{d_1}\sum_{i,i',j,j'}\U_{iw}\U_{i'w}\V_{jw}\V_{j'w} \A^{(\ell)}_{ij}\A^{(\ell)}_{i'j'} \\
&= \sum_{w=1}^{d_1}\sum_{i,j}\U_{iw}^2\V_{jw}^2 \A^{(\ell)}_{ij} \leq \max_{i,j}\sum_{w=1}^{d_1}\U_{iw}^2\V_{jw}^2  \\
&\leq \max_{i,j}\|\U(i,:)\|^2\|\V(j,:)\|_\infty^2 = \| \U^\top \|_{2,\infty}^2\| \V \|_{\infty,\infty}^2  \leq \gamma^2
\end{align*}
where the third equality follows because for an indicator matrix $\A^{(\ell)}$, it holds that $\A^{(\ell)}_{ij}\A^{(\ell)}_{i'j'} = 0$ if $(i,j)\neq (i',j')$. Thus, $X_{w,\ell}$ is a sub-Gaussian (more strongly, bounded) random variable with mean $\E[X_\ell]=R(\U,\V)$ and sub-Gaussian norm $\| X_\ell \|_{\psi_2} \leq \gamma^2 / \ln(2)$. Furthermore, $\|X_\ell - R(\U,\V)\|_{\psi_2} \leq C' \|X_\ell\|_{\psi_2} \leq C \gamma^2$, for some absolute constants $C', C$ (Lemma~2.6.8 of~\cite{vershynin2018high}). Using Theorem~\ref{thm:hoeffding}, for $t=Cd_1 \sqrt\frac{\log{2/\delta}}{n}$ we~get~that:
$$\bb{P}\left(\abs{\hat{R}(\U,\V) - R(\U,\V)} \geq t \right) = \bb{P}\left(\abs{\frac1n\sum_{\ell=1}^{n}X_\ell - R(\U,\V)} \geq C \gamma^2 \sqrt\frac{\log{2/\delta}}{n} \right) \leq \delta$$ which completes the proof.
\end{proof}

\begin{proposition}\label{prop:induced}[Induced regularizer] For $j\in [n]$, let $\A^{(j)}$ be an indicator matrix whose $(i,k)$-th element is selected randomly with probability $p(i)q(k)$, where $p(i)$ and $q(k)$ denote the probability of choosing the $i$-th row and the $k$-th column. Then $\Theta(\M)= \frac1{d_1} \| \diag(\sqrt p)\U\V^\top \diag(\sqrt q) \|_*^2$.
\end{proposition}
\begin{proof}[Proof of Proposition~\ref{prop:induced}]
For any pair of factors $(\U,\V)$ it holds that
\begin{align*}
R(\U,\V)&=\sum_{i=1}^{d_1}\E(\u_i^\top \A \v_i)^2 = \sum_{i=1}^{d_1}\sum_{j=1}^{d_2}\sum_{k=1}^{d_0}p(j)q(k)(\u_i^\top \e_j\e_k^\top \v_i)^2 \\
&= \sum_{i=1}^{d_1}\sum_{j=1}^{d_2}\sum_{k=1}^{d_0}p(j)q(k)\U(j,i)^2 \V(k,i)^2 = \sum_{i=1}^{d_1}\| \diag(\sqrt p)\u_i\|^2 \| \diag(\sqrt q)\v_i \|^2
\end{align*}
We can now lower bound the right hand side above as follows:
\begin{align*}
R(\U,\V) &\geq \frac1{d_1} \left(\sum_{i=1}^{d_1}\| \diag(\sqrt p)\u_i\| \| \diag(\sqrt q)\v_i \|\right)^2 \\
&= \frac1{d_1} \left(\sum_{i=1}^{d_1}\| \diag(\sqrt p)\u_i \v_i^\top \diag(\sqrt q) \|_*\right)^2 \\
&\geq \frac1{d_1} \left(\| \diag(\sqrt p)\sum_{i=1}^{d_1} \u_i \v_i^\top \diag(\sqrt q) \|_*\right)^2 = \frac1{d_1} \| \diag(\sqrt p)\U\V^\top \diag(\sqrt q) \|_*^2
\end{align*}
where the first inequality is due to Cauchy-Schwartz and the second inequality follows from the triangle inequality. The equality right after the first inequality follows from the fact that for any two vectors $\a,\b$, $\| \a\b^\top \|_*=\| \a\b^\top \|=\| \a\|  \|\b \|$. Since the inequalities hold for any $\U,\V$, it implies that $$\Theta(\U\V^\top)\geq \frac1{d_1} \| \diag(\sqrt p)\U\V^\top \diag(\sqrt q) \|_*^2.$$ 
Applying Theorem~\ref{thm:equalization} on $(\diag(\sqrt{p})\U,\diag(\sqrt{p})\V)$, there exist a rotation matrix $\Q$ such that
$$\| \diag(\sqrt p)\U\q_i\| \| \diag(\sqrt q)\V\q_i \| = \frac1{d_1}\| \diag(\sqrt p)\U\V^\top \diag(\sqrt q) \|_*$$
We evaluate the expected dropout regularizer at $\U\Q,\V\Q$:
\begin{align*}
R(\U\Q,\V\Q)&= \sum_{i=1}^{d_1}\| \diag(\sqrt p)\U\q_i\|^2 \| \diag(\sqrt q)\V\q_i \|^2 \\
&= \sum_{i=1}^{d_1} \frac{1}{d_1^2} \| \diag(\sqrt p)\U\V^\top \diag(\sqrt q)\|_*^2 = \frac{1}{d_1} \| \diag(\sqrt p)\U\V^\top \diag(\sqrt q)\|_*^2 \leq \Theta(\U\V^\top)
\end{align*}
which completes the proof of the first part.
\end{proof}

\begin{proof}[Proof of Theorem~\ref{thm:generalization_sensing_expected}]
We use Theorem~\ref{thm:generalization_regression} to bound the population risk in terms of the Rademacher complexity of the target class. Define the class of predictors with weighted trace-norm bounded by $\sqrt{\alpha}$, i.e. $$\cM_{\alpha}=\{ \M: \ \| \diag(\sqrt\p)\M\diag(\sqrt\q) \|_*^2 \leq \alpha \}.$$ In particular dropout empirical risk minimizers $\U,\V$ belong to this class:
\begin{align*}
\|\diag(\sqrt p)\U\V^\top \diag(\sqrt q) \|_*^2 &= d_1\Theta(\U\V^\top ) \leq d_1R(\U,\V) \leq \alpha \end{align*}
where the first inequality holds by definition of the induced regularizer, and the second inequality follows from the assumption of the theorem. Since $g$ is a contraction, by Talagrand's lemma and Theorem~\ref{thm:rademacher_sensing}, we have that $\fR_n(g \circ \cM_{\alpha})\leq \fR_n(\cM_{\alpha})\leq \sqrt\frac{\alpha d_2 \log(d_2)}{n}$. To obtain the maximum deviation parameter $M$ in Theorem~\ref{thm:generalization_regression}, we note that the assumption $\|\M_*\| \leq 1$ implies that $|\M_*(i,j)|\leq 1$ for all $i,j$, so that $g(\M_*)=\M_*$. We have that:
\begin{align*}
\max_\A | \langle \M_* - g(\U\V^\top), \A \rangle  | &= \max_{i,j}|\langle \M_* - g(\U\V^\top), \e_i \e_j^\top \rangle| \leq \max_{i,j}|\M_*(i,j)| + \max_{i,j} |\langle \U\V^\top, \e_i\e_j^\top \rangle| \leq \|\M_*\| + 1 \leq 2
\end{align*}
Let $L(g(\U\V^\top)):=\E(y-\langle g(\U\V^\top),\A\rangle )^2$ and $\hat{L}(g(\U\V^\top)):=\hat\E_i(y_i-\langle g(\U\V^\top),\A^{(i)}\rangle )^2$ denote the \emph{true risk} and the \emph{empirical risk} of $g(\U\V^\top)$, respectively. Plugging the above results in Theorem~\ref{thm:generalization_regression}, we get
\begin{align*}
L(g(\U,\V)) &\leq \hat{L}(g(\U,\V)) + 8\fR_n(g \circ \cM_\alpha) + 4\sqrt\frac{\log(2/\delta)}{2n}\\
&\leq \hat{L}(\U,\V) + 8\sqrt\frac{\alpha d_2 \log(d_2)}{n}+4\sqrt\frac{\log(2/\delta)}{2n}\\
&\leq \hat{L}(\U,\V) + 8\sqrt{\frac{2\alpha d_2 \log(d_2) + \frac14 \log(2/\delta)}{n}}
\end{align*}
where the second inequality holds since $\hat{L}(g(\U,\V)) \leq \hat{L}(\U,\V)$.
\end{proof}

\subsection{Optimistic Rates}
As we discussed in the main text, under additional assumptions on the value of $\alpha$, it is possible to give optimistic generalization bounds that decay as $\tilde{O}(\alpha d_2/n)$. This result is given as the following theorem.  
\begin{theorem}\label{thm:generalization_sensing_expected_optimistic} \normalfont 
Assume that $d_2\geq d_0$ and $\|\M_*\| \leq 1$. Furthermore, assume that $\min_{i,k}p(i)q(k) \geq \frac{\log(d_2)}{n\sqrt{d_2 d_0}}$. Let $(\U,\V)$ be a minimizer of the dropout ERM objective in equation~\eqref{eq:dropout_obj_sensing}.
Let $\alpha$ be such that $\max\{R(\U,\V),\Theta(\M_*)\}\leq \alpha/{d_1}$. Then, for any $\delta\in(0,1)$, the following generalization bounds holds with probability at least $1-\delta$ over a sample of size $n$:
\begin{equation*}
L(g(\U\V^\top)) \leq \frac{2K{\log(n)}^3 \alpha d_2 \log(d_2) + 4K\log(1/\delta)}{n}
\end{equation*}
where $K$ is an absolute constant~\cite{srebro2010optimistic}, $g(\M)$ thresholds $\M$ between $[-1,1]$, and $L(g(\U\V^\top)):=\E(y-\langle g(\U\V^\top),\A\rangle )^2$ is the \emph{true risk} of $g(\U\V^\top)$. %\begin{cases}
\end{theorem}

\begin{proof}[Proof of Theorem~\ref{thm:generalization_sensing_expected_optimistic}]
We use Theorem~\ref{thm:generalization_optimistic_loss} to bound the population risk in terms of the Rademacher complexity of the target class. Define the class of predictors with weighted trace-norm bounded by $\sqrt{\alpha}$, i.e. $$\cM_{\alpha}=\{ \M: \ \| \diag(\sqrt\p)\M\diag(\sqrt\q) \|_*^2 \leq \alpha \}.$$ In particular dropout empirical risk minimizers $\U,\V$ belong to this class:
\begin{align*}
\|\diag(\sqrt p)\U\V^\top \diag(\sqrt q) \|_*^2 &= d_1\Theta(\U\V^\top ) \leq d_1R(\U,\V) \leq \alpha \end{align*}
where the first inequality holds by definition of the induced regularizer, and the second inequality follows from the assumption of the theorem. Moreover, by assumption $\Theta(\M_*)\leq \alpha$, we have that $\M_*\in \cM_\alpha$. With this, we get that 
$$L_* := \min_{M\in g\circ \cM_\alpha} L(\M) \leq L(g(\M_*)) = L(g(\M_*)) = 0.$$
Since $g$ is a contraction, by Talagrand's lemma and Theorem~\ref{thm:rademacher_sensing}, we have that $\fR_n(g \circ \cM_{\alpha})\leq \fR_n(\cM_{\alpha})\leq \sqrt\frac{\alpha d_2 \log(d_2)}{n}$. Plugging the above in Theorem~\ref{thm:generalization_regression}, we get
\begin{align*}
L(g(\U,\V)) &\leq 2K{\log(n)}^3 \fR_n^2(g \circ \cM_\alpha)+ \frac{4K\log\frac1{\delta}}{n}\\
&\leq  \frac{2K{\log(n)}^3 \alpha d_2 \log(d_2) + 4K\log\frac1{\delta}}{n}
\end{align*}
\end{proof}

\section{Non-linear Neural Networks}
\begin{proposition}[Dropout regularizer in deep regression]\label{prop:dropout_reg_dnn}
\begin{equation*}
\hat{L}_\text{drop}(\w) = \hat{L}(\w) + \hat{R}(\w), \ \text{ where } \ \hat{R}(\w)=\lambda {\sum_{j=1}^{d_1}\|\u_j\|^2 \hat{a}_j^2}.
\end{equation*}
where $\hat{a}_j=\sqrt{\hat\E_i a_{j}(\x_i)^2}$ and $\lambda = \frac{p}{1-p}$ is the regularization parameter.
\end{proposition}
\begin{proof}[Proof of Proposition~\ref{prop:dropout_reg_dnn}]
Similar statements and proofs can be found in several previous works~\cite{srivastava2014dropout,wang2013fast,cavazza2018dropout,mianjy2018implicit}. Here we include a proof for completeness. Recall that $\E[\B_{ii}]=1$ and $\var(\B_{ii})=\frac{p}{1-p}$.
Conditioned on $\x,\y$ in the current mini-batch, we have that:
\begin{align*}
 \E_\B\|\y - \U^\top \B \a(\x) \|^2 = \sum_{i=1}^{d_2}\left(\E_\B[y_i - \u_i^\top\B \a(\x)] \right)^2 + \sum_{i=1}^{d_2}\var(y_i - \u_i^\top \B \a(\x))
\end{align*}
Since $\E[\B]=\I$, the first term on right hand side is equal to $\|\y - \U^\top\a(\x)\|^2$. For the second term we have
\begin{align*}
\sum_{i=1}^{d_2}\var(y_i - \u_i^\top \B\a(\x))
=\sum_{i=1}^{d_2}\var(\sum_{j=1}^{d_1}\U_{j,i}\B_{jj} a_{j}(\x)) = \sum_{i=1}^{d_2}\sum_{j=1}^{d_1}(\U_{j,i} a_{j}(\x))^2\var(\B_{jj}) = \frac{p}{1-p}\sum_{j=1}^{d_1} \|\u_j\|^2 a_{j}(\x)^2
\end{align*}
Thus, conditioned on the sample $(\x,\y)$, we have that 
\begin{align*}
\E_\B[\| \y - \U^\top\B\a(\x) \|^2] = \|\y - \U^\top\a(\x) \|^2 + \frac{p}{1-p}\sum_{j=1}^{d_1} \|\u_j\|^2 a_{j}(\x)^2
\end{align*}
Now taking the empirical average with respect to $\x,\y$, we get
\begin{align*}\label{eq:deriv_dropout}
\hat{L}_{\text{drop}}(\w) &= \hat{L}(\w) + \frac{p}{1-p} \sum_{j=1}^{d_1}\|\u_j\|^2 \hat{a}_j^2 = \hat{L}(\w) + \hat{R}(\w)
\end{align*}
which completes the proof.
\end{proof}

\begin{proposition}\label{prop:reg_two_layer} Consider a two layer neural network $f_\w(\cdot)$ with ReLU activation functions in the hidden layer. Furthermore, assume that the marginal input distribution $\bb{P}_\cX(\x)$ is symmetric and isotropic, i.e., $\bb{P}_\cX(\x) = \bb{P}_\cX(-\x)$ and $\E[\x\x^\top] = \I$. Then the following holds for the expected \ref{eq:exp_reg} due to dropout:
\begin{equation}
R(\w):=\E[\hat{R}(\w)]=\frac{\lambda}2\sum_{i_0,i_1,i_2=1}^{d_0,d_1,d_2}\U(i_1,i_2)^2\V(i_1,i_0)^2,
\end{equation}
\end{proposition}
\begin{proof}[Proof of Proposition~\ref{prop:reg_two_layer}]
Using Proposition~\ref{prop:dropout_reg_dnn}, we have that:
\begin{align*}
R(\w) = \E[\hat{R}(\w)] = \lambda \sum_{j=1}^{d_1}\|\u_j\|^2 \E[\relu(\V(j,:)^\top\x)^2]
\end{align*}
It remains to calculate the quantity $\E_\x[\relu(\V(j,:)^\top \x)^2]$. By symmetry assumption, we have that $\bb{P}_\cX(\x)=\bb{P}_\cX(-\x)$. As a result, for any $\v\in \R^{d_0}$, we have that $\bb{P}(\v^\top \x)=\bb{P}(-\v^\top \x)$ as well. That is, the random variable $z_j:= \W_1(j,:)^\top \x$ is also symmetric about the origin. It is easy to see that $\E_z[\relu(z)^2] = \frac12 \E_z[z^2]$.
\begin{align*}
\E_z[\relu(z)^2] = \int_{-\infty}^{\infty}\relu(z)^2 d\mu(z) = \int_{0}^{\infty}\relu(z)^2 d\mu(z) =  \int_{0}^{\infty} z^2 d\mu(z) = \frac12 \int_{\infty}^{\infty} z^2 d\mu(z) = \frac12 \E_z[z^2].
\end{align*}
Plugging back the above identity in the expression of $R(\w)$, we get that
\begin{align*}
R(\w) = \lambda \sum_{j=1}^{d_1}\|\u_j\|^2 \E[(\V(j,:)^\top\x)^2] = \frac{\lambda}2 \sum_{j=1}^{d_1}\|\u_j\|^2 \|\V(j,:)\|^2
\end{align*}
where the second equality follows from the assumption that the distribution is isotropic.
\end{proof}

\begin{proof}[Proof of Proposition~\ref{prop:weak_measure}]
For $\delta\in(0,\frac12)$, consider the following random variable:
\begin{equation*}
\x = \begin{cases}
[1; 0] \ &\text{with probability } \delta \\ \\
[\frac{-\delta}{1-\delta}; \frac{\sqrt{1-2\delta}}{1-\delta}] \ &\text{with probability } \frac{1-\delta}2 \\ \\
[\frac{-\delta}{1-\delta}; -\frac{\sqrt{1-2\delta}}{1-\delta}] \ &\text{with probability } \frac{1-\delta}2

\end{cases}
\end{equation*}
It is easy to check that the $\x$ has zero mean and is supported on the unit sphere. Consider the vector $\w=[\frac{1}{\sqrt\delta}; 0]$. It is easy to check that $\x$ satisfies $R(\w)=\sqrt{\E\sigma(\w^\top\x)^2} = 1$; however, for any given $C$, it holds that $\| \w\| \geq C$ as long as we let $\delta = C^2$.
\end{proof}

\begin{proof}[Proof of Theorem~\ref{thm:rademacher_dnn_new}]
For any $j\in[h]$, let $a_j^2 := \E[\sigma(\v_j^\top\x)^2]$ denote the average squared activation of the $j$-th node with respect to the input distribution. Given $n$ i.i.d. samples $\cS = \{ \x_1,\cdots,\x_n \}$, the empirical Rademahcer complexity is bounded as follows:
\begin{align*}
\fR_\cS(\cF_\alpha) &= \E_{\zeta} \sup_{f_{\{\u,\V \}}\in \cF_\alpha} \frac1n \sum_{j=1}^{h}u_j a_j  \sum_{i=1}^{n}\zeta_i \frac{\sigma(\v_j^\top \x_i)}{a_j}\\
&\leq \E_{\zeta} \sup_{f_{\{\u,\V \}}\in \cF_\alpha} \frac1n \sum_{j=1}^{h} | u_j a_j | \  | \sum_{i=1}^{n}\zeta_i \frac{\sigma(\v_j^\top\x_i)}{a_j}|\\
&\leq \E_{\zeta} \left[ \left(\sup_{f_{\{\u,\V\}}\in \cF_\alpha} \sum_{j=1}^{h} | u_j a_j | \right) \ \left(\sup_{\V} \max_{j\in [h]} |\frac1n \sum_{i=1}^{n}\zeta_i \frac{\sigma(\v_j^\top\x_i)}{a_j}| \right) \right] \nonumber
\end{align*}
where we used the fact that the supremum of product of positive functions is upperbounded by the product of the supremums. By definition of $\cF_\alpha$, the first term on the right hand side is bounded by $\alpha$. To bound the second term in the right hand side, we note that the maximum over rows of $\V^\top$ can be absorbed into the supremum.
\begin{align*}
\frac1n \E_{\zeta} \sup_{\v} | \sum_{i=1}^{n}\zeta_i \frac{\sigma(\v^\top\x_i)}{\sqrt{\E[\sigma(\v^\top\x)^2]}}| &= \frac1n \E_{\zeta} \sup_{\E[\sigma(\v^\top\x)^2]\leq 1} | \sum_{i=1}^{n}\zeta_i \sigma(\v^\top\x_i)|\\
&\leq \frac2n \E_{\zeta} \sup_{\E[\sigma(\v^\top\x)^2]\leq 1}  \sum_{i=1}^{n}\zeta_i \sigma(\v^\top\x_i)\\
&\leq \frac2n \E_{\zeta} \sup_{\beta \E (\v^\top\x)^2 \leq 1}  \sum_{i=1}^{n}\zeta_i \sigma(\v^\top\x_i) \tag{$\beta$-retentiveness}
\end{align*}

\noindent Let $\C^\dagger$ be the pseudo-inverse of $\C$. We perform the following change the variable: $\w \gets \C^{-\dagger/2}\v$.
\begin{align*}
\text{R.H.S.} &\leq \frac2n \E_{\zeta} \sup_{\E[(\w^\top \C^{\dagger/2}\x)^2]\leq 1/\beta}  \sum_{i=1}^{n}\zeta_i \w^\top \C^{\dagger/2}\x_i\\
&= \frac2n \E_{\zeta} \sup_{\|\w\|^2\leq 1/\beta}  \langle \w, \sum_{i=1}^{n}\zeta_i \C^{\dagger/2}\x_i \rangle \\
&= \frac2{n\sqrt\beta} \E_{\zeta} \| \sum_{i=1}^{n}\zeta_i \C^{\dagger/2}\x_i \| \\
&\leq \frac2{n\sqrt\beta} \sqrt{\E_{\zeta} \| \sum_{i=1}^{n}\zeta_i \C^{\dagger/2}\x_i \|^2} = \frac2{n\sqrt\beta} \sqrt{ \sum_{i=1}^{n}\x_i^\top \C^{\dagger}\x_i}
\end{align*}
where the last inequality holds due to Jensen's inequality. To bound the expected Rademacher complexity, we take the expected value of both sides with respected to sample $\cS$, which gives the following:
$$\fR_n(\cF_\alpha)=\E_\x[\fR_{\cS}(\cF_\alpha)]\leq \frac2{n\sqrt{\beta}}\E_\cS\sqrt{\sum_{i=1}^{n}\x_i^\top\C^{\dagger}\x_i}\leq \frac2{n\sqrt{\beta}}\sqrt{\sum_{i=1}^{n}\E_{\x_i}[\x_i^\top\C^{\dagger}\x_i]},$$ where the last inequality holds again due to Jensen's inequality. Finally, we have that $\E_{\x_i} \x_i^\top \C^{\dagger}\x_i = \E_{\x_i}\langle \x_i\x_i^\top, \C^{\dagger}\rangle = \langle \C, \C^\dagger\rangle = \rank(\C)$, which completes the proof of the Theorem.
\end{proof}

\begin{proof}[Proof of Theorem~\ref{thm:lb_rademacher}]
For simplicity, assume that the width of the hidden layer is even. Consider the linear function class:
$$\cG_r:=\{ g_\w: \x \mapsto \w^\top\x, \ \E(\w^\top\x)^2 \leq d_1 r/2 \}.$$
Recall that $\cH_r:=\{ h_\w:\x\mapsto \u^\top\sigma(\V^\top\x), \ R(\u,\V)\leq r \}$. First, we argue that $\cG_r \subset \cH_{r}$. Let $g_\w\in \cG_r$; we show that there exist $\u,\V$ such that $g_\w=f_{\u,\V}$ and $f_{\u,\V}\in\cH_r$. Define $\u:= \frac{2}{d_1}[1;-1;\cdots 1;-1]\in \R^{d_1}$, and let $\V=\w(\e_1-\e_2+\e_3-\e_4+\cdots+\e_{d_1-1}-\e_{d_1})^\top$, where $\e_i\in\R^{d_1}$ is the $i$-th standard basis vector. It's easy~to~see~that
\begin{align*}
f_{\u,\V}(\x)=\u^\top\sigma(\V^\top\x)&=\sum_{i=1}^{d_1}u_i\sigma(\v_i^\top\x) \\
&= \sum_{i=1}^{d_1} \frac2{d_1} (-1)^{i-1} \sigma(\v_i^\top\x) \\
&= \sum_{i=1}^{d_1/2} \frac2{d_1} (\sigma(\v_{2i-1}^\top\x)-\sigma(\v_{2i}^\top\x)) \\
&= \sum_{i=1}^{d_1/2}\frac{2}{d_1}(\sigma(\w^\top\x)-\sigma(-\w^\top\x))=\w^\top\x=g_\w.
\end{align*}
Furthermore, it holds for the explicit regularizer that
\begin{align*}
R(\u,\V)=\sum_{i=1}^{d_1}u_i^2\E\sigma(\v_i^\top\x)^2 &=\sum_{i=1}^{d_1/2}\frac{4}{d_1^2}\left(\E\sigma(\v_{2i-1}^\top\x)^2+\E\sigma(\v_{2i}^\top\x)^2\right) \\
&=\sum_{i=1}^{d_1/2}\frac{4}{d_1^2}\E[\sigma(\w^\top\x)^2+\sigma(-\w^\top\x)^2]\\
&=\frac{2}{d_1}\E(\w^\top\x)^2 \leq r
\end{align*}
Thus, we have that $\cG_r \subset \cH_{r}$, and the following inequalities follow.
\begin{align*}
\fR_{\cS}(\cH_r) &\geq \fR_{\cS}(\cG_{r}) = \E_{\epsilon_i}\sup_{g_\w\in\cG_{r}}\frac1n \sum_{i=1}^{n}\epsilon_i g_\w(\x_i)\\
&=\E_{\epsilon_i}\sup_{\E(\w^\top\x)^2 \leq d_1 r/2}\frac1n \sum_{i=1}^{n}\epsilon_i \w^\top \x_i\\
&=\E_{\epsilon_i}\sup_{\w^\top\C\w \leq d_1 r/2}\frac1n \langle \w, \sum_{i=1}^{n}\epsilon_i \x_i\rangle\\
&=\E_{\epsilon_i}\sup_{\|\C^{1/2}\w\|^2 \leq d_1 r/2}\frac1n \langle \C^{1/2}\w, \sum_{i=1}^{n}\epsilon_i \C^{-\dagger/2}\x_i\rangle\\
&=\frac{\sqrt{d_1 r}}{\sqrt{2}n} \E_{\epsilon_i} \| \sum_{i=1}^{n}\epsilon_i \C^{\dagger/2}\x_i\|\\
&\geq \frac{c\sqrt{d_1 r}}{\sqrt{2}n}   \sqrt{\sum_{i=1}^{n} \|\C^{\dagger/2}\x_i\|^2} = \frac{c\sqrt{d_1 r}\|\X\|_{\C^\dagger}}{\sqrt{2}n}
\end{align*}
where the last inequality follows from Khintchine-Kahane inequality in Lemma~\ref{lem:Khintchine-Kahane}.
\end{proof}

Next, we define some function classes that will be used frequently in the proofs.
\begin{definition}\label{def:classes}
For any closed subset $[a, b] \subset \R$, let  $\Pi_{[a,b]}(y):= \max\{ a, \min\{ b, y \} \}$. For $\z:=(\x,y)$ and $f:\cX\to\cY$, define the squared loss $\ell_2(f,\z):=(1-yf(\x))^2$. For a given value $\alpha > 0$, consider the following classes
\begin{align*}
\cW_\alpha &:= \{ \w=(\u,\V)\in \R^{d_1}\times \R^{d_0 \times d_1}, \ \sum_{i=1}^{d_1} |u_i|\sqrt{\E\sigma(\v_i^\top\x)^2} \leq \alpha \}\\
\cF_{\alpha} &:= \{f_\w : \x \mapsto \u^\top\sigma(\V^\top\x), \  \w\in \cW_\alpha  \}, \\
\cG_{\alpha} &:= \Pi_{[-1, 1]} \circ \cF_{\alpha} = \{g_\w = \Pi_{[-1, -1]} \circ f_\w, \  f_\w \in \cF_\alpha  \} \\
\cL_\alpha &:=\{ \ell_2 : (g_\w, \z) \mapsto (y-g_\w(\x))^2, \  g_\w \in \cG_\alpha \}
\end{align*}
\end{definition}

\begin{lemma}\label{lem:rademacher_twice}
Let $\cW_\alpha,\cF_\alpha,\cG_\alpha,\cL_\alpha$ be as defined in Definition~\ref{def:classes}. Then the following holds true:
\begin{enumerate}
\item $\fR_\cS(\cG_{\alpha}) \leq \fR_\cS(\cF_{\alpha})$.
\item If $\cY=\{-1,+1\}$ (binary classification), then it holds that $\fR_\cS(\cL_{\alpha}) \leq 2\fR_\cS(\cG_{\alpha})$.
\end{enumerate}
\end{lemma}
\begin{proof}
Since $\Pi_{[-1, -1]}(\cdot)$ is 1-Lipschitz, by Talagrand's contraction lemma, we have that $\fR_\cS(\cG_{\alpha}) \leq \fR_\cS(\cF_{\alpha})$. The second claim follows from
\begin{align*}
\fR_\cS(\cL_\alpha)&=\E_\zeta\sup_{\w\in\cW}\frac1n\sum_{i=1}^{n}\zeta_i(y_i-g_\w(\x_i))^2\\
&=\E_\zeta\sup_{\w\in\cW}\frac1n\sum_{i=1}^{n}\zeta_i(1-y_i g_\w(\x_i))^2 \tag{$y_i \in \{ -1,+1 \}$} \\
&\leq 2\E_\zeta\sup_{\w\in\cW}\frac1n\sum_{i=1}^{n}\zeta_i y_i g_\w(\x_i)\\
&= 2\E_\zeta\sup_{\w\in\cW}\frac1n\sum_{i=1}^{n}\zeta_i  g_\w(\x_i) = 2\fR_\cS(\cG_\alpha)
\end{align*}
where the first inequality follows from Talagrand's contraction lemma due to the fact that $h(z)=(1-z)^2$ is 2-Lipschitz for $z\in[-1,1]$, and the penultimate holds true since for any fixed $(y_i)_{i=1}^{n}\in \{-1,+1\}^n$, the distribution of $(\zeta_1 y_1, \ldots, \zeta_ny_n)$ is the same as that of $(\zeta_1,\ldots,\zeta_n)$.
\end{proof}

\begin{proof}[Proof of Corollary~\ref{cor:gen_bound_beta_regression}]
We use the standard generalization bound in Theorem~\ref{thm:generalization_regression} for class $\cG_\alpha$:
\begin{align*}
L_\cD(g_\w) &\leq \hat{L}_\cS(g_\w) + 4M\fR_\cS(\cG_\alpha)+3M^2\sqrt\frac{\log(2/\delta)}{2n}\\
&\leq \hat{L}_\cS(g_\w) + 8\fR_\cS(\cF_\alpha)+12\sqrt\frac{\log(2/\delta)}{2n} \tag{Lemma~\ref{lem:rademacher_twice}} \\
&\leq \hat{L}_\cS(g_\w) + \frac{16\alpha\|\X\|_{\C^\dagger}}{\sqrt{\beta}n}+12\sqrt\frac{\log(2/\delta)}{2n} \tag{Theorem~\ref{thm:rademacher_dnn_new}}
\end{align*}
where second inequality follows because the maximum deviation parameter $M$ in Theorem~\ref{thm:generalization_regression} is bounded as $$M=\sup_{\w\in\cW}\sup_{(\x,y)\in\cX\times\cY}|y - g_\w(\x)|\leq \sup_{\w\in\cW}\sup_{(\x,y)\in\cX\times\cY}|y| + |g_\w(\x)| \leq 2.$$
\end{proof}

\begin{proof}[Proof of Corollary~\ref{cor:gen_bound_nobeta_regression}] Recall that the input is jointly distributed as $(\x,y)\sim \cD$.
For $\cX\subseteq \R_+^{d_0}$, let $\cX'=\cX \cup -\cX$ be the \emph{symmetrized input domain}. Let $\zeta$ be a Rademacher random variable. Denote the \emph{symmetrized input} by $\x' = \zeta\x$, and the joint distribution of $(\x', y)$ by $\cD'$. By construction, $\cD'$ is centrally symmetric w.r.t. $\x'$, i.e., it holds for all $(\x,y)\in\cX\times \cY$ that $\cD'(\x,y)=\cD'(-\x,y)=\frac12\cD(\x,y)$. As a result, population risk with respect to the original distribution $\cD$ can be bounded in terms of the population risk with respect to the \emph{symmetrized distribution} $\cD'$ as follows:
\begin{align}%\label{eq:bound_symmetrized}
L_\cD(f) &:= \E_\cD[\ell(f(\x),y)] \nonumber \\
&\leq \E_\cD[\ell(f(\x),y)+\ell(f(-\x),y)] \nonumber \\
&= 2\E_\cD[\frac12 \ell(f(\x),y)+\frac12\ell(f(-\x),y)] \nonumber \\
&= 2\E_\cD\E_\zeta [\ell(f(\zeta\x),y) \ | \ \x,y] \nonumber \\
&= 2\E_{\cD'}[\ell(f(\x'),y)] =  2L_{\cD'}(f)
\end{align}
Moreover, since $\cD'$ is centrally symmetric, Assumption~\ref{ass:symmetry} holds with $\beta = \frac12$. The proof of Corollary~\ref{cor:gen_bound_nobeta_regression} follows by doubling the right hand side of inequalities in Corollary~\ref{cor:gen_bound_beta_regression}, and substituting $\beta = \frac12$.
\end{proof}

\subsection{Classification}
Although in the main text we only focus on the task of regression with squared loss, it is not hard to extend the results to binary classification. In particular, the following two Corollaries bound the miss-classification error in terms of  the training error and the Rademacher complexity of the target class, with and without symmetrization.
\begin{corollary}\label{cor:gen_bound_beta_classification}
Consider a binary classification setting where $\cY = \{-1,+1\}$. For any $\w\in \cF_{\alpha}$, for any $\delta\in(0,1)$, the following generalization bound holds with probability at least $1-\delta$ over $\cS=\{(\x_i,y_i) \}_{i=1}^{n}\sim{\cD}^n$:
\begin{equation*}
\bP\{ y f_\w(\x) < 0 \} \leq \hat{L}_{\cS}(g_\w) + \frac{8\alpha\|\X\|_{\C^\dagger}}{\sqrt{\beta}n} + 4\sqrt{\frac{\log(1/\delta)}{2n}}
\end{equation*}
where $g_\w(\cdot) = \max\{ -1, \min\{1, f_\w(\cdot)\} \}$ projects the network output onto the range $[-1, 1]$.
\end{corollary}

\begin{proof}[Proof of Corollary~\ref{cor:gen_bound_beta_classification}]
We use the standard generalization bound in Theorem~\ref{thm:generalization_classification}. Recall that $g_w = \Pi_{[-1, 1]}(f_\w)$, where $\Pi_{[-1, 1]}(y) = \max\{ -1, \max\{1, y\} \}$ projects onto the range $[-1, 1]$. It is easy to bound the classification error of $f_\w$ in terms of the $\ell_2$-loss of $g_\w$:
\begin{equation}
\bP\{ \operatorname{sgn}(f_\w(\x)) \neq y \} = \bP\{ yf_\w(\x) < 0 \} = \E[\1_{yf_\w(\x) < 0}] = \E[\1_{yg_\w(\x) < 0}] \leq \E(1-yg_\w(\x))^2 = L_\cD(g_\w).
\end{equation}
We use Theorem~\ref{thm:generalization_classification} for class $\frac14\cL_\alpha$ to get the generalization bound as follows:
\begin{align*}
\frac14 L_\cD(g_\w) &\leq \frac14 \hat{L}_{\cS}(g_\w) +2\fR_\cS(\frac14 \cL_\alpha)+ \sqrt\frac{\log(1/\delta)}{2n}\\
\implies L_\cD(g_\w) &\leq \hat{L}_{\cS}(g_\w) + 4\fR_\cS( \cF_\alpha)+4\sqrt\frac{\log(1/\delta)}{2n} \tag{by Lemma~\ref{lem:rademacher_twice}}\\
\implies L_\cD(g_\w) &\leq \hat{L}_{\cS}(g_\w) + \frac{8\alpha\|\X\|_{\C^\dagger}}{\sqrt{\beta}n}+4\sqrt\frac{\log(1/\delta)}{2n} \tag{by Theorem~\ref{thm:rademacher_dnn_new}}\\
\end{align*}
\end{proof}

\begin{corollary}\label{cor:gen_bound_nobeta_classification}
Consider a binary classification setting where $\cY = \{-1,+1\}$. For any $\w\in \cF'_{\alpha}$, for any $\delta\in(0,1)$, the following generalization bound holds with probability at least $1-\delta$ over a sample of size $n$ and the randomization in symmetrization \begin{equation*}
\bP\{ y g_\w(\x) < 0 \} \leq 2\hat{L}_{\cS'}(g_\w) + \frac{23\alpha'\|\X\|_{\C^\dagger}}{n} + 8\sqrt{\frac{\log(1/\delta)}{2n}}
\end{equation*}
\end{corollary}

\begin{proof}[Proof of Corollary~\ref{cor:gen_bound_nobeta_classification}]
Akin to proof of Corollary~\ref{cor:gen_bound_nobeta_regression}, we have that $L_\cD(f) \leq 2L_{\cD'}(f)$, and the marginal distribution is $\frac12$-retentive. Proof of Corollary~\ref{cor:gen_bound_nobeta_classification} follows by doubling the right hand side of inequalities in Corollary~\ref{cor:gen_bound_beta_classification}, and substituting $\beta=\frac12$.
\end{proof}

\section{Additional Experiments}
In this section, we include additional plots which was not reported in the main paper due to the space limitations.

\subsection{Matrix Completion}
Figure~\ref{fig:sensing} in the main paper shows comparisons between plain SGD and the dropout algorithm on the MovieLens dataset for a factorization size of $d_1=70$. The observation that we make with regard to those plots is not at all limited to the specific choice of the factorization size. In Figure~\ref{fig:add_mlens} here, we report similar experiments with factorization sizes $d_1 \in \{ 30, 110, 150, 190 \}$. It can be seen that the overall behaviour of plain SGD and dropout are very similar in all experiments. In particular, plain SGD always achieves the best training error but it has the largest generalization gap. Furthermore, increasing the dropout rate increases the training error but results in a tighter generalization gap. 

It can be seen that an appropriate choice of the dropout rate \emph{always} perform better than the plain SGD in terms of the test error. For instance, a dropout rate of $p=0.2$ seems to always outperform plain SGD. Moreover, as the factorization size increases, the function class becomes more complex, and a larger value of the dropout rate is more helpful. For example, when $d_1=30$, the dropout with rates $p=0.3, 0.4$ fail to achieve a good test performance, where as for larger factorization sizes ($d_1 \in \{ 110,150,190 \}$), they consistently outperform plain SGD as well as other dropout rates.
\begin{figure*}[!t]
\centering
\begin{tabular}{cccc}
\hspace*{-15pt} 
$d_1 = 30$ & \hspace*{-25pt} 
$d_1 = 110$ & \hspace*{-25pt} 
$d_1 = 150$ & \hspace*{-25pt} 
$d_1 = 190$\\
\hspace*{-15pt} 
\includegraphics[width=0.28\textwidth]{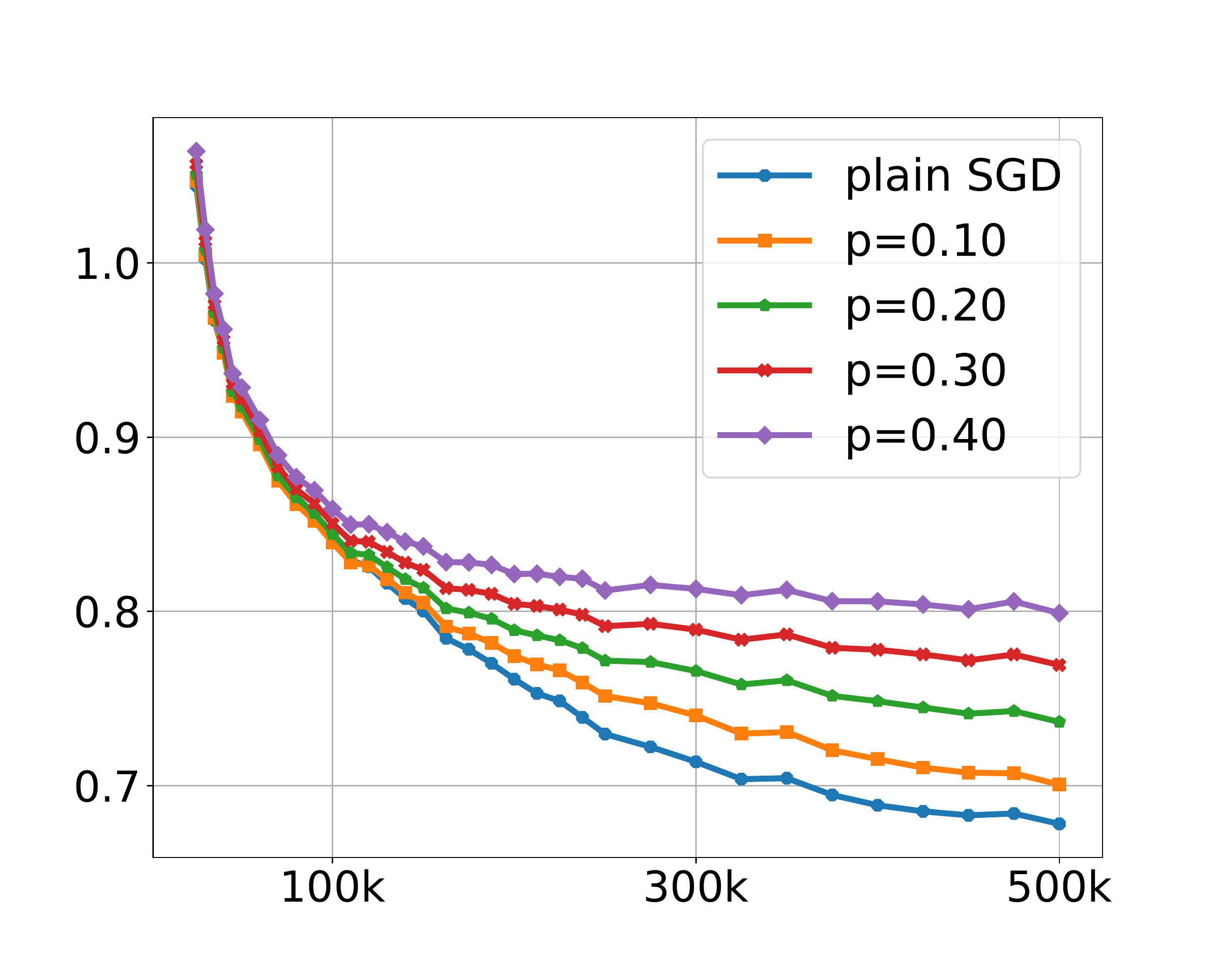}
&
\hspace*{-22pt} 
\includegraphics[width=0.28\textwidth]{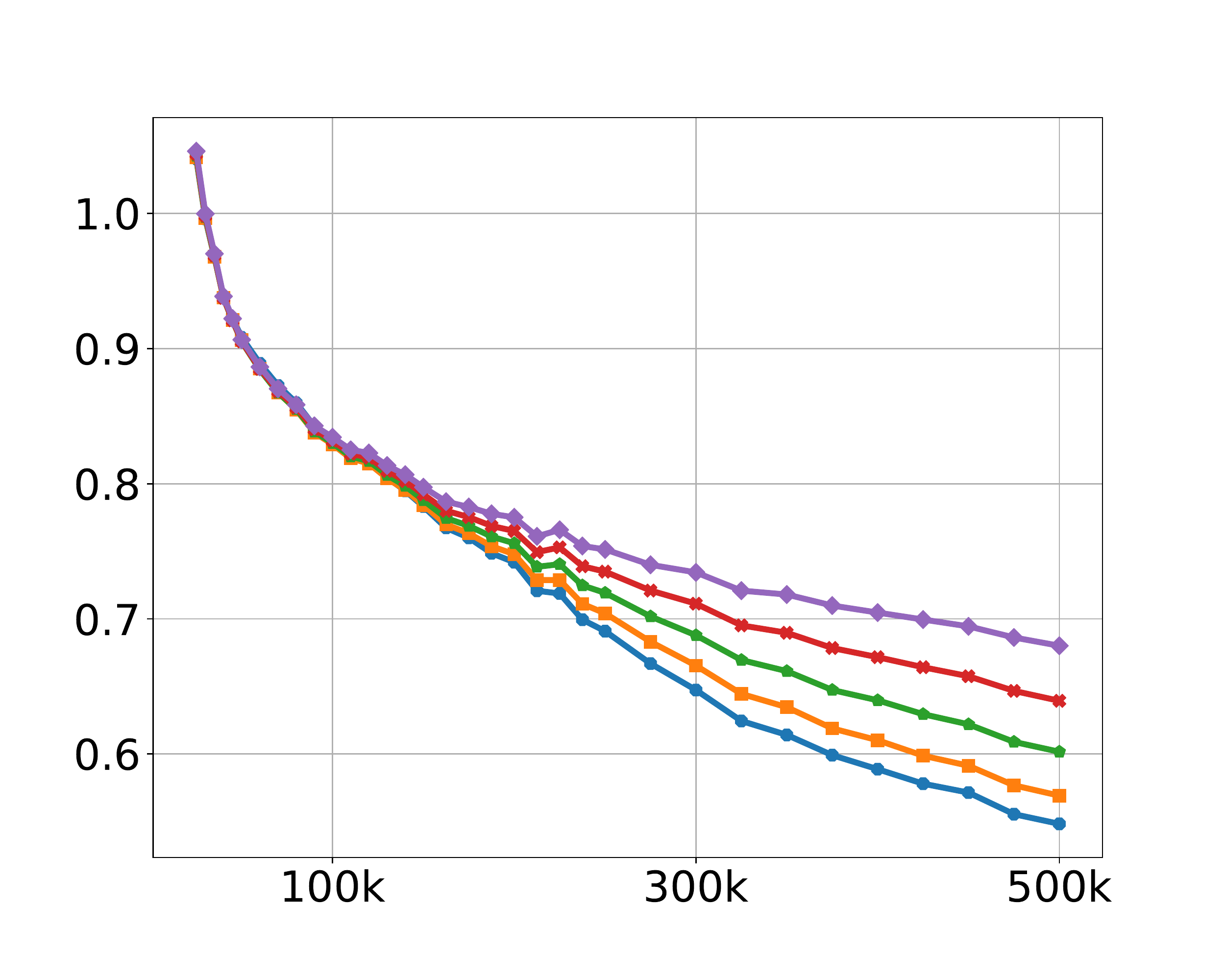}
&
\hspace*{-22pt} 
\includegraphics[width=0.28\textwidth]{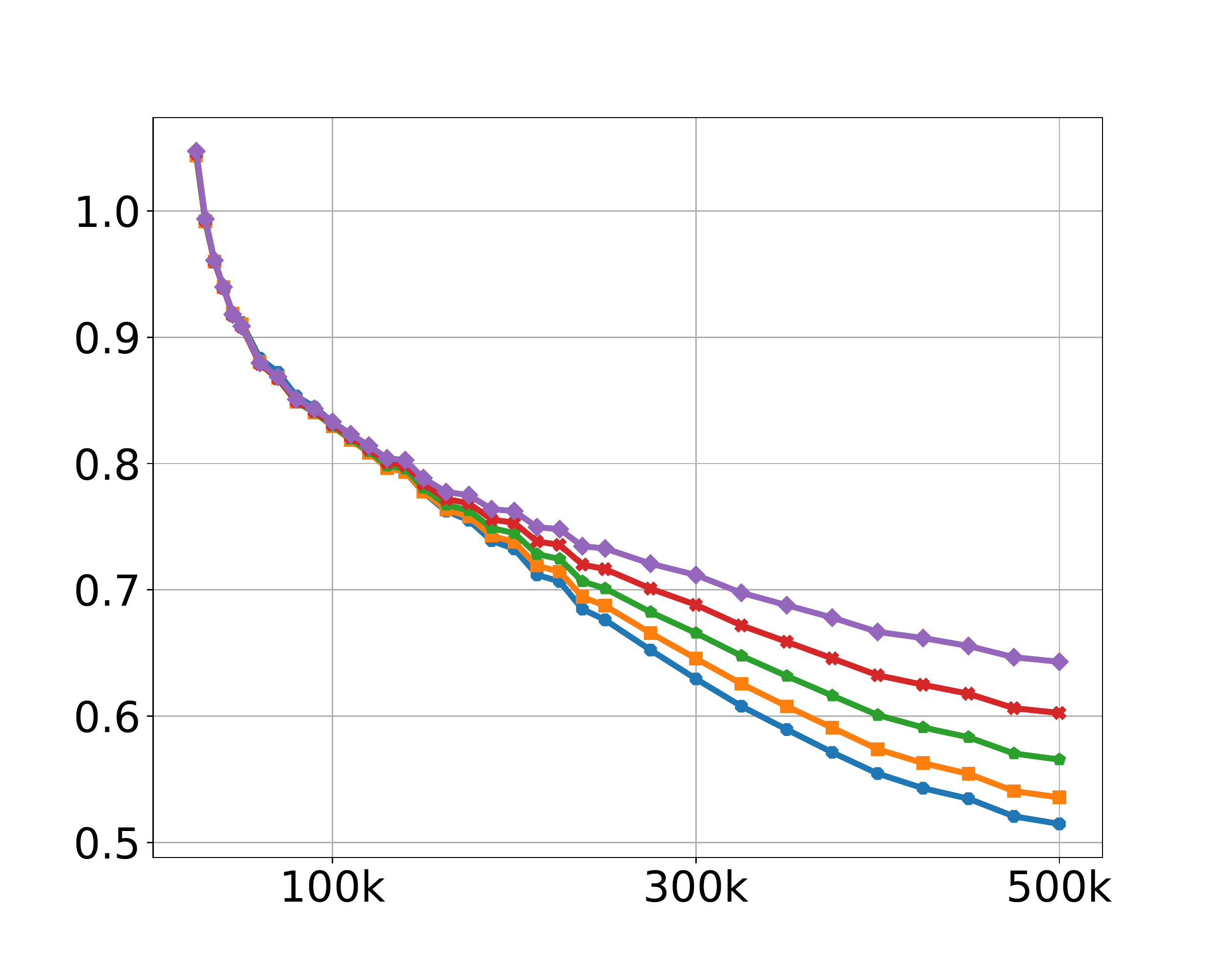}
&
\hspace*{-23pt} 
\includegraphics[width=0.28\textwidth]{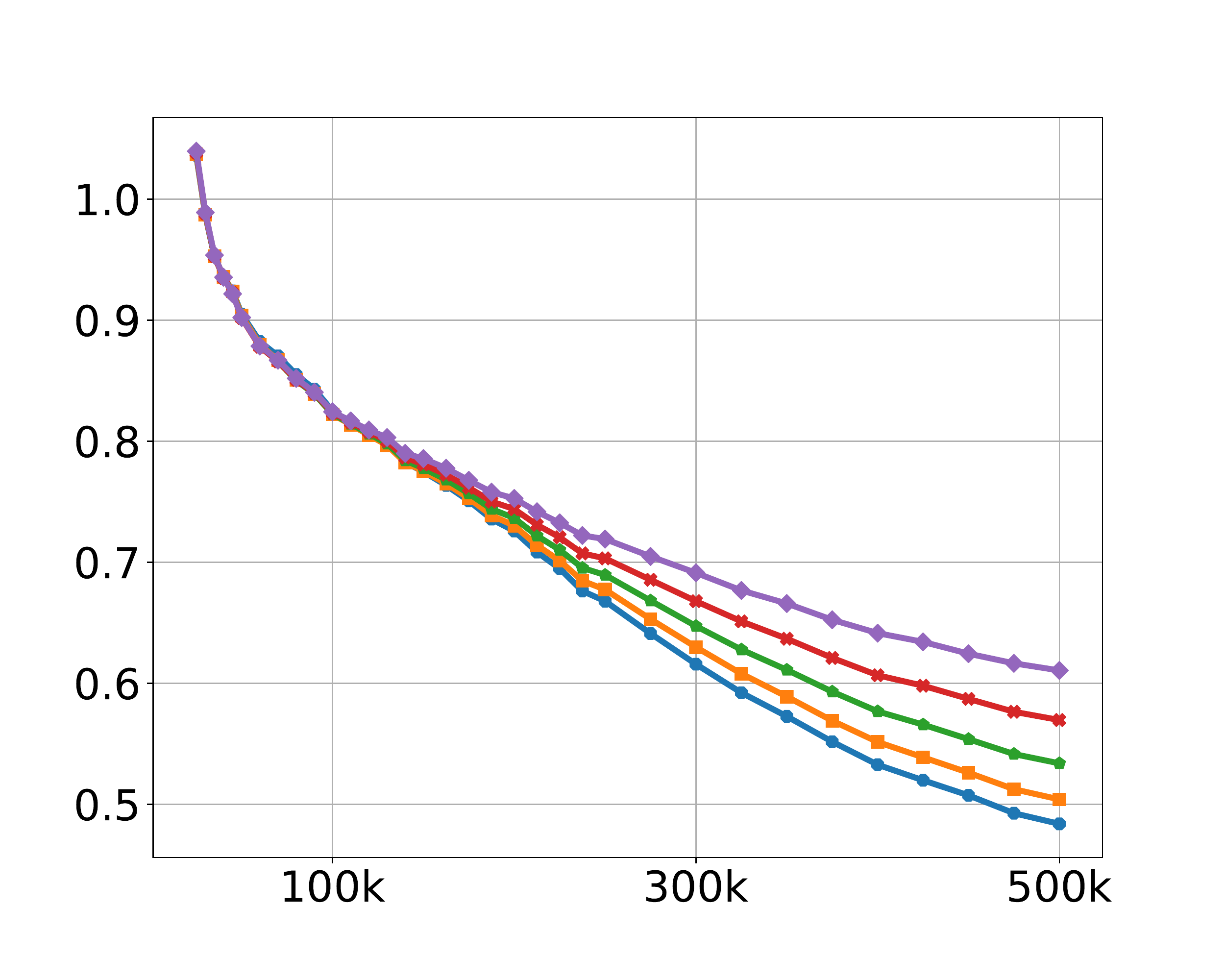}\\
\hspace*{-15pt} 
\includegraphics[width=0.28\textwidth]{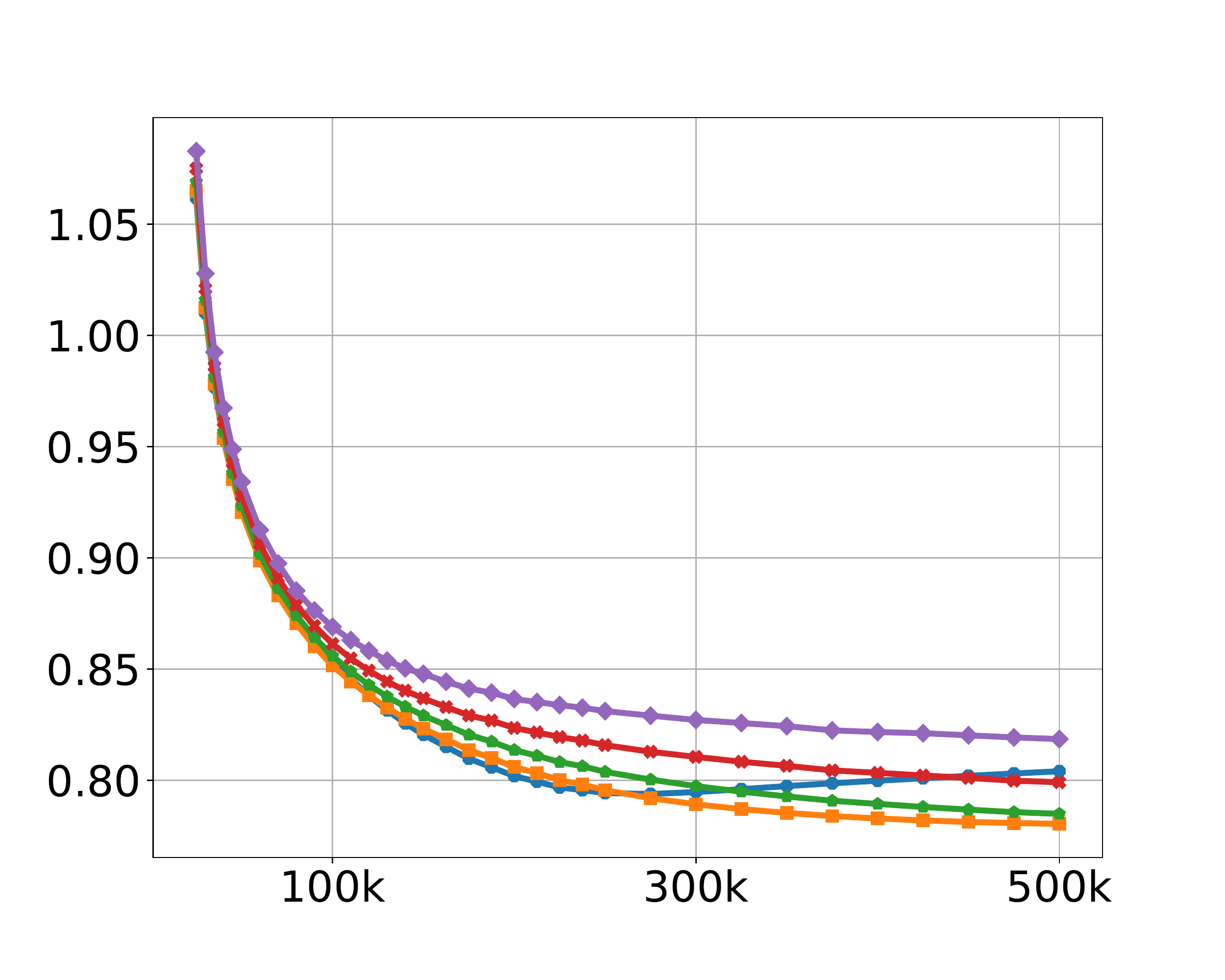}
&
\hspace*{-22pt} 
\includegraphics[width=0.28\textwidth]{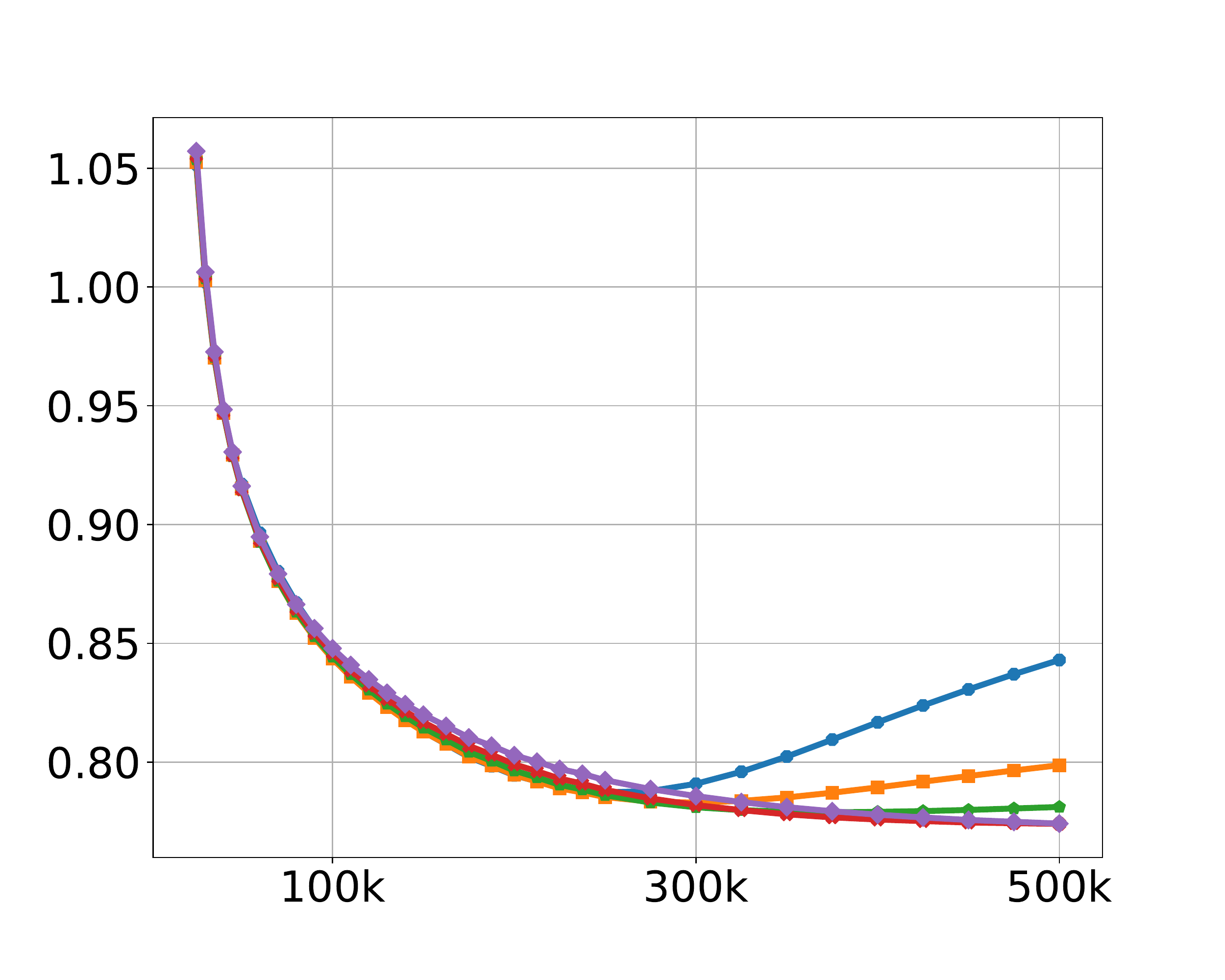}
&
\hspace*{-22pt} 
\includegraphics[width=0.28\textwidth]{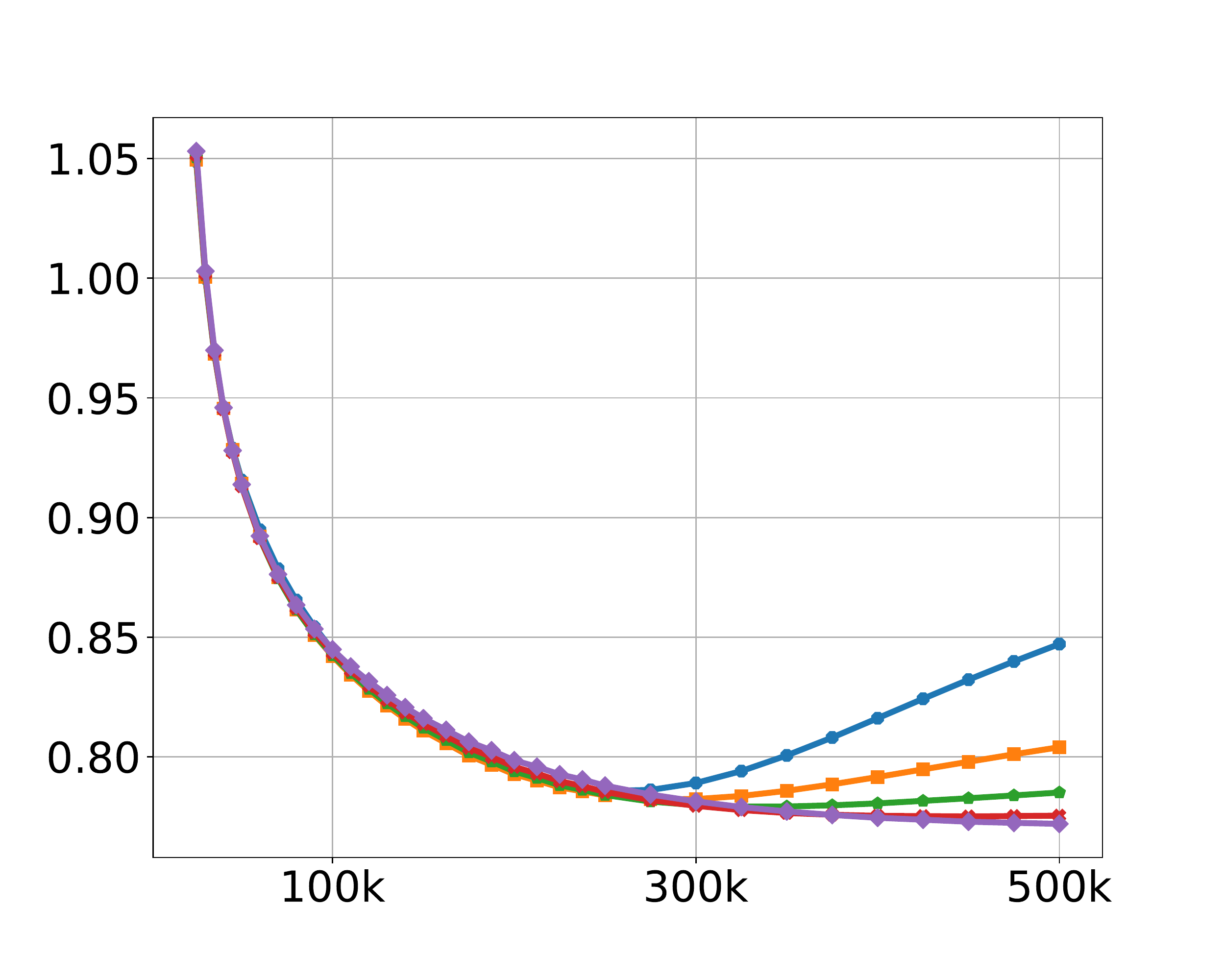}
&
\hspace*{-23pt} 
\includegraphics[width=0.28\textwidth]{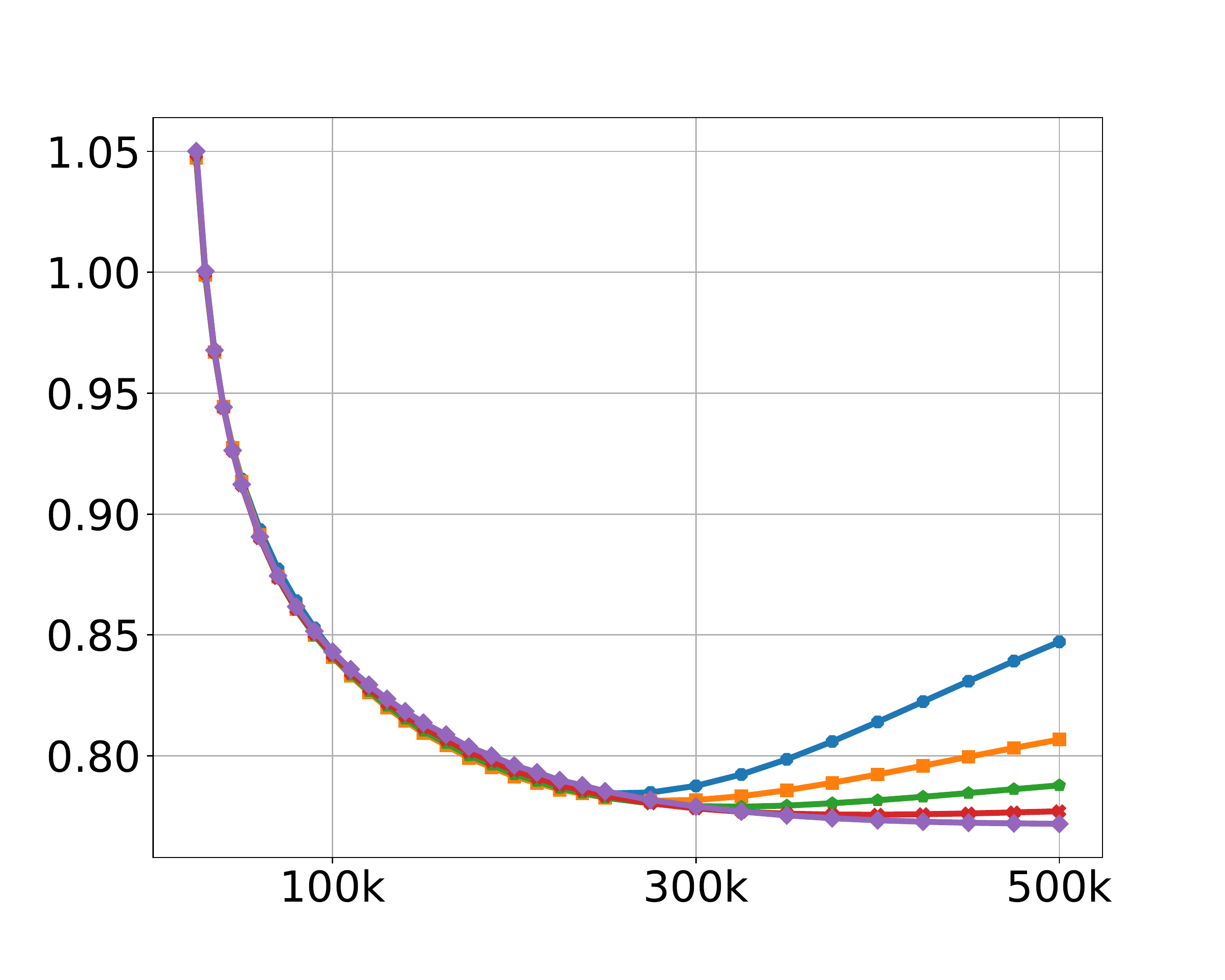}\\
\hspace*{-15pt} 
\includegraphics[width=0.28\textwidth]{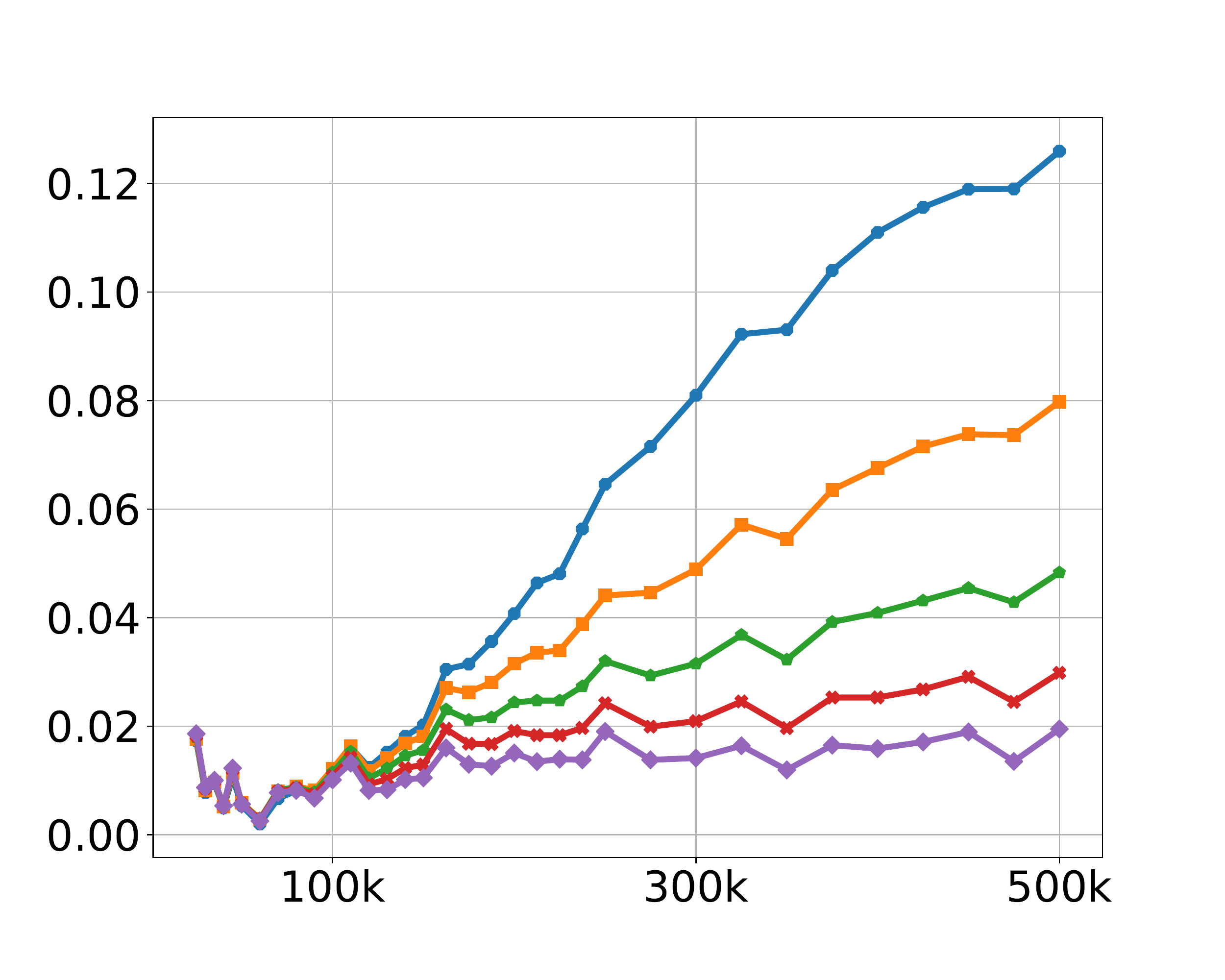}
&
\hspace*{-22pt} 
\includegraphics[width=0.28\textwidth]{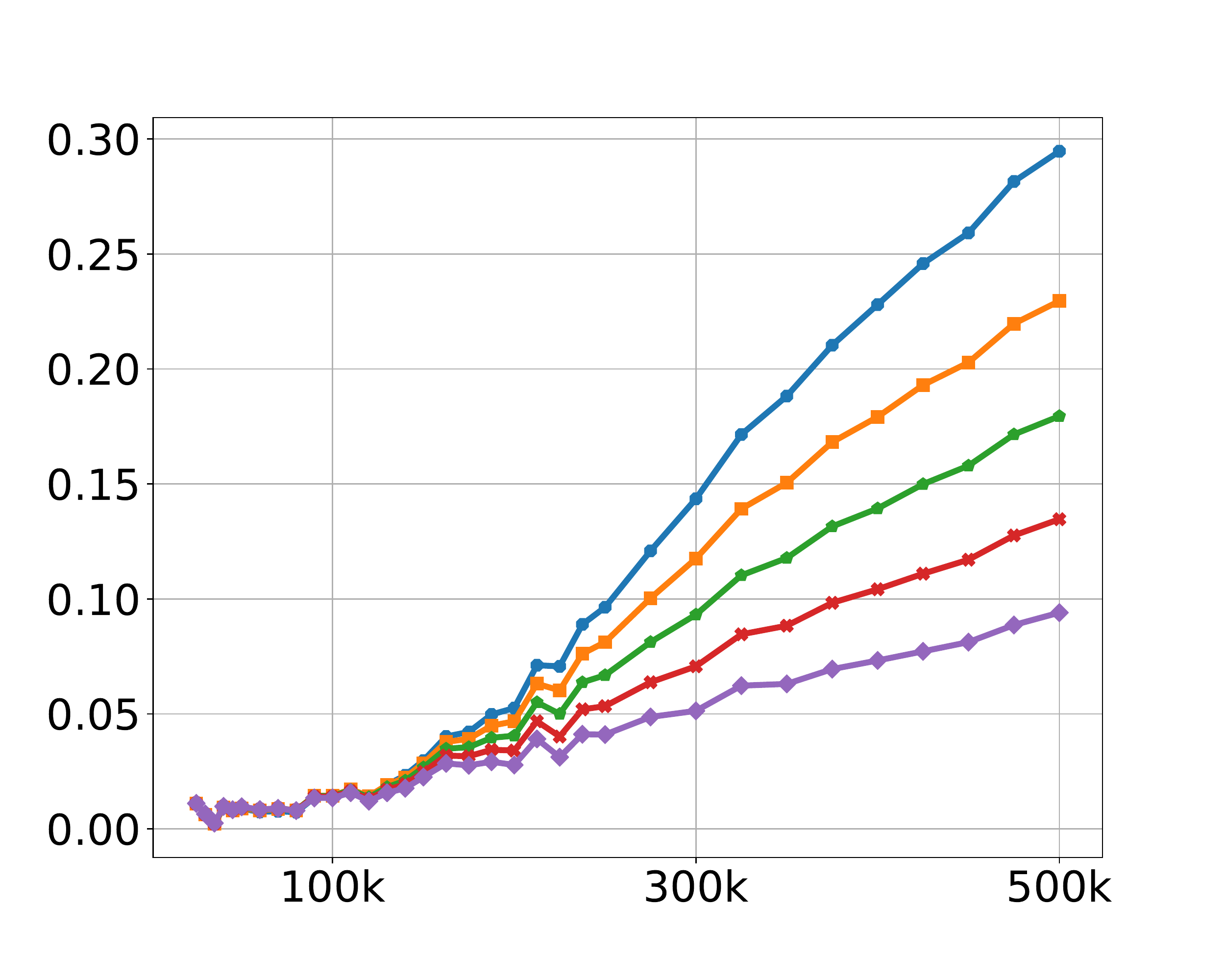}
&
\hspace*{-22pt} 
\includegraphics[width=0.28\textwidth]{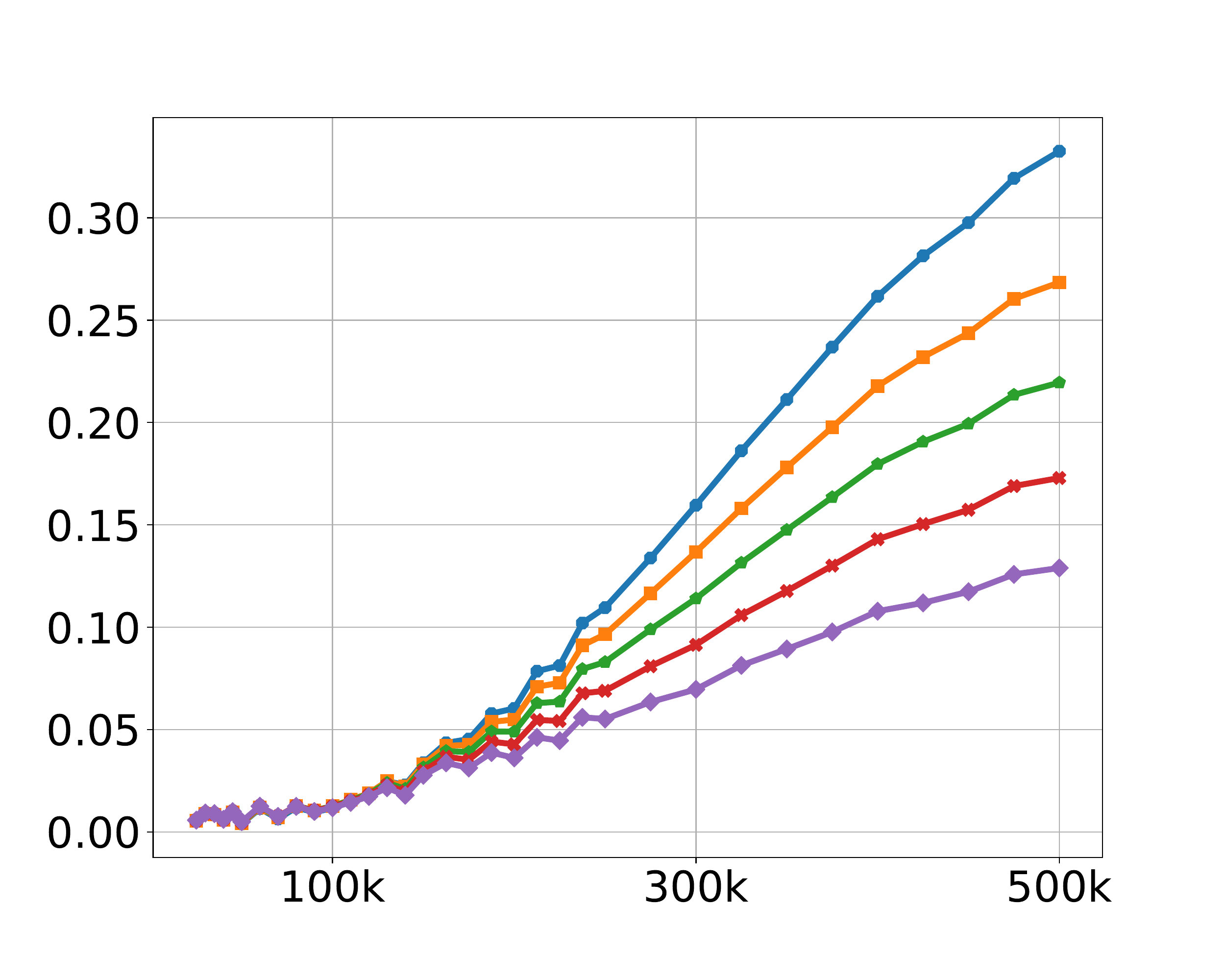}
&
\hspace*{-23pt} 
\includegraphics[width=0.28\textwidth]{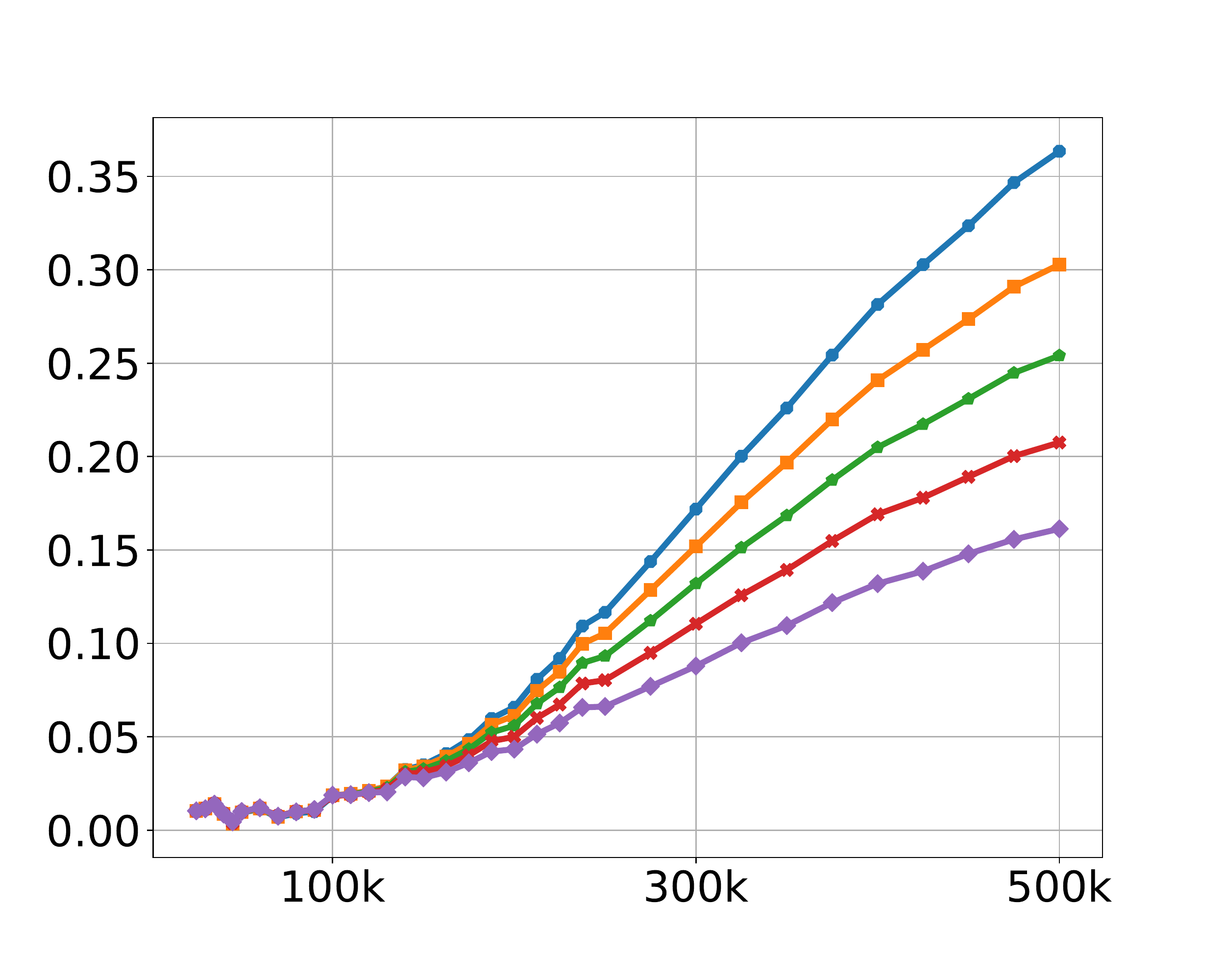}
\vspace*{-15pt} 
\end{tabular}
\caption{\small{MovieLens dataset: the training error ({\bf top}), the test error ({\bf middle}), and the generalization gap for plain SGD as well as dropout with $p \in \{ 0.25, 0.50, 0.75 \}$ as a function of the number of iterates, for different factorization sizes $d_1 = 30$ (first column), $d_1 = 110$ (second column), $d_1 = 150$ (third column), and $d_1 = 190$ (forth column).}}
\label{fig:add_mlens}
\end{figure*}

\subsection{Shallow Neural Networks}
\begin{figure*}
\centering
\begin{tabular}{ccc}
\hspace*{-20pt} 
\includegraphics[width=0.365\textwidth]{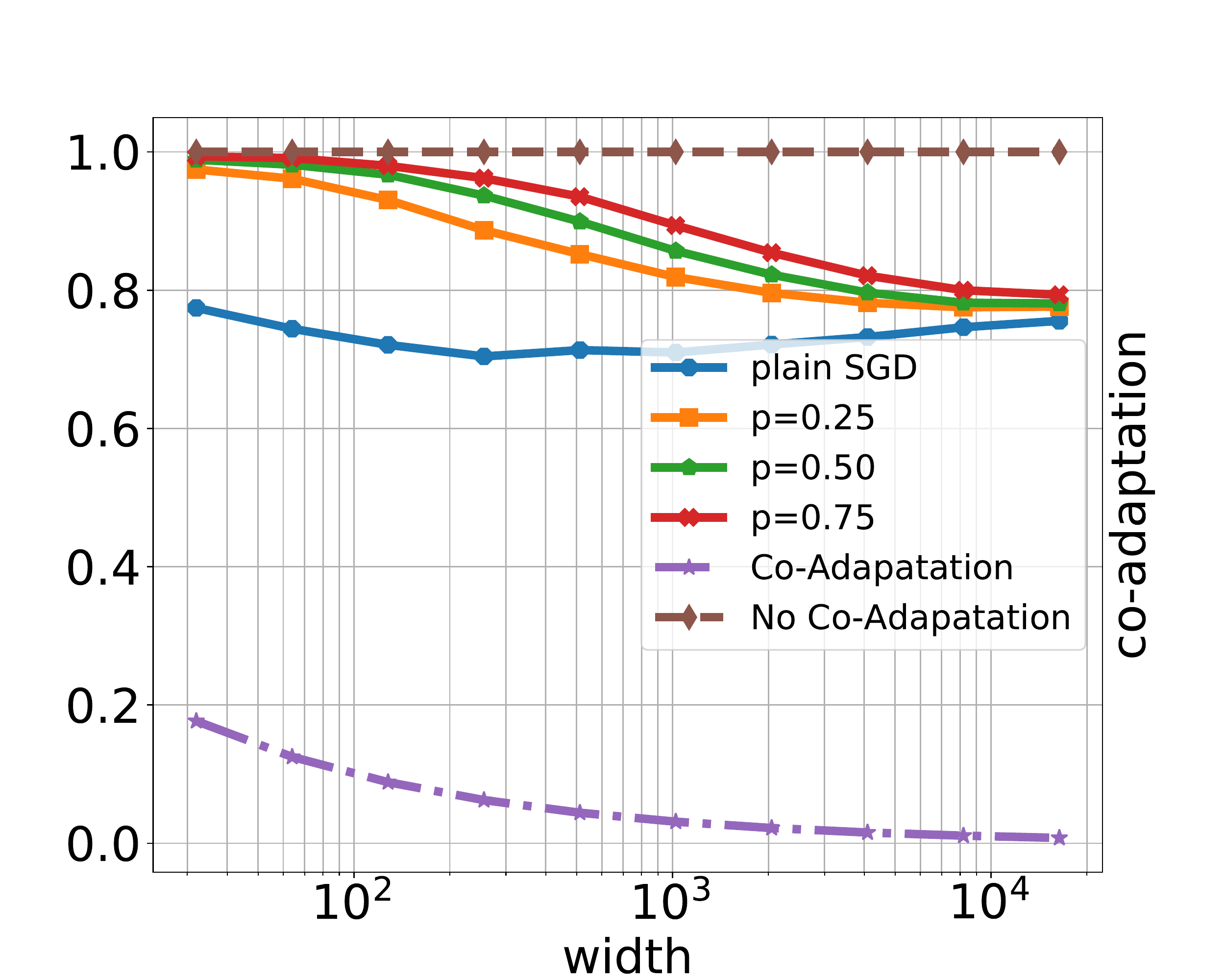}
&
\hspace*{-22pt} 
\includegraphics[width=0.365\textwidth]{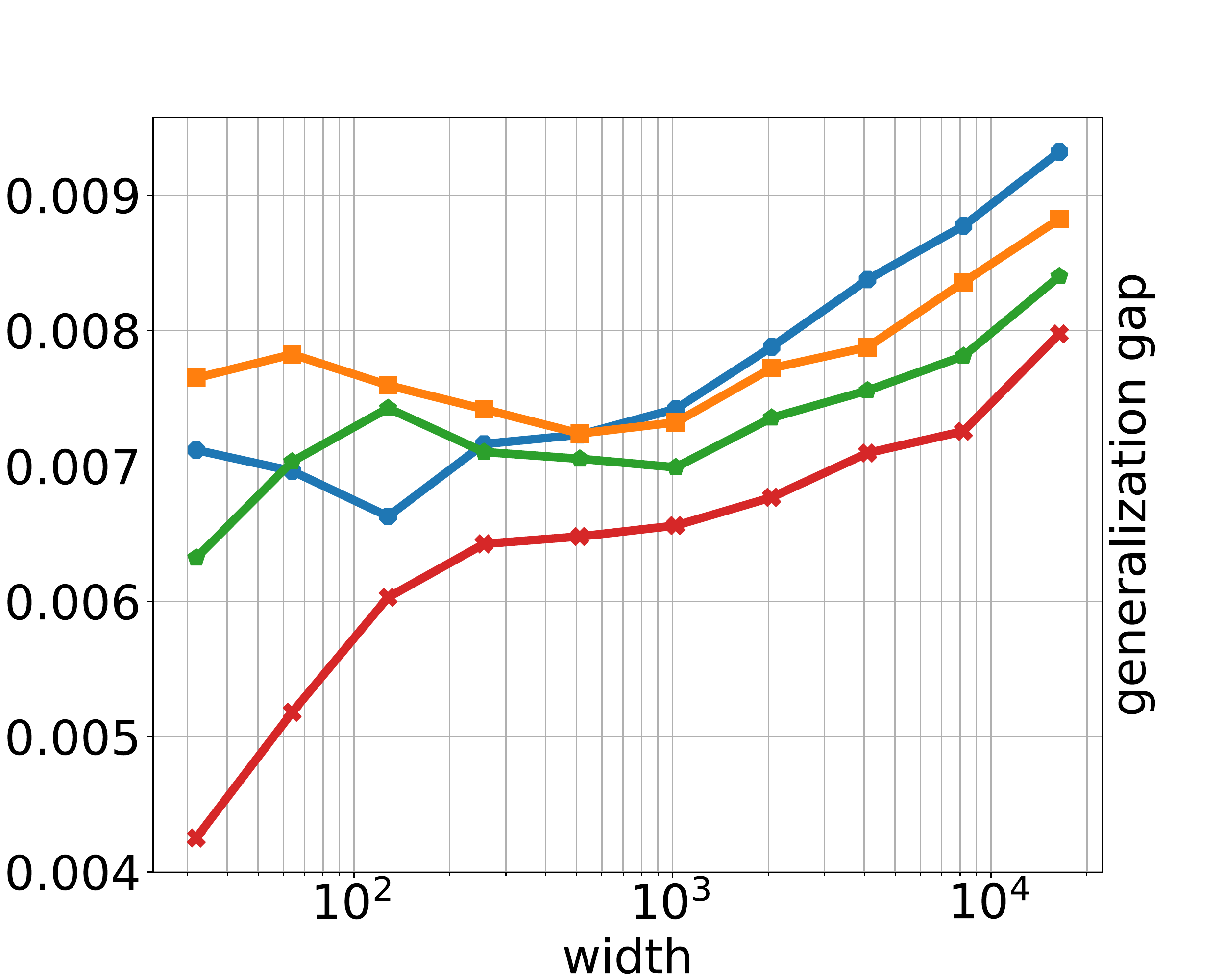}
&
\hspace*{-22pt} 
\includegraphics[width=0.365\textwidth]{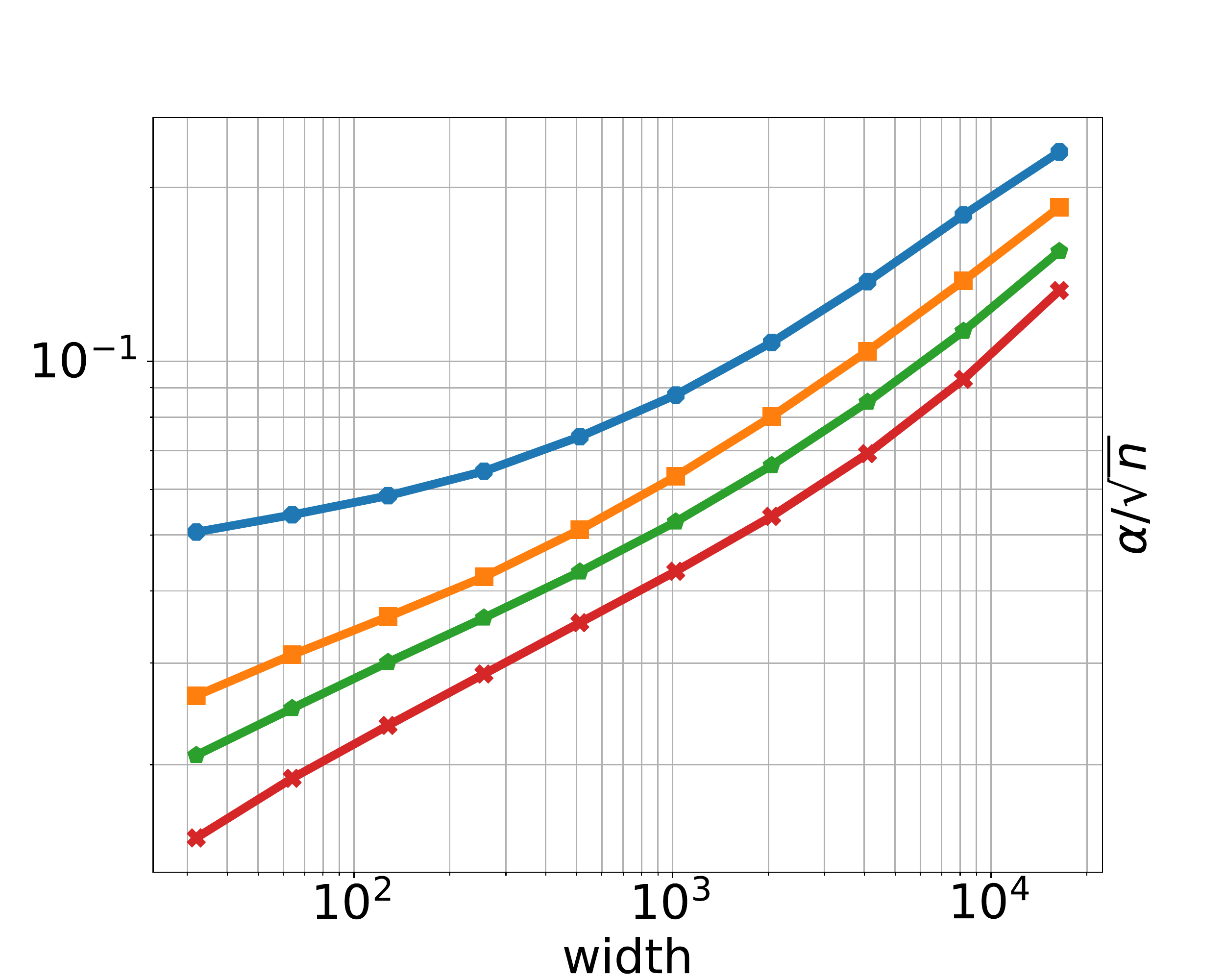}
\vspace*{-10pt} \\
\hspace*{-20pt} 
\includegraphics[width=0.365\textwidth]{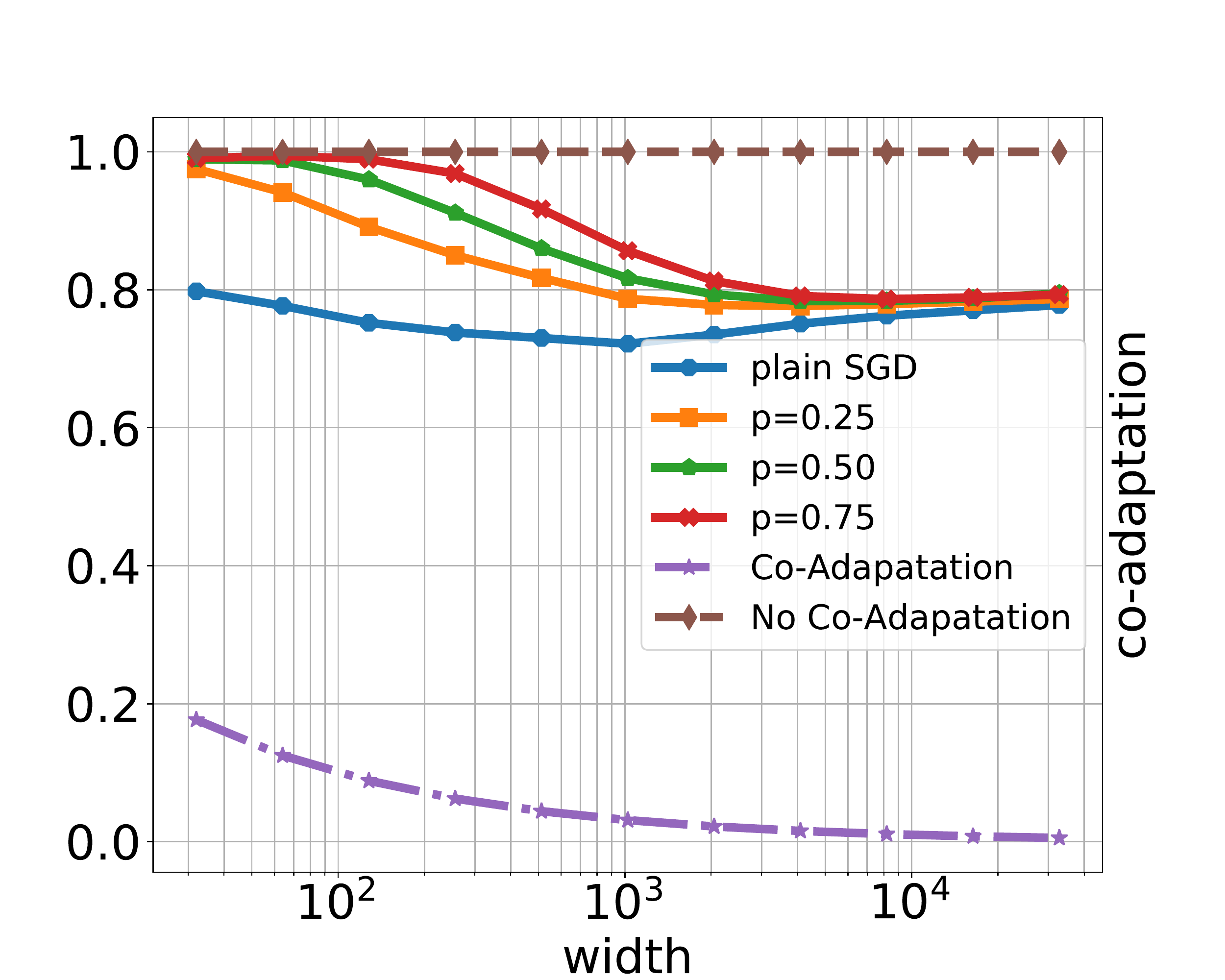}
&
\hspace*{-22pt} 
\includegraphics[width=0.365\textwidth]{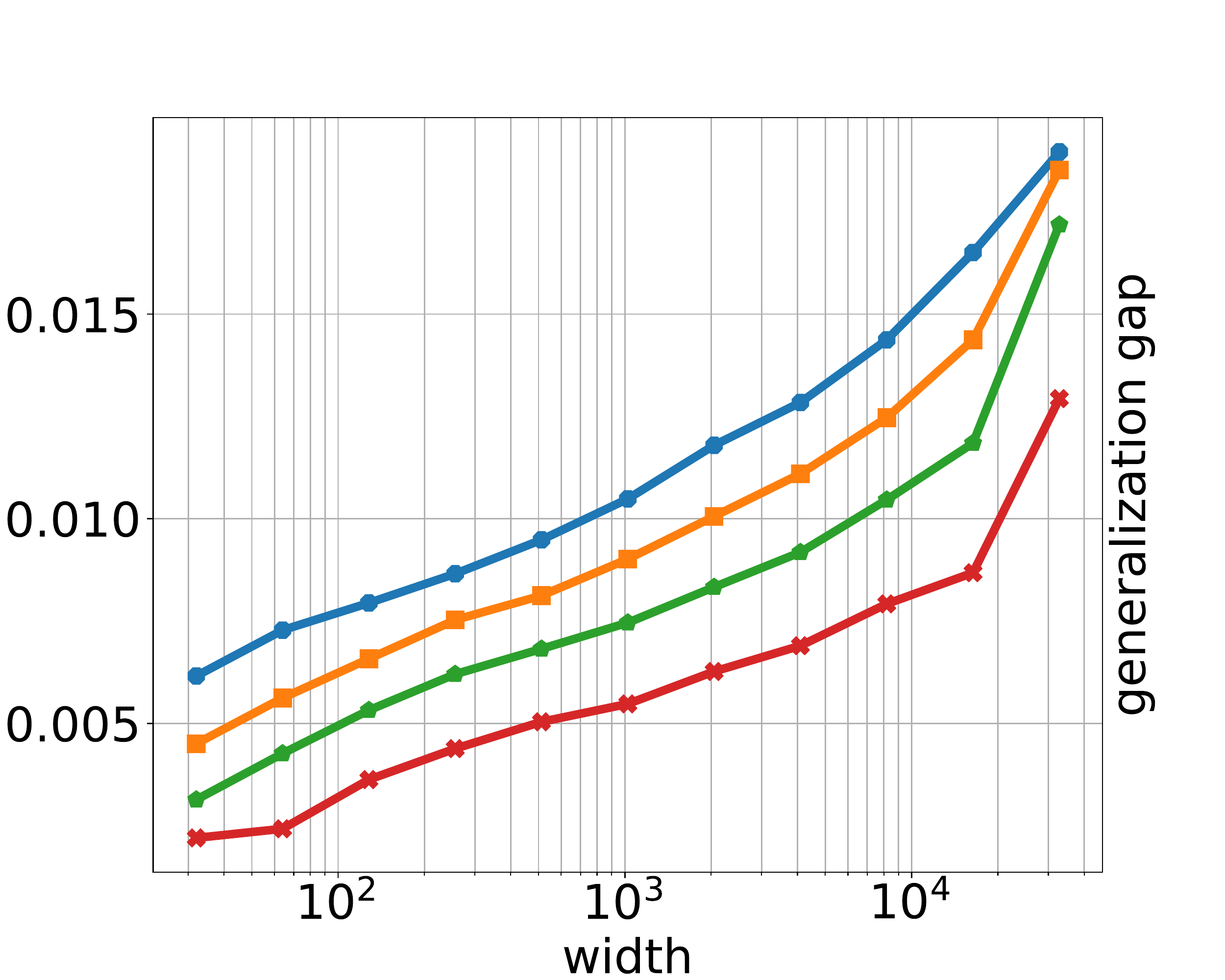}
&
\hspace*{-22pt} 
\includegraphics[width=0.365\textwidth]{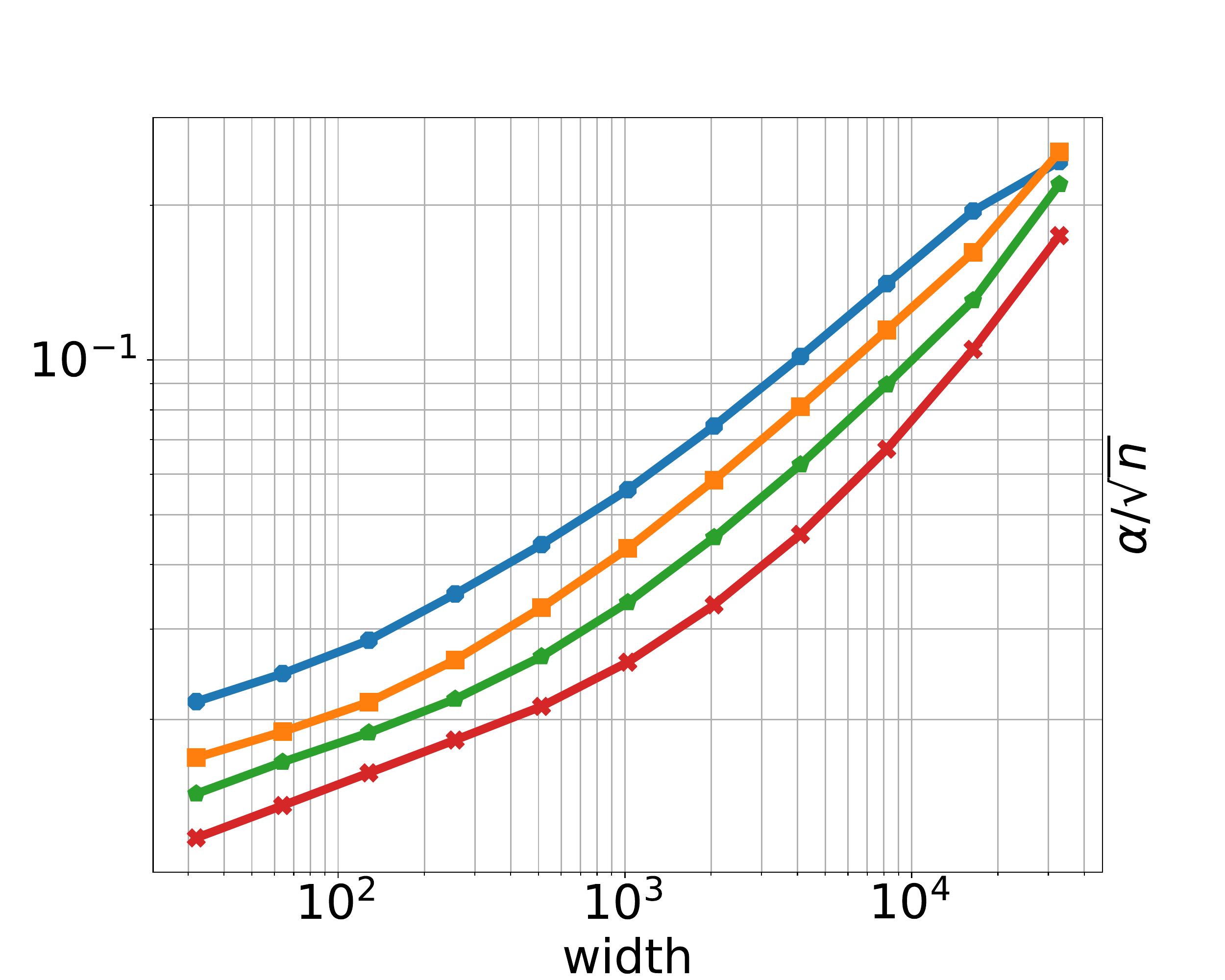}
\vspace*{-15pt} 

\end{tabular}
\caption{\label{fig:2nn_app}({\bf left}) \emph{``co-adaptation''}; ({\bf middle}) generalization gap; and ({\bf right}) $\alpha/\sqrt{n}$ ({\bf top}) with symmetrization on FashionMNIST; and ({\bf bottom}) without symmetrization on MNIST. In left column, the dashed brown and dotted purple lines represent minimal and maximal co-adaptations, respectively.}
\end{figure*}

In Figure~\ref{fig:2nn_app}, we plot the co-adaptation measure, the generalization gap, as well as the complexity measure $\alpha/\sqrt{n}$ as a function of width of the network, for FashionMNIST with symmetrization, and for MNIST without symmetrization.

The co-adaptation plot is very similar to Figure~\ref{fig:2nn} in the main text. In particular, 1) increasing the dropout rate results in less co-adaptation; 2) even plain SGD is biased towards networks with less co-adapation; and 3) as the networks becomes wider, the co-adaptation curves corresponding to plain SGD converge to those of dropout. We also make similar observations for the generalization gap as well as the complexity term $\alpha/\sqrt{n}$. In particular, 1) a higher dropout rate corresponds to a lower generalization gap, uniformly for all widths; 2) the generalization gap is higher for wider networks; and 3) curves with smaller complexity terms in the right plot correspond to curves with smaller generalization gaps in the middle plot. 

\subsection{Deep Neural Networks}
In Section~\ref{sec:dnn}, we derived generalization bounds that scale with the explicit regularizer as $O(\sqrt\frac{\texttt{width} \cdot R(\w)}{n})$. Although our theoretical analysis is limited to two-layer networks; empirically, we show in Figure~\ref{fig:dnn} that the generalization gap correlates well with this measure even for deep neural networks. In particular, we train deep convolutional neural networks with a dropout layer on top of the feature extractor, i.e. the top hidden layer. Let $\texttt{feature}_i$ denote the $i$-th hidden node in the top hidden layer. Akin to the derivation presented in Proposition~\ref{prop:dropout_reg_dnn}, it is easy to see that the (expected) explicit regularizer is given by $R(\w)=\frac{p}{1-p}\sum_{i=1}^{\texttt{width}} \|\u_i\|^2 a_i^2$, where $\texttt{width}$ is the width of the top hidden layer, $\U$ denotes the top layer weight matrix, and $a_i^2 = \E_\x [\texttt{feature}_i(\x)^2]$ is the second moment of the $i$-th node in the top hidden layer. 

We train convolutional neural networks with and without dropout, on MNIST, Fashion MNIST, and CIFAR-10. The CIFAR-10 dataset consists of 60K $32\times 32$ color images in 10 classes, with 6k images per class, divided into a training set and a test set of sizes 50K and 10K respectively~\cite{krizhevsky2009learning}. We do not perform symmetrization in these experiments. In contrast with the experiments in the previous section, here we run the experiments on full datasets, representing each of the ten classes as a one-hot target vector.

For MNIST and Fashion MNIST datasets, we use a convolutional neural network with one convolutional layer and two fully connected layers. The convolutional layer has 16 convolutional filters, padding and stride of 2, and kernel size of 5. We report experiments on networks with the width of the top hidden layer chosen from $\texttt{width}\in \{ 2^6, 2^7, 2^8, 2^9, 2^{10}, 2^{11} \}$. In all the experiments, a fixed learning rate $\texttt{lr}=0.5$ and a mini-batch of size 256 is used to perform the updates. We train the models for 30 epochs over the whole training set.

For CIFAR-10, we use an AlexNet~\cite{krizhevsky2012imagenet}, where the layers are modified accordingly to match the dataset. The only difference here is that we apply dropout to the top hidden layer, whereas  in~\cite{krizhevsky2012imagenet}, dropout is used on top of the second and the third hidden layers from the top. We report experiments on networks with the width of the top hidden layer chosen from $\texttt{width}\in \{ 2^5, 2^6, 2^7, 2^8, 2^9, 2^{10}, 2^{11} , 2^{12} \}$.
In all the experiments, an initial learning rate $\texttt{lr}=5$ and a mini-batch of size 256 is used to perform the updates. We train the models for 100 epochs over the whole training set. We decay the learning rate by a factor of 10 every 30 epochs.

\begin{figure*}[!t]
\centering
\begin{tabular}{ccc}
\hspace*{-10pt} 
MNIST
&
\hspace*{-12pt} 
Fashion MNIST
&
\hspace*{-12pt} 
CIFAR-10
\vspace*{-3pt}
\\
\hspace*{-10pt} 
\includegraphics[width=0.30\textwidth]{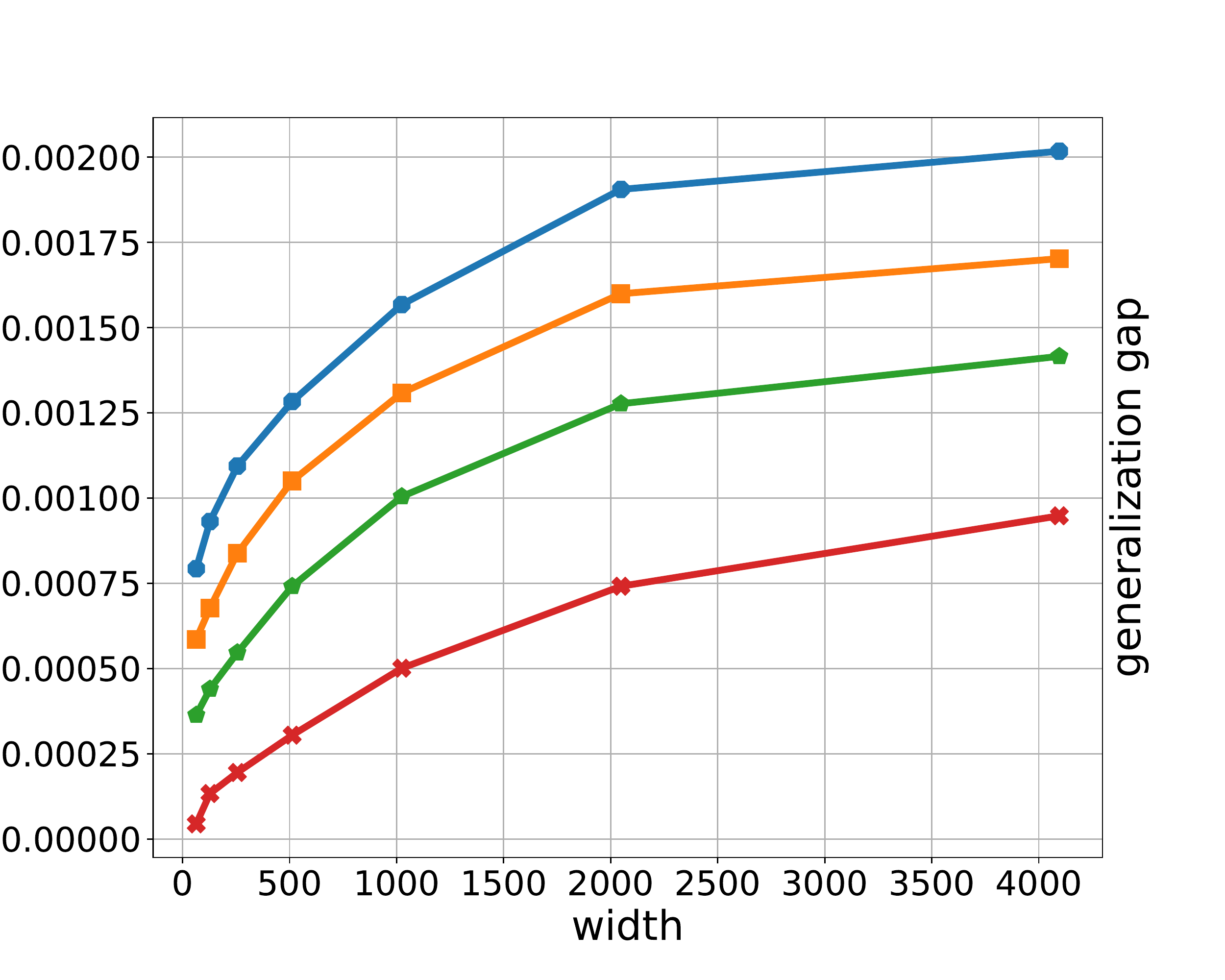}
&
\hspace*{-12pt} 
\includegraphics[width=0.30\textwidth]{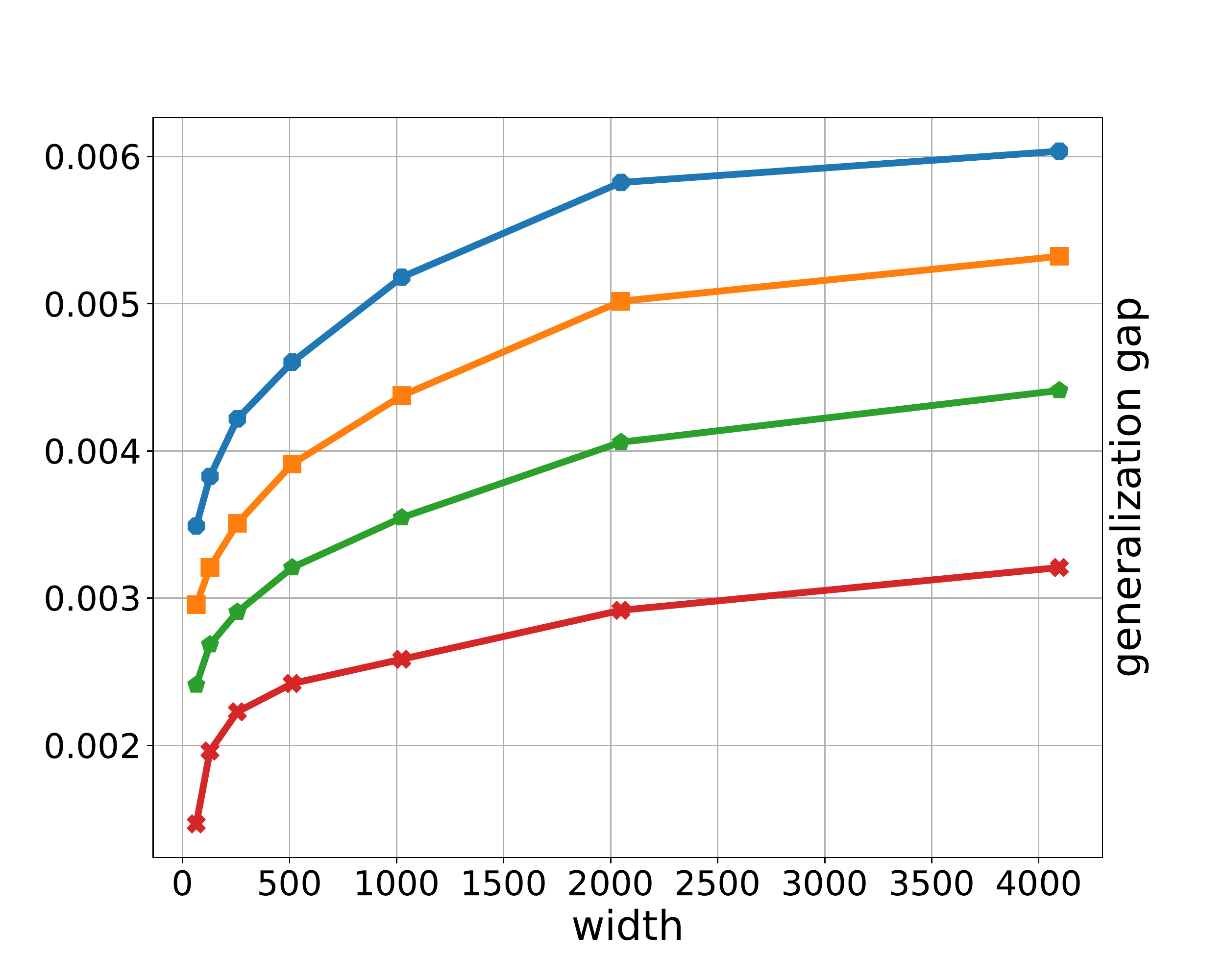}
&
\hspace*{-12pt} 
\includegraphics[width=0.30\textwidth]{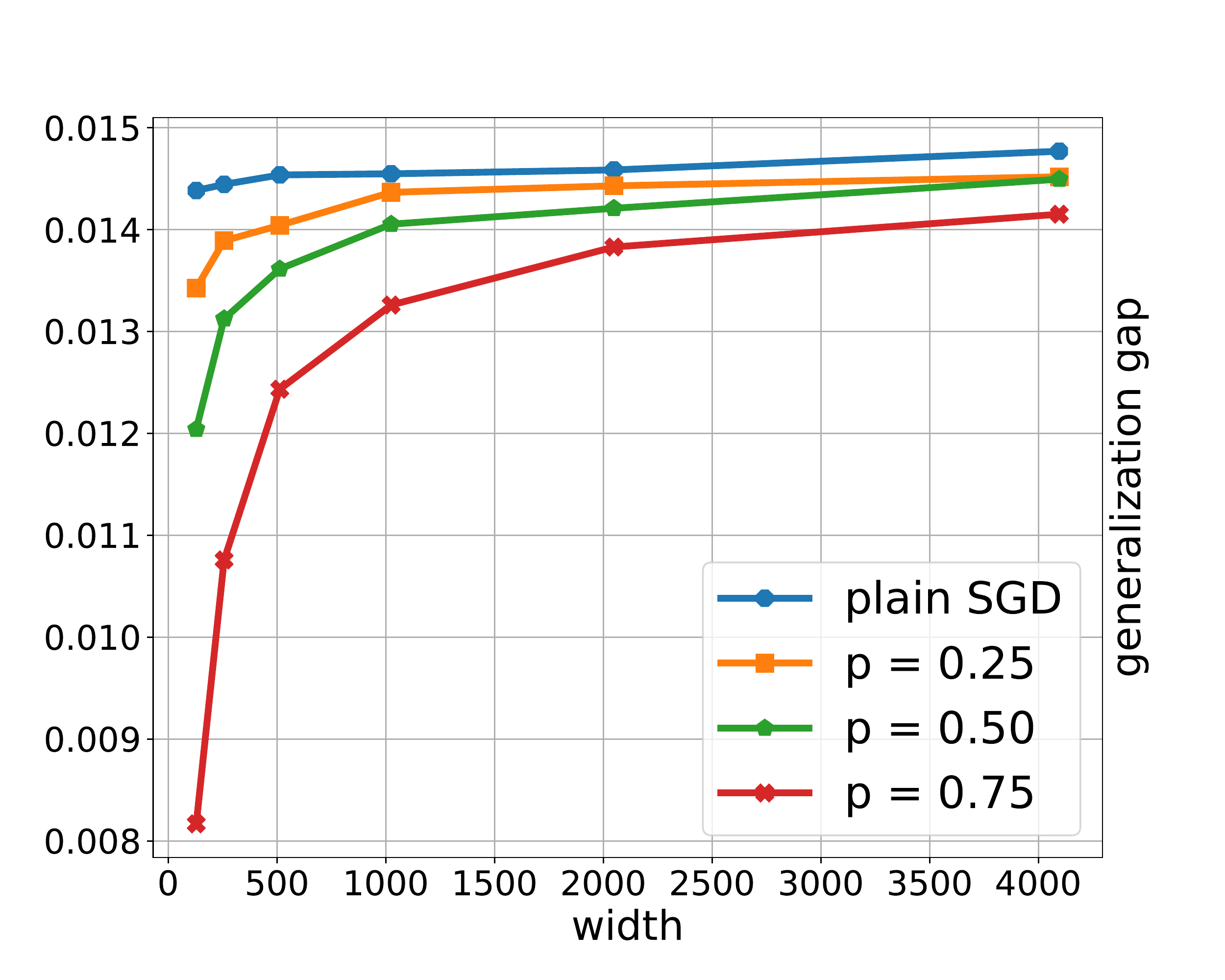}\\
\hspace*{-10pt} 
\includegraphics[width=0.30\textwidth]{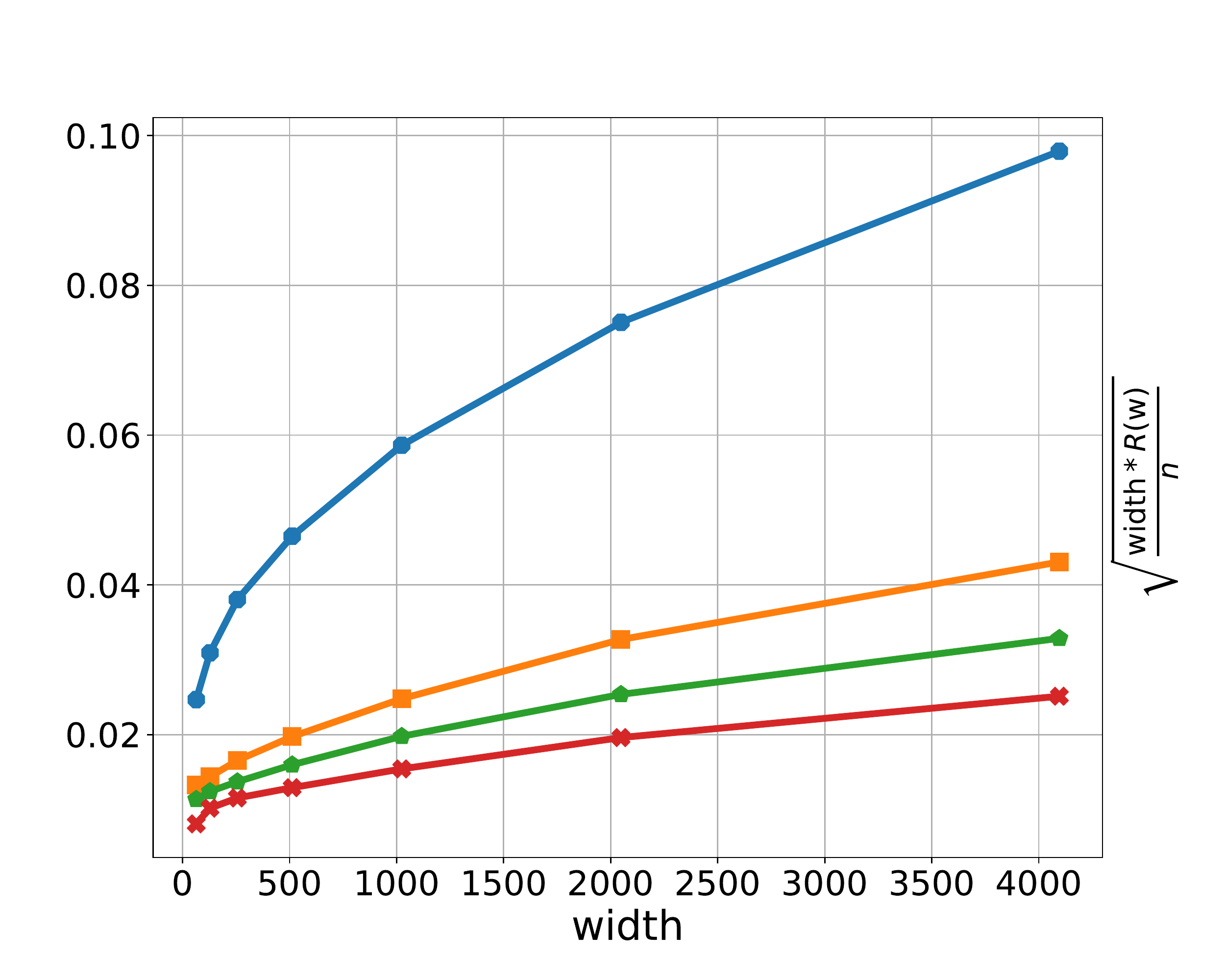}
&
\hspace*{-12pt} 
\includegraphics[width=0.30\textwidth]{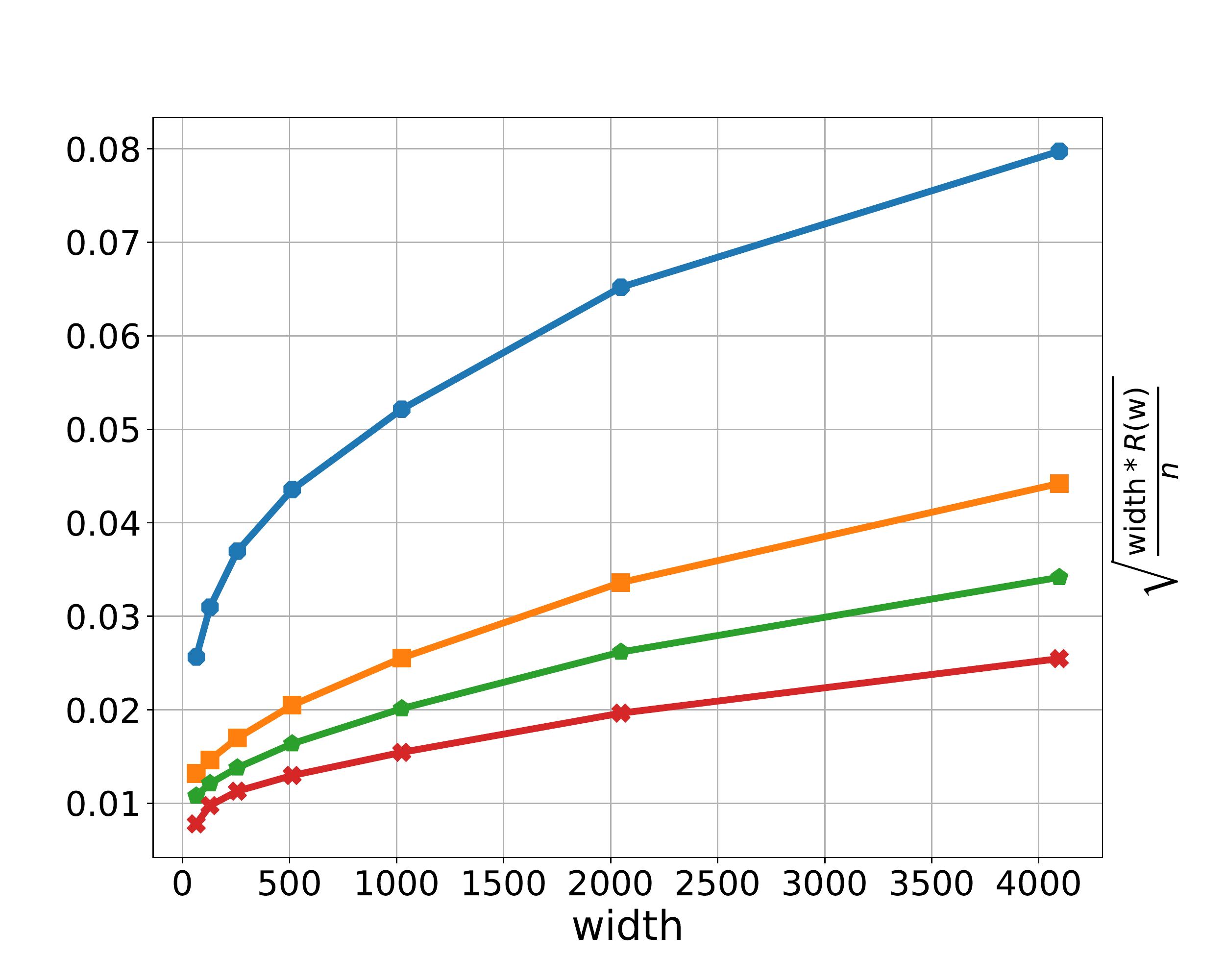}
&
\hspace*{-12pt} 
\includegraphics[width=0.30\textwidth]{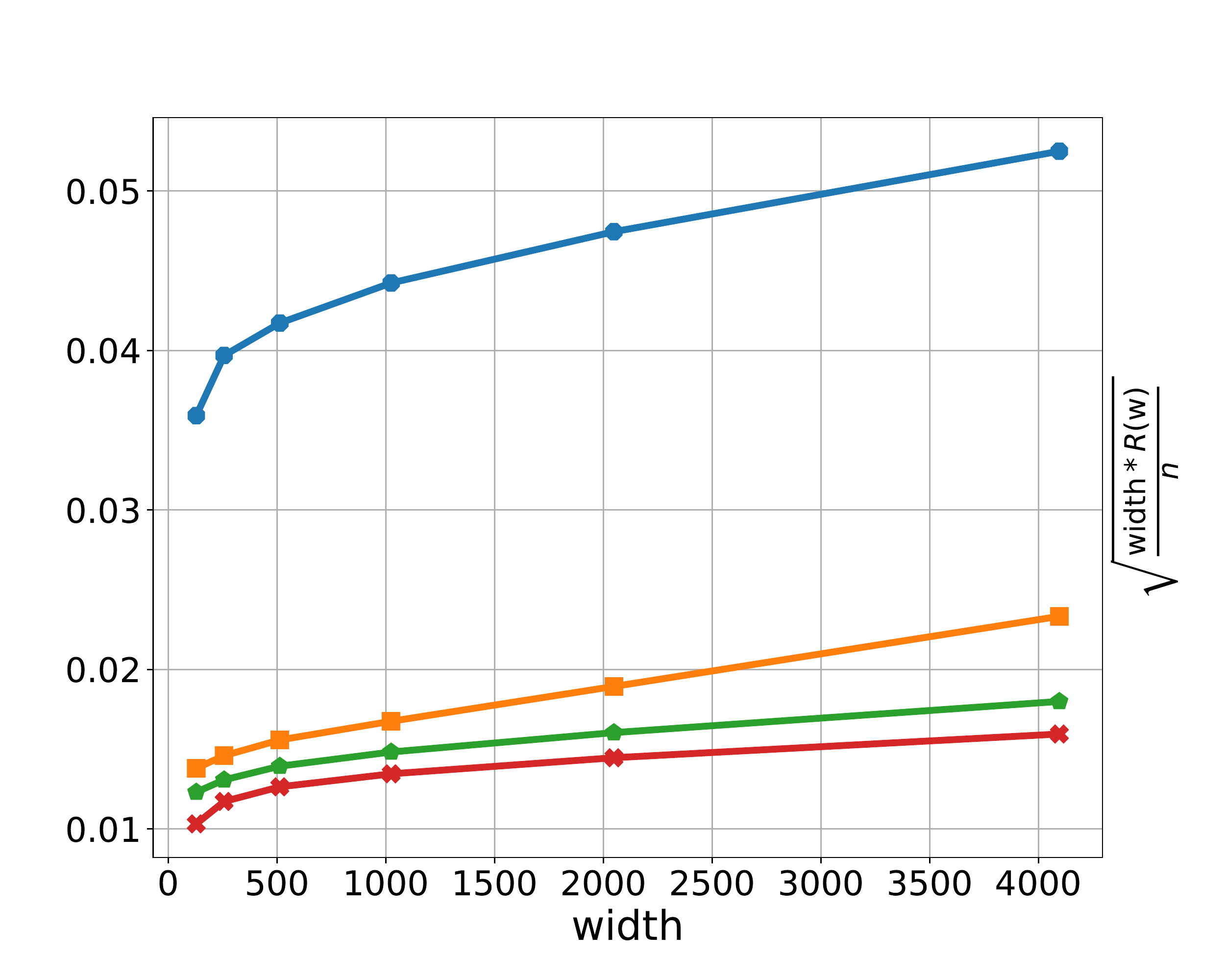}
\end{tabular}
\caption{({\bf top}) generalization gap and ({\bf bottom}) the complexity measure $(\sqrt\frac{\texttt{width}\cdot R(\w)}{n})$ as a function of the width of the top hidden layer on ({\bf left}) MNIST, ({\bf middle}) Fashion MNIST, and ({\bf right}) CIFAR-10.}
\label{fig:dnn}
\end{figure*}

%%%%%%%%%%%%%%%%%%%%%%%%%%%%%%%%%%%%%%%%%%%%%%%%%%%%%%%%%%%%%%%%%%%%%%%%%%%%%%%
%%%%%%%%%%%%%%%%%%%%%%%%%%%%%%%%%%%%%%%%%%%%%%%%%%%%%%%%%%%%%%%%%%%%%%%%%%%%%%%

\end{document}